\documentclass[10pt,journal,compsoc]{IEEEtran}

\ifCLASSOPTIONcompsoc
  \usepackage[nocompress]{cite}
\else
  \usepackage{cite}
\fi
\usepackage{array}
\usepackage{graphicx}
\usepackage{stfloats}
\usepackage{subfigure}
\usepackage{multirow}
\usepackage{booktabs}
\usepackage{makecell}
\usepackage{amsmath}
\usepackage{color}
\usepackage[dvipsnames]{xcolor}
\usepackage{caption}
\makeatletter
\let\NAT@parse\undefined
\makeatother
\usepackage{bm}
\usepackage[numbers,sort&compress]{natbib}
\usepackage{amssymb}
\usepackage{ragged2e}
\usepackage{enumitem} 
\usepackage{colortbl}
\definecolor{light}{gray}{.85}
\usepackage{CJKutf8}
\usepackage{amsthm}
\newtheorem{theorem}{Theorem}
\usepackage{pifont}
\usepackage{tablefootnote}
\usepackage{algorithm}
\usepackage{algorithmic}
\newtheorem{lemma}[theorem]{Lemma}
\usepackage{hyperref}  

\newcommand{\revised}{\textcolor{black}}
\ifCLASSINFOpdf
\else
\fi
\begin{document}
\title{Learning Disentangled Representations for Generalized Multi-view Clustering}

\author{Xin~Zou$^{\ast,\ddagger}$,
        Ruimeng Liu$^{\ast}$,
        Chang~Tang$^{\dagger}$,~\IEEEmembership{Senior Member,~IEEE},\\
        Zhenglai Li,
        Xinwang~Liu,
        Kunlun He,
        and Wanqing~Li,~\IEEEmembership{Senior Member,~IEEE}
\IEEEcompsocitemizethanks{
\IEEEcompsocthanksitem The work was supported in part by the National Natural Science Foundation of China under grant 62522604 and 62476258, and in part by the Interdisciplinary Research Program of HUST under grant 2025JCY3024, and in part by the Natural Science Foundation of Hubei Province under grant 2025AFA113. ($\dagger$Corresponding Author: Chang Tang, and $\ddagger$ denotes the project lead, $^{\ast}$ indicates core contributors.)
\IEEEcompsocthanksitem Xin Zou is with the AI Thrust, The Hong Kong University of Science and Technology (Guangzhou) (e-mail: dylan.zoux@gmail.com).
\IEEEcompsocthanksitem Ruimeng Liu is with the School of Computer Science and Technology, Huazhong University of Science and Technology, Wuhan 430074, China (e-mail: lrmhust@icloud.com).
\IEEEcompsocthanksitem Chang Tang is with the School of Software Engineering, Huazhong University of Science and Technology, Wuhan 430074, China (e-mail: tangchang@hust.edu.cn).
\IEEEcompsocthanksitem Zhenglai Li is with Shenzhen Institutes of Advanced Technology, Chinese Academy of Sciences, Shenzhen 518055, China (e-mail: zhenglai.li@siat.ac.cn).
\IEEEcompsocthanksitem Xinwang Liu is with the School of Computer, National University of Defense Technology, Changsha 410073, China (e-mail: xinwangliu@nudt.edu.cn).
\IEEEcompsocthanksitem Kunlun He is with Medical Big Data Research Center, Medical Engineering Laboratory of Chinese PLA General Hospital, Beijing, China (e-mail: kunlunhe@plagh.org).
\IEEEcompsocthanksitem Wanqing Li is with the School of Computing and Information Technology, University of Wollongong, NSW, 2500, Australia (e-mail: wanqing@uow.edu.au).}
}


\markboth{IEEE Transactions on Pattern Analysis and Machine Intelligence}%
{Shell \MakeLowercase{\textit{et al.}}: Bare Advanced Demo of IEEEtran.cls for IEEE Computer Society Journals}
%

\IEEEtitleabstractindextext{%
\begin{abstract}
\justifying  
\revised{Multi-View Clustering (MVC) has gained significant attention for its ability to leverage complementary information across diverse views.
However, existing deep MVC methods often struggle with view-distribution entanglement during cross-view fusion, which hampers the quality of the shared latent space and leads to suboptimal Figures.
To address this issue, we propose the Generalized Multi-view Auto-Encoder (GMAE), a framework designed to preserve cross-view complementarity through disentangled representation learning. Specifically, GMAE employs dual-path autoencoders to decouple source features into view-specific and view-common embeddings, facilitating the discovery of clearer clustering structures. We further construct cross-view adversarial discriminators to guide view-specific encoders in capturing more discriminative features. By strategically modulating mutual information, GMAE effectively aligns distributions and prevents representation collapse, ensuring the generation of robust, non-trivial embeddings. Comprehensive experiments on 13 benchmark datasets demonstrate that GMAE consistently outperforms state-of-the-art methods in both complete and incomplete MVC tasks. Our code implementation is available at the repository: \url{https://github.com/obananas/GMAE}.}

\end{abstract}

}

\maketitle

\IEEEdisplaynontitleabstractindextext

%
\IEEEpeerreviewmaketitle

\ifCLASSOPTIONcompsoc
\IEEEraisesectionheading{\section{Introduction}\label{sec:introduction}}
\else
\section{Introduction}\label{sec:introduction}
\fi
\IEEEPARstart{I}{n} \revised{the era of information explosion, data is often presented in multiple views, each offering a unique perspective on the underlying phenomena \cite{kan2015multi,sridharan2008information,fang2023comprehensive}. For instance, a social media user profile can be represented by textual posts, images, and interaction patterns \cite{ding2017multi}; a medical diagnosis can be based on multi-omics data (\textit{e.g.}, genome, proteome, transcriptome) \cite{boekel2015multi,zou2023hierarchical,zou2023dpnet}, clinical symptoms \cite{zou2024dai}, and multi-modal medical images \cite{zhang2021modality}. The challenge lies in effectively leveraging this multifaceted knowledge to uncover the intrinsic structures and patterns obscured when views are considered in isolation. Multi-view clustering (MVC) is a promising solution to address this challenge by integrating complementary information from diverse sources \cite{trosten2021reconsidering,li2021consensus,xu2023adaptive,11057929}.}

\begin{figure}[t]
  \centering
  \vspace{-1em}
\includegraphics[width=1\linewidth]{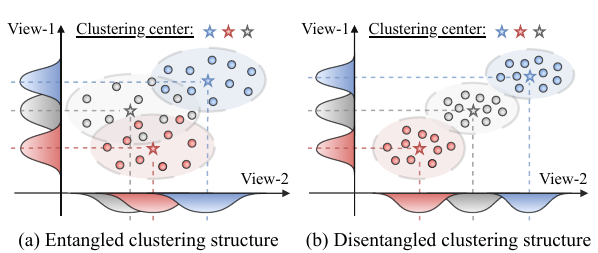}
  \caption{An illustrative example of our motivation. As cluster distributions of different views are distinguished distinctly, previous methods that fuse multi-view entanglement representations for clustering lead to unclear cluster structures. We fix it by learning disentangled representations.}
  \label{figure:exp1} 
  \vspace{-1em}
\end{figure}
\revised{Traditional MVC methods, including matrix factorization techniques \cite{liu2021one,li2024orthogonal,10577554}, graph-based approaches \cite{wang2019gmc,zou2023inclusivity,pan2021multi,tang2022unified}, and subspace learning algorithms \cite{cao2015diversity,gao2015multi,zhang2018generalized}, have made significant strides in harnessing multi-view data for clustering tasks. However, the reliance on handcrafted features and simple fusion strategies generally limits their capability to capture the intricate inter-dependencies and complementary information across views fully\cite{chao2021survey}, especially when confronted with such noise, incompleteness, and inter-cluster entanglement in multi-view data. Thus, some researchers have turned to deep learning-based approaches \cite{li2019deep,tang2022deep,xiao2025dual} that learn robust and discriminative representations from multi-view data to enhance Figures. Deep MVC methods typically focus on integrating view-specific representations into a unified latent space or learning a consensus representation across views, and have demonstrated remarkable success in various unsupervised learning tasks \cite{trosten2023effects,cui2024novel,zhu2024trusted}. Despite their potential, existing deep MVC methods tend to struggle with the issue of learning disentangled representations \cite{ke2024rethinking} that are not only aligned across view distributions but also robust to view-specific noise \cite{ke2024rethinking,chen2023deep}.}

 Existing deep MVC methods generally induce unclear clustering structures due to variations in the distribution of data across views, \textit{i.e.}, the distributions of different views are not aligned and calibrated. We call this problem view distribution entanglement. As shown in Fig. \ref{figure:exp1} (a), alignment without view distribution disentanglement causes the clustering structure to be unclear and the group clusters not tight. We expect to obtain a disentangled and tight clustering distribution as shown in Fig. \ref{figure:exp1} (b).
This paper introduces GMAE, a novel disentangled multi-view clustering framework with cross-view adversarial alignment to address this issue. Briefly, we learn disentangled representations by multi-view Auto-Encoders (AE), which decompose views into specific and common factors to facilitate collaborative multi-view clustering in the next. Further, we propose view-adversarial alignment to obtain intra-view independent and inter-view consistent representations for information collaboration, thus enhancing generalized clustering capability. Experimentally, Fig. \ref{fig:allvis} shows the clustering distribution of several SOTA methods on STL-10 dataset. From Fig. \ref{fig:allvis}, it can be observed that the clustering structure learned by the previous approaches is unclear and discrete. GMAE apparently learns a clear and tight clustering structure. We believe the reason for this gap is that the previous methods failed to implement view distribution alignment with disentangled learning. Experimental results on 13 datasets show the robustness and effectiveness of our strategy. While other SOTA methods have fluctuating Figures on various datasets from different domains due to the lack of alignment with disentangled learning.
 
\begin{figure*}[ht]
\centering
\subfigure[DealMVC{\tiny\textcolor{gray}{[CVPR’21]}}\cite{lin2021completer}] { 
\includegraphics[width=0.378\columnwidth]{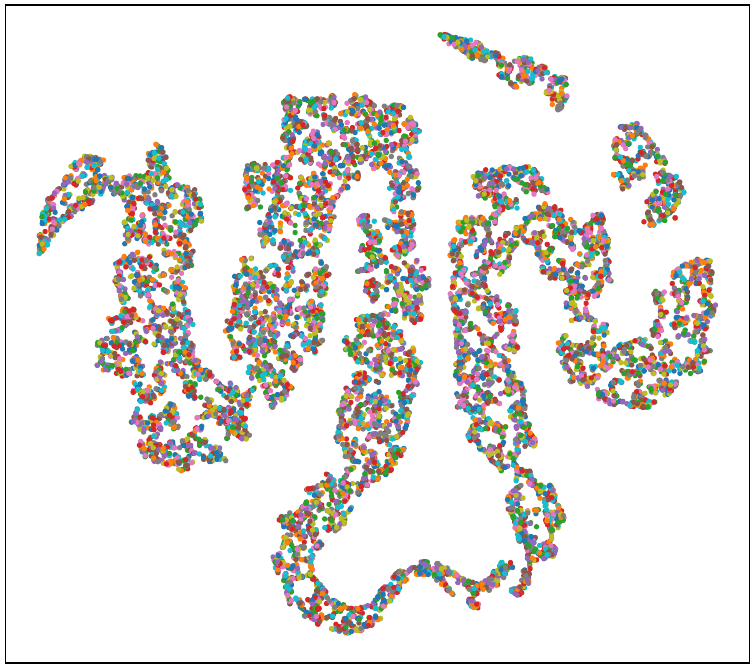}
}
\subfigure[GCFAgg{\tiny\textcolor{gray}{[CVPR’23]}}\cite{Gcfaggyan2023gcfagg}] { 
\includegraphics[width=0.378\columnwidth]{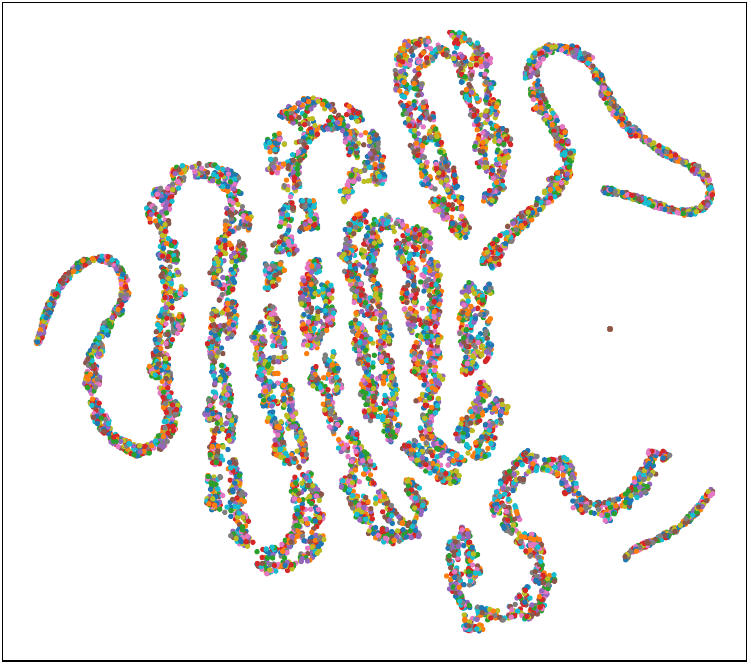}
}
\subfigure[SCMVC{\tiny\textcolor{gray}{[TMM’24]}}\cite{SCMVCwu2024self}] { 
\includegraphics[width=0.378\columnwidth]{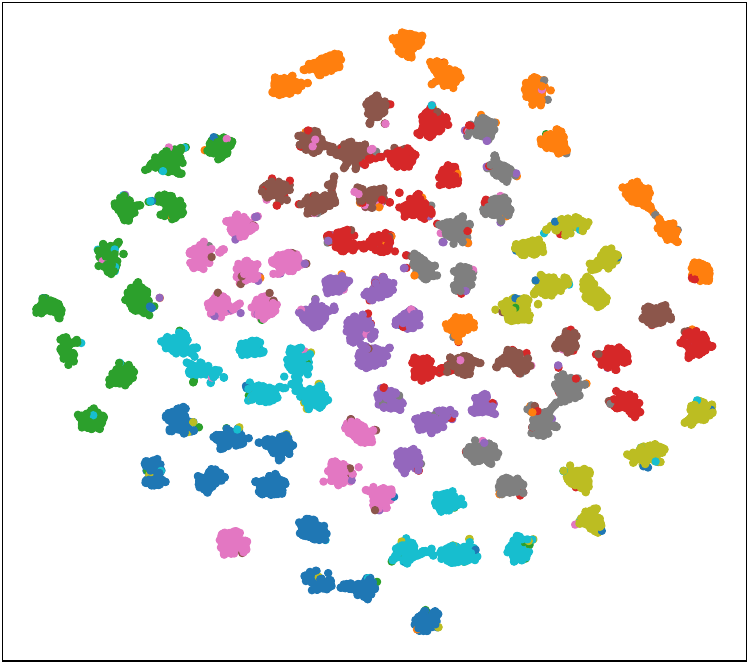}
}
\subfigure[MVCAN{\tiny\textcolor{gray}{[CVPR’24]}}\cite{MVCANxu2024investigating}] { 
\includegraphics[width=0.378\columnwidth]{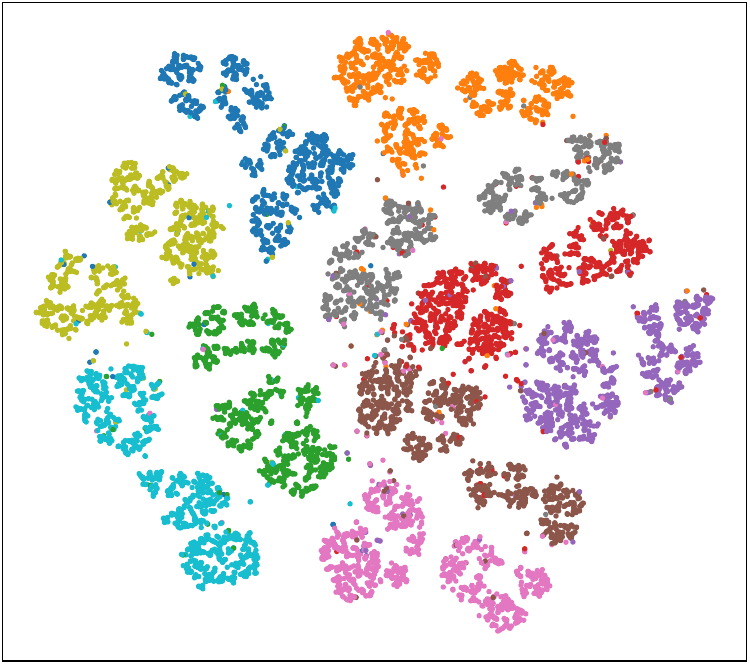}
}
\subfigure[GMAE (ours)] { 
\includegraphics[width=0.375\columnwidth]{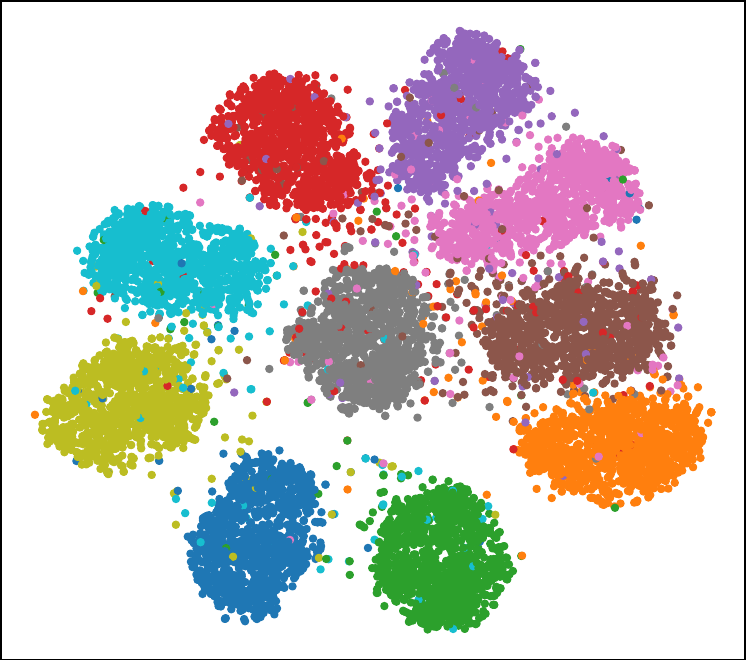}
}\caption{The t-SNE visualization \cite{van2008visualizing} results of feature representations on the STL-10 dataset for different SOTA methods.}
\label{fig:allvis} 
\end{figure*}

Specifically, GMAE adopts a disentangled autoencoder framework to extract view-specific and view-common representations in a self-supervised manner. To further enforce disentanglement, a cross-view adversarial mechanism is introduced. This mechanism facilitates the learning of a common representation that captures the shared structure across views while preserving the distinct characteristics of each view. Additionally, GMAE integrates a disentanglement criterion designed to encourage the learning of independent factors within the representations, thereby enhancing their robustness and discriminative power for clustering tasks. The main contributions can be summarized as follows: \vspace{0.5em}
\begin{itemize}[leftmargin=*]
\item[\ding{182}] To the best of our knowledge, the proposed self-supervised disentangled representation learning framework is the first to address the challenge of view distribution and inter-cluster entanglement in MVC. 
\vspace{0.5em}
\item[\ding{183}] We rigorously derive an optimizable loss function based on the task of mutual information estimation, and provide a theoretical analysis of its capabilities. Our findings demonstrate that the disentangled representations learned by GMAE framework contain more cluster-relevant information and less cluster-irrelevant information.
\vspace{0.5em}
\item[\ding{184}] The proposed framework simultaneously improves the clustering consistency across views and preserves their diversity. Experimental evaluations across different types of public datasets, including text, image, omics, and synthetic data, demonstrate the superiority of our proposed GMAE over state-of-the-art deep MVC methods in both complete and incomplete multi-view clustering tasks.  
\end{itemize}
\vspace{0.5em}
The remainder of this article is organized as follows. In Section \ref{sec:relatedwork}, we review some autoencode-based multi-view clustering methods, representation learning with mutual information and disentangled representation. Section \ref{sec:method} provides a formulation of the problem and introduces the proposed GMAE framework, cross-view distribution alignment strategy, and the analysis of generalization bounds. Extensive experiments and comparisons with SOTA methods are presented in Section \ref{sec:exp}. Section \ref{sec:conclusion} concludes this paper.

\section{Related Works}\label{sec:relatedwork}
This section briefly reviews the topics related to this work, including AE-based clustering, mutual information representation, and disentangled representation learning, and highlights the key differences in the proposed framework.
\subsection{AE-based Multi-view Clustering} 
\revised{Autoencoders, first introduced by Hinton et al. (1993) \cite{hinton1993autoencoders}, use an encoder to approximate the latent distribution of data and a decoder to minimize reconstruction error, laying the foundation for deep learning autoencoders. In clustering, autoencoders have advanced significantly, particularly for high-dimensional data, by learning compressed latent representations that capture non-linear structures difficult for traditional clustering methods. Xie et al. \cite{xie2016unsupervised} integrated autoencoders with distance-based clustering, introducing a stacked framework that iteratively optimized cluster centers. Guo et al. \cite{guo2017improved} improved this approach by adding reconstruction loss to preserve local structures. Jiang et al. \cite{jiang2016variational} further advanced clustering by combining Variational Autoencoders with Gaussian Mixture Models, offering a probabilistic approach for complex data. The Gaussian Mixture VAE \cite{dilokthanakul2016deep} introduced additional random variables to capture intricate data relationships. Xu et al. \cite{xu2021multi} enhanced clustering in multi-view data by using separate latent variables to capture common and unique information. Meanwhile, generative adversarial networks \cite{goodfellow2014generative,goodfellow2020generative,corvi2023intriguing} gained popularity. For example, Ge et al. \cite{ge2019dual} used adversarial networks to regularize the latent space for improved clustering. Further, SparseMVC \cite{liu2025sparsemvc} incorporated sparsity constraints and sparse autoencoders into the latent representation learning process to filter out noise and redundant information.}

Different from these AE-based MVC methods, we propose a disentangled AE framework designed to separate view-specific and view-common representation learning. 

\subsection{Representation with Mutual Information} 
Mutual information (MI) is essential in multi-view learning, especially for representation learning and deep clustering \cite{chen2022representation}. By maximizing shared information across views, MI enhances data representations, improving clustering and classification. Bell and Sejnowski \cite{bell1995information} pioneered an information maximization approach for blind source separation, which influenced later work. Hjelm et al. \cite{hjelm2018learning} expanded on this in the Deep InfoMax framework, where maximizing MI between input data and encoded representations boosted representation learning \cite{bachman2019learning}. Bachman et al. further optimized MI to extract shared semantic features for multi-view learning. In multi-view clustering, Mao et al. \cite{mao2021deep} introduced a deep MI maximin method to improve clustering by maximizing shared and minimizing redundant information, with variational optimization ensuring convergence. Zhang et al. \cite{zhang2023multi} proposed triplex information maximization to strengthen cross-view relationships for more consistent clustering. Recently, Lu et al. \cite{lu2024decoupled} used high-order random walk contrastive learning to enhance MI, improving multi-view clustering for complex datasets, while Zhang et al. \cite{zhou2024mcoco} introduced a multi-level consistency collaborative approach to enhance cross-view MI. Unlike previous methods that maximize MI between different views through contrastive learning (CL) losses, we derive a set of optimizable losses for both maximizing MI between views and minimizing MI between view-specific and view-common representations.

\begin{figure*}[htbp]
  \centering
  \includegraphics[width=1\linewidth]{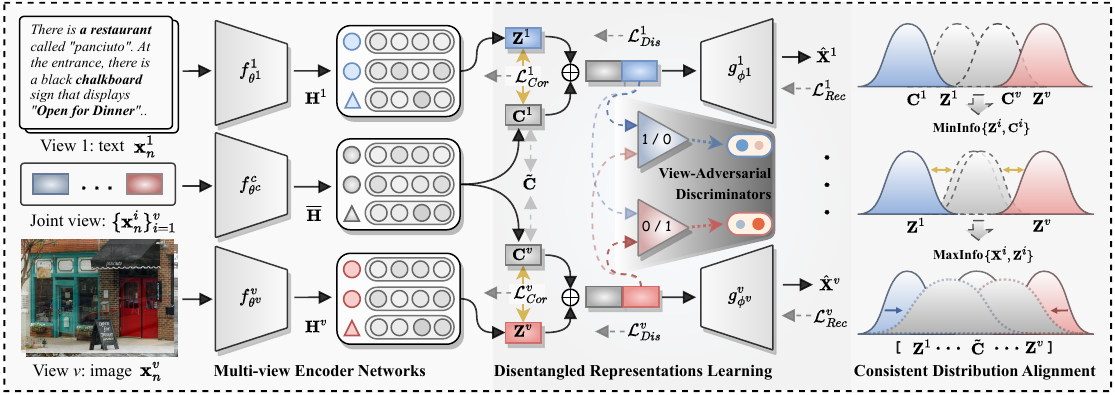}
  \caption{The flowchart of our proposed GMAE. Specifically, given the multi-view feature matrix, GMAE first employs the Multi-Layer Perceptron (MLP) as generators to encode view-specific and -common representations (\textit{i.e.}, $\mathbf{Z}^{v}$ and $\mathbf{C^{v}}$) for all views. After that, GMAE investigates the view-adversarial discrimination loss $\mathcal{L}_{Dis}^v$ and correlation loss $\mathcal{L}_{Cor}^{v}$ to obtain clean and disentangled information. Meanwhile, GMAE envisages the reconstruction loss $\mathcal{L}_{Rec}^{v}$ and self-consistent alignment loss $\mathcal{L}_{Con}$ to preserve a more discriminative and disentangled representation for the final clustering task.}
  \label{figure:framework} 
\end{figure*}
\subsection{Disentangled Representation Learning} 
\revised{Disentangled Representation Learning focuses on breaking down complex data into independent latent factors to enhance model interpretability, controllability, and generalization. Higgins et al. \cite{higgins2017beta} introduced the foundational Beta-VAE framework, which regulates KL divergence to enforce disentanglement, improving latent factor independence. Burgess et al. \cite{burgess2018understanding} refined Beta-VAE by balancing reconstruction quality and independence, while Locatello et al. \cite{locatello2019challenging} revealed the limitations of unsupervised disentanglement, leading to supervised and weakly supervised models. In adversarial learning, Chen et al. \cite{chen2016infogan} combined adversarial training with self-supervision to separate latent factors via contrastive loss. In generative models, approaches like Yang et al. \cite{yang2021causalvae} and Lee et al. \cite{lee2020high} excelled in image synthesis and causal disentanglement by balancing disentanglement with high-fidelity generation. Multi-view disentangled learning saw advancements with the VMI-VAE framework by Serdega and Kim \cite{serdega2020vmi}, which maximizes mutual information for improved latent representations, and Fan et al. \cite{fan2022debiasing}, who used Graph Neural Networks to disentangle node features from structural data. Dynamic models based on autoencoders and variational inference \cite{MFLVCxu2022multi} further adapted disentanglement levels to handle diverse multi-view datasets. Besides, Multi-VAE \cite{xu2021multi} separates view-common features from view-peculiar representations, ensuring that the clustering process is driven by consistent cross-view information while preserving complementarity.}

In this work, we design a novel self-supervised representation learning framework for MVC that enforces view distribution alignment and disentanglement, while learning decoupled view-specific and view-common knowledge. A detailed discussion of existing state-of-the-art MVC methods can be found in Section \ref{sec:sota}, where we present detailed explanations of their key designs and operational mechanisms. This comparative analysis shows how our framework differs from others and our contributions.
\section{The proposed Method}
\label{sec:method}
\textit{Problem Statement}. Given a multi-view dataset $\{\boldsymbol{x}_i^1\in\mathbb{R}^{d_1},\boldsymbol{x}_i^2\in\mathbb{R}^{d_2},\dots,\boldsymbol{x}_i^V\in\mathbb{R}^{d_V}\}_{i=1}^N$, each sample contains $V$ views with complementary, $N$ is sample size, and $d_1, d_2, \dots, d_V$ are the dimension of views. Multi-view clustering aims to separate these samples into $K$ clusters.

The overall GMAE framework is illustrated in Figure \ref{figure:framework}. The proposed GMAE first learns the disentangled view-specific representations $\mathbf{Z}^v\in\mathbb{R}^{N\times D_z}$ and view-common representations $\mathbf{C}^{v}\in\mathbb{R}^{N\times D_c}$ through common representation $\mathbf{C}_{\ast}^{v}\in\mathbb{R}^{N\times D_c}$, where $D_z$ and $D_c$ denote the dimensions of feature space. Then, we use the comprehensive representations $\mathbf{Q}\in\mathbb{R}^{N\times (D_c+V\times D_z)}$ to perform clustering.

\subsection{Multi-view Disentangled Representation Learning}
Raw data of multiple views are inherently noise-perturbed and naturally carry inter-cluster entanglement. We first learn representative feature representations for each view from the original distributions using autoencoders \cite{hinton2006reducing,yang2019deep,li2021adaptive}. As our motivation is to learn clear and disentangled multi-view representations, we further introduce a shared encoder to explore view-common information. Specifically, let $f_{\theta^v}^v$ represents the encoder of $v$-th view, $f_{\theta^c}$ denotes the cross-view shared encoder, and $\{g_{\phi^v}^v\}_{v=1}^V$ are the decoders for view-specific reconstruction. We obtain coarse view-specific and view-common feature representations by
\begin{equation}
\boldsymbol{h}_i^v=f_{\theta^v}^v(\boldsymbol{x}_i^v), \;\; \overline{\boldsymbol{h}}_i=f_{\theta^c}(\boldsymbol{x}_i^v), 
\end{equation}
where $\overline{\mathbf{h}}_i$ represents cluster information of all views. To facilitate the disentangling of view-specific and -common features from encoding, we employ MLP to map $\mathbf{h}_i^v$ and $\overline{\boldsymbol{h}}_i$ into specific and common representations (\textit{i.e.}, $\boldsymbol{z}^v$ and $\boldsymbol{c}^v$). The decoder reconstructs the sample of $i$-th view by
\begin{equation}
\hat{\boldsymbol{x}}_i^v=g_{\phi^v}^v(\boldsymbol{e}_i^v)=g_{\phi^v}^v([\boldsymbol{z}_i^v\oplus \boldsymbol{c}_i^v]),
\end{equation}
where $\oplus$ denotes concatenation operation. To enhance the quality of the reconstructed view, we expect to preserve the essential information of the original feature of $v$-th view in the compact representation $\mathbf{E}^{v}=[\boldsymbol{e}_1^v,\dots,\boldsymbol{e}_N^v]$, \textit{i.e.}, to maximize the mutual information (MI) \cite{lin2021completer} between $\mathbf{X}^v$ and $\mathbf{E}^v$ by the following formulae,
\begin{equation}
   \max I(\mathbf{X}^v,\mathbf{E}^v) = \sum_{\boldsymbol{x}_i^v}\sum_{\boldsymbol{e}_i^v} p(\boldsymbol{x}_i^v, \boldsymbol{e}_i^v) \log\frac{p(\boldsymbol{x}_i^v \mid \boldsymbol{e}_i^v)}{p(\boldsymbol{x}_i^v) },
\end{equation}
where $p(\cdot)$ denotes the probability distribution. Let $q(\boldsymbol{x}_i^v\mid \boldsymbol{e}_i^v)$ as the variational estimation of $p(\boldsymbol{x}_i^v \mid \boldsymbol{e}_i^v)$, we compute Kullback-Leibler (KL) divergence as follows, 
\begin{align}
    D_{KL}(P \| Q) = &\sum_{\boldsymbol{x}_i^v} p(\boldsymbol{x}_i^v \mid \boldsymbol{e}_i^v) \log\frac{p(\boldsymbol{x}_i^v \mid \boldsymbol{e}_i^v)}{q(\boldsymbol{x}_i^v\mid \boldsymbol{e}_i^v)} \geq 0 \Rightarrow\\
     \sum_{\boldsymbol{x}_i^v}&\sum_{\boldsymbol{e}_i^v} p(\boldsymbol{x}_i^v, \boldsymbol{e}_i^v) \log\frac{p(\boldsymbol{x}_i^v \mid \boldsymbol{e}_i^v)}{p(\boldsymbol{x}_i^v)} \geq \nonumber \\
    \sum_{\boldsymbol{x}_i^v}&\sum_{\boldsymbol{e}_i^v} p(\boldsymbol{x}_i^v, \boldsymbol{e}_i^v) \log\frac{q(\boldsymbol{x}_i^v \mid \boldsymbol{e}_i^v)}{p(\boldsymbol{x}_i^v)}.
\label{eq:6}
\end{align}
Suppose $q(\cdot)$ is Gaussian density distribution, we can derive
\begin{equation}
\begin{aligned}
    I(\mathbf{X}^v,\mathbf{E}^v) &\geq \sum_{\boldsymbol{x}_i^v}\sum_{\boldsymbol{e}_i^v} p(\boldsymbol{x}_i^v, \boldsymbol{e}_i^v) \log q(\boldsymbol{x}_i^v \mid \boldsymbol{e}_i^v) \geq \\
    \sum_{\boldsymbol{x}_i^v}\sum_{\boldsymbol{e}_i^v} p(\boldsymbol{x}_i^v&, \boldsymbol{e}_i^v) \log \frac{1}{\sqrt{2 \pi} \sigma} \exp (-\frac{\|\boldsymbol{x}_i-g_{\phi^v}^v(\boldsymbol{e}_i^v)\|_2^2}{\sigma^2}).
\end{aligned}
\end{equation}
Thus, the solution to maximizing $I(\mathbf{X}^v,\mathbf{E}^v)$ is transformed into minimizing the lower bound loss 
\begin{equation}
\mathcal{L}_{Rec}^v=\sum_{i=1}^{N}\|\boldsymbol{x}_i^v-g_{\phi^v}^v (\boldsymbol{e}_i^v)\|_{2}^2 ,
\label{eq:7}
\end{equation}
Further, in order to decouple view-specific and -common representations, we need to constrain the statistical independence between them. We obtain disentangled view-specific and view-common representations by the minimization of the MI between them. In particular, we optimize the MI between $\mathbf{Z}^v$ and $\mathbf{C}^v$ through the correlation loss, as follows,
\begin{equation}
\mathcal{L}_{Cor}^v=\sum_{i=1}^{N}\| \boldsymbol{z}_i^v (\boldsymbol{c}_i^v-\mu_i^v)^{\top}\|_{1},
\label{eq:8}
\end{equation}
where $\mu_i^v$ denotes the mean value vector of $\boldsymbol{c}_i^v$. The above formula is a simplified loss, we aim to minimize the covariance between $\boldsymbol{z}_i^v$ and $\boldsymbol{c}_i^v$, to make $\mathbf{Z}^v\mathbf{C}^{v\top}$ an all-zero matrix, \textit{i.e.}, $\min\|\mathbf{Z}^v\mathbf{C}^{v\top}\|_{0} \; s.t.\; \mathbf{C}^{v} = \mathbf{C}_{\ast}$, where $\mathbf{C}_{\ast}$ denotes optimal common representation. As constraining $l_0$-norm is an NP-hard problem, thus we actually relax it to $l_1$-norm.
\subsection{Consistent Cross-view Adversarial Alignment}
To align the distributions across different views and eliminate their inconsistencies, we design a cross-view adversarial strategy for consistent alignment.
In contrast to conventional alignment strategies, we aim to pull the distributions of different views to a consistent distribution by alternating adversarial learning. In this process, the inherent clustering structure in multi-view data is explored through view-specific discriminators, which produce discriminative and clean latent representations. 
Concretely, for view-specific representations $\{\mathbf{Z}^{v}\}_{v=1}^V$, each view of a sample $\boldsymbol{z}_i^v$ follows a distribution $\boldsymbol{z}_i^v \sim P({Z}^v)$. As with \cite{goodfellow2014generative,yang2024trustworthy,yan2021zeronas}, we use view-specific discriminators $\mathcal{D}_{\kappa^v}^{v}$ to distinguish the identical view across different views. For the $v$-th view, $\boldsymbol{z}_i^v$ is set as the real representation, and $\boldsymbol{z}_i^u (u \neq v)$ is set as the fake representation. 
View-adversarial discriminators are constructed to distinguish cross-view real and fake representations as follows:
\begin{equation}
\mathcal{D}_{\kappa^v}^{v} (\{\boldsymbol{z}_i^u\}_{u=1}^V)=
\begin{cases} 
1,  & \mbox{if }u=v. \\
0, & \mbox{if }u\neq v.
\end{cases}
\end{equation}
Then, all view-specific discriminators are trained alternately to obtain cross-view consistent representations, which are indistinguishable from the discriminators. Mathematically, this process can be formulated by optimizing a min-max generative adversarial loss as follows:
\begin{equation}
\begin{aligned}
\min_{\theta^u} \max_{\kappa^v} \;\, & E_{\boldsymbol{z}_i^v \sim P(Z^v)}[\log \mathcal{D}_{\kappa^v}^v(\boldsymbol{z}_i^v)] \\
 &+E_{\boldsymbol{z}_i^u \sim P(Z^u)}[\log (1-\mathcal{D}_{\kappa^v}^v(\boldsymbol{z}_i^u)],
\end{aligned}
\end{equation}
\begin{equation}
\mathcal{L}_{Dis}^v=-\sum_{i=1}^{N} log(\mathcal{D}_{\kappa^v}^v(\boldsymbol{z}_i^v)) + log(1-\mathcal{D}_{\kappa^v}^v(\boldsymbol{z}_i^u)),
\label{eq:11}
\end{equation}
\textbf{Self-Correlation Constrain}. With disentangled multi-view representations, we form the final representation via $\mathbf{Q}=[\mathbf{C}_{\ast},\{\mathbf{Z}^v\}_{v=1}^V]$. As clustering tasks lack labels for guided representation learning, thus weakening the correlation between learned representations and labels. To mitigate this, we enhance mutual information (MI) between the embedding space and cluster structure. Maximizing the MI among sample points and their neighbors promotes intra-cluster cohesion and inter-cluster separation. Mathematically, this $N_{\omega}$-nearest neighbors enhancement is expressed as:
\begin{equation}
\max I(\mathbf{Q}, \mathbf{Q}_{\omega})=\frac{1}{N} \sum_{i=1}^N \frac{1}{N_{\omega}} \sum_{k=1}^{N_{\omega}} \log \frac{p(\mathbf{q}_k \mid \mathbf{q}_i)}{p(\mathbf{q}_k)},
\label{eq:12}
\end{equation}
where $N_{\omega}$ denotes neighbors. Motivated by InfoNCE \cite{oord2018representation} and MIMC \cite{zhang2023mutual}, we estimate $p(\mathbf{q}_k \mid \mathbf{q}_i)/p(\mathbf{q}_k)$ by $m (\mathbf{q}_k , \mathbf{q}_i)$.
\begin{theorem}
Minimizing cross-entropy loss of $m (\mathbf{q}_k, \mathbf{q}_i)$ is equivalent to maximizing the MI between $\mathbf{Q}$ and $\mathbf{Q}_{\omega}$ as Eq. \ref{eq:12}.
\label{the:t1}
\end{theorem}
\begin{proof}
The optimization of cross-entropy loss of $m (\mathbf{q}_k, \mathbf{q}_i)$ is
\begin{equation}
\begin{aligned}
 \mathcal{L}_{Ent}=&-\frac{1}{N} \sum_{i=1}^N \sum_{j=1}^{N_{\omega}} \log \frac{m(\mathbf{q}_j, \mathbf{q}_i)}{\sum_{k=1}^N m(\mathbf{q}_k, \mathbf{q}_i)} \\
 =&-\frac{1}{N} \sum_{i=1}^N \sum_j^{N_{\omega}}\left[\log \frac{p(\mathbf{q}_j \mid \mathbf{q}_i)}{p(\mathbf{q}_j)}-\log \sum_{k=1}^N \frac{p(\mathbf{q}_k \mid \mathbf{q}_i)}{p(\mathbf{q}_{k=1})}\right]\\
&-\frac{1}{N} \sum_{i=1}^N \sum_{j=1}^{N_{\omega}} \log \frac{p(\mathbf{q}_j \mid \mathbf{q}_i)}{p(\mathbf{q}_j)}+N_{\omega} \log N \\
=&-N_{\omega} I(\mathbf{Q}, \mathbf{Q}_{\omega})+N_{\omega} \log N,
\end{aligned}
\label{eq:13}
\end{equation}
we can observe that when $\mathcal{L}_{Ent}$ reaches the minimum, $I(\mathbf{Q}, \mathbf{Q}_{\omega})$ takes the maximum. The proof is completed.
\end{proof}
Thus, we can optimize Eq. \ref{eq:13} to maximize $I(\mathbf{S}, \mathbf{S}_{\omega})$. In particular, we employ the sampling manner to estimate $\sum_{k=1}^N m(\mathbf{q}_k, \mathbf{q}_i)$  for simplified calculations, as follows
\begin{equation}
    \sum_{k=1}^{\hat{N}} m(\mathbf{q}_k, \mathbf{q}_i)=\hat{N}\cdot\mathbb{E}_{k}(m(\mathbf{q}_k, \mathbf{q}_i))=\frac{\hat{N}}{N}\cdot \sum_{k=1}^{N} m(\mathbf{q}_k, \mathbf{q}_i),
\end{equation}
where $\hat{N}$ denotes the number of sampling instances. In this way, Eq. \ref{eq:13} can be reformulated in the following form:
\begin{equation}
    \mathcal{L}_{Ent}=-\frac{1}{N} \sum_{i=1}^N \sum_{j=1}^{N^{+}} \log \frac{m(\mathbf{q}_j, \mathbf{q}_i)}{\sum_{k=1}^{N^{-}} m(\mathbf{q}_k, \mathbf{q}_i)},
    \label{eq:15}
\end{equation}
where $N^{+}$ and $N^{-}$ are the number of positive and negative samples respectively, while $N^{+}=N_{\omega}$ and $N^{-}=\hat{N}$. Specifically, we restrict $N^{-}=N-N^{+}$, \textit{i.e.}, samples other than positive instances are negative instances, and the function $m(\mathbf{q}_j, \mathbf{q}_i)$ is obtained by formula $\exp{(\frac{\mathbf{q}_j^{\top} \cdot \mathbf{q}_i}{\|\mathbf{q}_j\| \cdot \|\mathbf{q}_i\|})}$.
Suppose the optimal view-common representations as $\mathbf{C}_{\ast}$, which
contains clean and intact common knowledge, we have Theorem \ref{the:t2} on view-common information extraction. 

\begin{theorem}
Assume the solution of view-common representations that satisfies the constraints in Eq. \ref{eq:7}, Eq. \ref{eq:8} and \ref{eq:11} have been found, then we have $\mathbf{C}^{v}=p_c^{v} \circ f_{\theta^c}([\mathbf{X}^{v}]_{v=1}^V)=\varphi(\mathbf{C}_{\ast})$ for $\forall v \in[1, {V}]$, where $\varphi$ is an invertible function.
\label{the:t2}
\end{theorem}
\begin{proof}
Consider the constrain in Eq. \ref{eq:8}, which indicates that view-common representations $\mathbf{C}^{v}$ are aligned, \textit{i.e.}, $\mathbf{C}^{v}=\mathbf{C}^{u} (v \neq u)$, and we have:
$p_c^{v} \circ f_{\theta^c}(\mathbf{X}^{v})=p_c^{u} \circ f_{\theta^c}(\mathbf{X}^{u})$.

Since the solution satisfies the reconstruction loss in \ref{eq:7} and the correlation loss in \ref{eq:8}, the view-common representation $\mathbf{C}^{v}$ and the view-specific representation $\mathbf{Z}^{v}$ must be statistically independent, and the function $p_{c}^{v} \circ f_{\theta^c}$ must be invertible. Let
$q_{c}^{v} = f_{\theta^c}^{(-1)} \circ p_{c}^{v^{(-1)}}$
denote its inverse. Suppose $\mathbf{C}_*$ and $\mathbf{Z}_*^{v}$ are optimal view-common and view-specific representations, respectively, which contain all necessary and noise-free information. As these are statistically independent, we transform $p_c^{v} \circ f_{\theta^c}(\mathbf{X}^{v})=p_c^{u} \circ f_{\theta^c}(\mathbf{X}^{u})$ to:
\begin{equation}
q_{c}^{v}\!\left(\!
\begin{bmatrix}
\mathbf{C}_*\\[5pt]
\mathbf{Z}_*^{v}
\end{bmatrix}
\right)=
q_{c}^{u}\!\left(\!
\begin{bmatrix}
\mathbf{C}_*\\[5pt]
\mathbf{Z}_*^{u}
\end{bmatrix}
\right).
\label{eq:16}
\end{equation}
\noindent
\emph{Key Idea.} To show $f_{\theta^c}$ can
extract intact and clean view-common information, we need to prove $q_c^{v}$ depends only on $\mathbf{C}_*$ (and not on $\mathbf{Z}_*^{v}$). Define
\begin{equation}
q^{v}(\boldsymbol{\vartheta}^{v})
=
\begin{bmatrix}
q_{c}^{v}(\boldsymbol{\vartheta}^{v}) \\[3pt]
q_{p}^{v}(\boldsymbol{\vartheta}^{v})
\end{bmatrix},
\quad
\boldsymbol{\vartheta}^{v}
=
\begin{bmatrix}
\mathbf{C}_*\\[3pt]
\mathbf{Z}_*^{v}
\end{bmatrix}.
\end{equation}
The Jacobian of $q^{v}$ is block-partitioned:
\begin{equation}
\boldsymbol{J}^{v}
=
\begin{bmatrix}
\boldsymbol{J}_{11}^{v} & \boldsymbol{J}_{12}^{v} \\[3pt]
\boldsymbol{J}_{21}^{v} & \boldsymbol{J}_{22}^{v}
\end{bmatrix},
\end{equation}
where $\boldsymbol{J}_{12}^{v}$ corresponds to the partial derivatives of $q_{c}^{v}$ with respect to $\mathbf{Z}_*^{v}$. We will show $\boldsymbol{J}_{12}^{v}=\mathbf{0}$ and $\det(\boldsymbol{J}_{11}^{v})\neq 0$, implying $q_{c}^{v}$ depends only on $\mathbf{C}_*$.

\noindent
\emph{Zero Block Proof.} From Eq.~\ref{eq:16}, bt fixing any $\overline{\mathbf{C}}_*$ and $\overline{\mathbf{Z}}_*^{u}$ but letting $\mathbf{Z}_*^{v}$ vary, we have
\begin{equation}
q_c^{v}\left(\!
\begin{bmatrix}
\overline{\mathbf{C}}_*\\[3pt]
\mathbf{Z}_*^{v}
\end{bmatrix}
\right)
=
q_c^{u}\left(\!
\begin{bmatrix}
\overline{\mathbf{C}}_*\\[3pt]
\overline{\mathbf{Z}}_*^{u}
\end{bmatrix}
\right),
\end{equation}
which forces $\boldsymbol{J}_{12}^{v} = \mathbf{0}$. Consequently,
\begin{equation}
\boldsymbol{J}^{v}
=
\begin{bmatrix}
\boldsymbol{J}_{11}^{v} & \mathbf{0}\\[3pt]
\boldsymbol{J}_{21}^{v} & \boldsymbol{J}_{22}^{v}
\end{bmatrix},
\det(\boldsymbol{J}^{v})
=
\det(\boldsymbol{J}_{11}^{v})
\;\det(\boldsymbol{J}_{22}^{v})
\;\neq\;0.
\end{equation}
This implies $\det(\boldsymbol{J}_{11}^{v})\neq 0$ and so $q_{c}^{v}$ is only a function of $\mathbf{C}_*$. Hence, we can write for any $v\in [1,V]$,
\begin{equation}
\mathbf{C}^{v}=\varphi(\mathbf{C}_*),
\end{equation}
where $\varphi$ is invertible due to $\det(\boldsymbol{J}_{11}^{v}) \neq 0$. The proof ends.
\end{proof}
Theorem \ref{the:t2} indicates that if the solution satisfies the constraints in Eq. \ref{eq:7}, Eq. \ref{eq:8} and \ref{eq:11}, the common representations learned by GMAE and the optimal common representations can be transformed from each other due to the invertibility of the function. Therefore, the common representations $\mathbf{C}_{\ast}$ have all the information of the optimal common representations and thus extract complete and clean common information provably. As a result, we disentangle the common and view-specific representations to obtain intact and clean common knowledge. Fig. \ref{fig:tsne1} also shows that GMAE learned a disentangled and clear clustering distribution.

\begin{theorem}\label{the:t3}
For clustering task $T$, the comprehensive multi-view representations $\widehat{\mathbf{H}}^v$ contain more cluster-relevant information and less cluster-irrelevant information than $\widetilde{\mathbf{H}}^v$, \textit{i.e.},
\begin{equation}
    I(\widehat{\mathbf{H}}^v, T) \geq I(\widetilde{\mathbf{H}}^v, T), \quad
    H(\widehat{\mathbf{H}}^v \mid T) \leq H(\widetilde{\mathbf{H}}^v \mid T),
\end{equation}
where $I(\widehat{\mathbf{H}}^v, T)$ is the mutual information (MI) between $\widehat{\mathbf{H}}^v$ and $T$, and $H(\widehat{\mathbf{H}}^v \mid T)$ is the entropy of $\widehat{\mathbf{H}}^v$ conditioned on $T$.
\end{theorem}

\begin{proof}
To aid in the proof of Theorem \ref{the:t3}, we first introduce the following Lemma:
\begin{lemma} \label{lemma:1} \textit{(From Lemma B.1 in \cite{mo2023disentangled})}
For clustering task $T$, let $\widehat{\mathbf{H}}^{v}$ be the view representations obtained by GMAE, and $\widetilde{\mathbf{H}}^{v}$ learned by CL-based methods which maximize $I(\widetilde{\mathbf{H}}^v, \widetilde{\mathbf{H}}^u)$, have:
\begin{equation} \label{eq:22}
    I(\widehat{\mathbf{H}}^{v}, \mathbf{X}^{u}, T)
    =I(\widetilde{\mathbf{H}}^{v}, \mathbf{X}^{u}, T)
    =I(\mathbf{X}^{v}, \mathbf{X}^{u}, T),
\end{equation}
\begin{equation} \label{eq:23}
    H(\widehat{\mathbf{H}}^{v}) - H(\widetilde{\mathbf{H}}^{v})
    = H(\widehat{\mathbf{H}}^{v} \mid \mathbf{X}^{u}) - H(\widetilde{\mathbf{H}}^{v} \mid \mathbf{X}^{u}).
\end{equation}
\end{lemma}  
Now we can prove Theorem \ref{the:t3}. We first prove that $I(\widehat{\mathbf{H}}^{v}, T) \geq I(\widetilde{\mathbf{H}}^{v}, T)$ holds. 
Define the complementary information learned by GMAE and the previous CL-based method, w.r.t.\ $\mathbf{X}^u$ as $
I(\widehat{\mathbf{H}}^{v}, T \mid \mathbf{X}^{u})$ and
$I(\widetilde{\mathbf{H}}^{v}, T \mid \mathbf{X}^{u})$.
The objective of contrastive learning loss achieves its minimum, the complementarity in each view is preserved, yielding 
\begin{equation}
I(\widehat{\mathbf{H}}^{v}, T \mid \mathbf{X}^{u}) \;\ge\; I(\widetilde{\mathbf{H}}^{v}, T \mid \mathbf{X}^{u}).
\end{equation}
With chain rule of the mutual information and Lemma \ref{lemma:1} (Eq. \ref{eq:22}), we have the following derivation
\begin{equation}
\begin{aligned}
I(\widehat{\mathbf{H}}^v, T)
&= I(\widetilde{\mathbf{H}}^v, T, \mathbf{X}^u) + I(\widehat{\mathbf{H}}^v, T \mid \mathbf{X}^u) \\
&= I(\widetilde{\mathbf{H}}^v, T) - I(\widetilde{\mathbf{H}}^v, T \mid \mathbf{X}^u) 
+ I(\widehat{\mathbf{H}}^v, T \mid \mathbf{X}^u) .
\end{aligned}
\end{equation}
\revised{Based on $I(\widehat{\mathbf{H}}^{v}, T \mid \mathbf{X}^{u}) \;\ge\; I(\widetilde{\mathbf{H}}^{v}, T \mid \mathbf{X}^{u})$, we can obtain $I(\widehat{\mathbf{H}}^v, T) \ge I(\widetilde{\mathbf{H}}^v, T)$.
Next, let the noisy information be $
H(\widehat{\mathbf{H}}^v \mid \mathbf{X}^u, T)$ and $H(\widetilde{\mathbf{H}}^v \mid \mathbf{X}^u, T)$.
As CL loss ensures cross-view consistency, thus minimizing noise among views, we have
$H(\widehat{\mathbf{H}}^v \mid \mathbf{X}^u, T)
\le H(\widetilde{\mathbf{H}}^v \mid \mathbf{X}^u, T)$.
By Lemma \ref{lemma:1} (Eq. \ref{eq:23}) and some properties of mutual information and entropy, we have the following derivation}
\begin{equation}
\begin{aligned}
H&(\widehat{\mathbf{H}}^{v} \mid T) =H(\widehat{\mathbf{H}}^{v})-I(\widehat{\mathbf{H}}^{v}, T) \\
=&H(\widehat{\mathbf{H}}^{v})-\left[I(\widehat{\mathbf{H}}^{v}, T, \mathbf{X}^u)+I(\widehat{\mathbf{H}}^{v}, T \mid \mathbf{X}^u)\right] \\
=&H(\widehat{\mathbf{H}}^{v})-I(\widetilde{\mathbf{H}}^{v}, T, \mathbf{X}^u)-I(\widehat{\mathbf{H}}^{v}, T \mid \mathbf{X}^u) \\
=&H(\widehat{\mathbf{H}}^{v})-I(\widetilde{\mathbf{H}}^{v}, T)+I(\widetilde{\mathbf{H}}^{v}, T \mid \mathbf{X}^u)-I(\widehat{\mathbf{H}}^{v}, T \mid \mathbf{X}^u) \\
=&H(\widehat{\mathbf{H}}^{v})-\left[H(\widetilde{\mathbf{H}}^{v})-H(\widetilde{\mathbf{H}}^{v} \mid T)\right]+\\
&I(\widetilde{\mathbf{H}}^{v}, T \mid \mathbf{X}^u)-I(\widehat{\mathbf{H}}^{v}, T \mid \mathbf{X}^u) \\
=&H(\widetilde{\mathbf{H}}^{v} \mid T)+H(\widehat{\mathbf{H}}^{v})-H(\widetilde{\mathbf{H}}^{v})\\
&-H(\widetilde{\mathbf{H}}^{v}\mid \mathbf{X}^u, T)-H(\widetilde{\mathbf{H}}^{v} \mid \mathbf{X}^u)-H(\widehat{\mathbf{H}}^{v}\mid \mathbf{X}^u, T) \\
=&H(\widetilde{\mathbf{H}}^{v} \mid T)-H(\widetilde{\mathbf{H}}^{v}\mid \mathbf{X}^u, T)+H(\widehat{\mathbf{H}}^{v} \mid \mathbf{X}^u, T),
\end{aligned}
\end{equation}
and combining with $I(\widehat{\mathbf{H}}^v, T)$ similarly as above, we have 
\begin{equation}
H(\widehat{\mathbf{H}}^v \mid T) 
\;\le\;
H(\widetilde{\mathbf{H}}^v \mid T).
\end{equation}
Thus, two inequalities hold, we complete the proof.
\end{proof}

Based on Theorem \ref{the:t3}, the common and view-specific representations learned by our method preserve more cluster-relevant and less cluster-irrelevant information than those learned by previous CL-based approaches.
\subsection{Objective Function}
Integrating the reconstruction loss in Eq. \ref{eq:7}, the correlation
loss in Eq. \ref{eq:8}, the generative adversarial loss in Eq. \ref{eq:11}, with the
cross-entropy loss in Eq. \ref{eq:15}, the objective function of our
proposed GMAE is formulated as:
\begin{equation}
\mathcal{J}=\sum_{v=1}^V\mathcal{L}_{Rec}^v+\alpha\times (\mathcal{L}_{Cor}^v+\mathcal{L}_{Dis}^v)+\beta\times\mathcal{L}_{Ent}\;,
\label{eq:16overloss}
\end{equation}
where $\alpha$ and $\beta$ are non-negative hyper-parameters to balance the regularization terms. The optimization of overall loss $\mathcal{J}$ makes multi-view autoencoders learn more discriminative features.  Algorithm 1 summarizes the flow of GMAE.
\begin{algorithm}[ht]
\renewcommand{\algorithmicrequire}{\textbf{Input:}}
	\renewcommand{\algorithmicensure}{\textbf{Output:}}
		\caption{Disentangled Representations for GMAE}
		\label{algo}
		\begin{algorithmic}[1]
			\REQUIRE Multi-view data $\{\mathbf{X}^{v}\}_{v=1}^{V}$, cluster number $K$, hyper- parameters $\alpha$, $\beta$, and number of training epochs $E_{train}$.
			\ENSURE Multi-view comprehensive representation $\mathbf{Q}$.
			\STATE Initialize learning rate of $10^{-3}$, select Adam optimizer.
   \FOR{$epoch = 1 : E_{train}$}
   \FOR{$v = 1 : V$}
		\STATE Compute reconstruct loss $\mathcal{L}_{Rec}^v$ by Eq. \ref{eq:7};
		\STATE Compute correlation loss  $\mathcal{L}_{Cor}^v$ by Eq. \ref{eq:8};
		\STATE Compute view adversarial loss $\mathcal{L}_{Dis}^v$ by Eq. \ref{eq:11};
		   \ENDFOR
		\STATE Compute cross-entropy loss $\mathcal{L}_{Ent}$ by Eq. \ref{eq:15};
		\STATE Update $\{\mathbf{Z}^v$ and $\mathbf{C}^v\}_{v=1}^V$ by optimizing Eq. \ref{eq:16overloss};
     \ENDFOR
     \STATE Obtain $\mathbf{Q}$ by concatenation, \textit{i.e.}, $\mathbf{Q}=[\mathbf{C}_{\ast},\{\mathbf{Z}^v\}_{v=1}^V]$;
     \STATE Perform $k$-means on $\mathbf{Q}$ to obtain final clustering result.
 \end{algorithmic}
\end{algorithm}

\revised{Overall, the proposed GMAE framework effectively disentangles representations through optimization, producing clean and complete view-common representations alongside complementary view-specific representations, which together facilitate efficient and robust multi-view clustering. }

By iteratively leveraging global discriminative information, each view is guided toward learning progressively disentangled representations. This mechanism is particularly beneficial for views with limited initial discriminative power, as it substantially improves their representation quality. As a result, the embedded features across views exhibit enhanced discriminative properties, unveiling well-defined and disentangled clustering structures.


\subsection{Complexity Analysis}
Let $K$, $V$, and $N$ denote the number of clusters, views, and data points, respectively. Define $M$ as the maximum number of neurons in the autoencoder and $Z$ as the maximum dimensionality of view-specific variables. It is commonly assumed that $V$, $K$, and $Z$ are significantly smaller than $M$. The optimization process of GMAE, as outlined in Algorithm \ref{algo}, aims to minimize Eq. \ref{eq:16overloss}.  

In each iteration, the computational complexity for generating the prior distribution of the view-common features is $O(N \cdot K)$, while that for the view-specific features is $O(V \cdot N \cdot Z)$. Furthermore, processing all views through the autoencoders incurs a complexity of $O(V \cdot N \cdot M^2)$. Consequently, the overall computational complexity of the proposed method is dominated by these operations and scales linearly with the size of the dataset, $N$, ensuring computational efficiency even for large-scale data. 

\begin{table*}[htbp]
\begin{center} 
\setlength{\tabcolsep}{8.6pt}
\renewcommand{\arraystretch}{1}
\setlength{\abovecaptionskip}{0.1cm}  
\caption{\revised{Clustering performance comparison on different MVC datasets, metrics include ACC, NMI, and PUR.}} \label{table1:results1}
\begin{tabular}{p{2.64cm}p{2.18cm}<{\centering}p{2.18cm}<{\centering}p{2.18cm}<{\centering}p{2.18cm}<{\centering}p{2.18cm}<{\centering}p{2.18cm}<{\centering}p{2.18cm}<{\centering}p{2.18cm}<{\centering}p{2.18cm}<{\centering}p{2.18cm}<{\centering}p{2.18cm}<{\centering}p{2.18cm}<{\centering}}
\toprule
   \multirow{2}{*}{Methods}   & \multicolumn{3}{c}{Synthetic3D}     & \multicolumn{3}{c}{LGG}    & \multicolumn{3}{c}{Dermatology}   & \multicolumn{3}{c}{BRCA}   \\  \cmidrule(lr){2-4} \cmidrule(lr){5-7} \cmidrule(lr){8-10} \cmidrule(lr){11-13} 
   & \multicolumn{1}{c}{ACC}  & \multicolumn{1}{c}{NMI} & \multicolumn{1}{c}{PUR} & \multicolumn{1}{c}{ACC}  & \multicolumn{1}{c}{NMI} & \multicolumn{1}{c}{PUR} & \multicolumn{1}{c}{ACC}  & \multicolumn{1}{c}{NMI} & \multicolumn{1}{c}{PUR} & \multicolumn{1}{c}{ACC}  & \multicolumn{1}{c}{NMI} & \multicolumn{1}{c}{PUR}  \\ \midrule				
   COMPLETER{\tiny\textcolor{gray}{[CVPR’21]}}\cite{lin2021completer}    & \multicolumn{1}{c}{93.33} & \multicolumn{1}{c}{76.06} & \multicolumn{1}{c}{93.33}  & \multicolumn{1}{c}{80.15} & \multicolumn{1}{c}{49.25} & \multicolumn{1}{c}{80.15} & \multicolumn{1}{c}{77.65} & \multicolumn{1}{c}{80.11} & \multicolumn{1}{c}{82.12} & \multicolumn{1}{c}{55.53} & \multicolumn{1}{c}{34.65} & \multicolumn{1}{c}{65.33} \\  
    DCP{\tiny\textcolor{gray}{[TPAMI’22]}}\cite{lin2022dual}    & \multicolumn{1}{c}{97.17} & \multicolumn{1}{c}{87.60} & \multicolumn{1}{c}{97.17}  & \multicolumn{1}{c}{59.55} & \multicolumn{1}{c}{44.82} & \multicolumn{1}{c}{73.03} & \multicolumn{1}{c}{72.91} & \multicolumn{1}{c}{77.22} & \multicolumn{1}{c}{80.73} & \multicolumn{1}{c}{57.29} & \multicolumn{1}{c}{{39.51}} & \multicolumn{1}{c}{60.55} \\  											
   MFLVC{\tiny\textcolor{gray}{[CVPR’22]}}\cite{MFLVCxu2022multi}    & \multicolumn{1}{c}{90.67} & \multicolumn{1}{c}{72.59} & \multicolumn{1}{c}{90.67}  & \multicolumn{1}{c}{79.03} & \multicolumn{1}{c}{49.73} & \multicolumn{1}{c}{79.03} & \multicolumn{1}{c}{58.10} & \multicolumn{1}{c}{56.20} & \multicolumn{1}{c}{62.85} & \multicolumn{1}{c}{55.53} & \multicolumn{1}{c}{27.74} & \multicolumn{1}{c}{60.05} \\ 
    DSMVC{\tiny\textcolor{gray}{[CVPR’22]}}\cite{DSMVCtang2022deep}   & \multicolumn{1}{c}{96.83} & \multicolumn{1}{c}{86.64} & \multicolumn{1}{c}{96.83}  & \multicolumn{1}{c}{{82.77}} & \multicolumn{1}{c}{{54.13}} & \multicolumn{1}{c}{{82.77}} & \multicolumn{1}{c}{{92.74}} & \multicolumn{1}{c}{{87.82}} & \multicolumn{1}{c}{{92.74}} & \multicolumn{1}{c}{54.52} & \multicolumn{1}{c}{33.53} & \multicolumn{1}{c}{68.84} \\  
   SURE{\tiny\textcolor{gray}{[TPAMI’22]}}\cite{SUREyang2022robust}    & \multicolumn{1}{c}{96.33} & \multicolumn{1}{c}{85.16} & \multicolumn{1}{c}{96.33}  & \multicolumn{1}{c}{62.92} & \multicolumn{1}{c}{38.01} & \multicolumn{1}{c}{65.17} & \multicolumn{1}{c}{88.27} & \multicolumn{1}{c}{77.03} & \multicolumn{1}{c}{88.55} & \multicolumn{1}{c}{39.70} & \multicolumn{1}{c}{12.85} & \multicolumn{1}{c}{48.99} \\ 	
    DealMVC{\tiny\textcolor{gray}{[MM’23]}}\cite{Dealmvcyang2023dealmvc}    & \multicolumn{1}{c}{87.50} & \multicolumn{1}{c}{72.07} & \multicolumn{1}{c}{87.50}  & \multicolumn{1}{c}{72.28} & \multicolumn{1}{c}{40.55} & \multicolumn{1}{c}{72.28} & \multicolumn{1}{c}{45.53} & \multicolumn{1}{c}{31.13} & \multicolumn{1}{c}{45.53} & \multicolumn{1}{c}{59.55} & \multicolumn{1}{c}{32.79} & \multicolumn{1}{c}{61.56} \\ 
   GCFAgg{\tiny\textcolor{gray}{[CVPR’23]}}\cite{Gcfaggyan2023gcfagg}    & \multicolumn{1}{c}{96.67} & \multicolumn{1}{c}{85.54} & \multicolumn{1}{c}{96.67}  & \multicolumn{1}{c}{55.06} & \multicolumn{1}{c}{22.95} & \multicolumn{1}{c}{61.80} & \multicolumn{1}{c}{88.27} & \multicolumn{1}{c}{79.25} & \multicolumn{1}{c}{88.27} & \multicolumn{1}{c}{51.51} & \multicolumn{1}{c}{32.41} & \multicolumn{1}{c}{61.31} \\  
    CPSPAN{\tiny\textcolor{gray}{[CVPR’23]}}\cite{CPSPANjin2023deep}    & \multicolumn{1}{c}{97.83} & \multicolumn{1}{c}{90.15} & \multicolumn{1}{c}{97.83}  & \multicolumn{1}{c}{63.30} & \multicolumn{1}{c}{30.53} & \multicolumn{1}{c}{63.30} & \multicolumn{1}{c}{76.26} & \multicolumn{1}{c}{84.63} & \multicolumn{1}{c}{85.20} & \multicolumn{1}{c}{{66.83}} & \multicolumn{1}{c}{34.48} & \multicolumn{1}{c}{\underline{74.12}} \\   
   SDMVC{\tiny\textcolor{gray}{[TKDE’23]}}\cite{SDMVCxu2023self}    & \multicolumn{1}{c}{96.83} & \multicolumn{1}{c}{86.47} & \multicolumn{1}{c}{90.00}  & \multicolumn{1}{c}{63.67} & \multicolumn{1}{c}{43.86} & \multicolumn{1}{c}{67.79} & \multicolumn{1}{c}{70.67} & \multicolumn{1}{c}{83.30} & \multicolumn{1}{c}{84.92} & \multicolumn{1}{c}{57.79} & \multicolumn{1}{c}{33.80} & \multicolumn{1}{c}{64.57} \\  
    CVCL{\tiny\textcolor{gray}{[ICCV’23]}}\cite{CVCLchen2023deep}    & \multicolumn{1}{c}{95.31} & \multicolumn{1}{c}{82.36} & \multicolumn{1}{c}{95.31}  & \multicolumn{1}{c}{58.20} & \multicolumn{1}{c}{23.73} & \multicolumn{1}{c}{58.20} & \multicolumn{1}{c}{56.25} & \multicolumn{1}{c}{56.01} & \multicolumn{1}{c}{67.97} & \multicolumn{1}{c}{61.98} & \multicolumn{1}{c}{34.68} & \multicolumn{1}{c}{68.49} \\  
   SCMVC{\tiny\textcolor{gray}{[TMM’24]}}\cite{SCMVCwu2024self}    & \multicolumn{1}{c}{97.00} & \multicolumn{1}{c}{87.11} & \multicolumn{1}{c}{97.00}  & \multicolumn{1}{c}{73.41} & \multicolumn{1}{c}{39.76} & \multicolumn{1}{c}{73.41} & \multicolumn{1}{c}{\underline{93.85}} & \multicolumn{1}{c}{\underline{88.44}} & \multicolumn{1}{c}{\underline{93.85}} & \multicolumn{1}{c}{50.25} & \multicolumn{1}{c}{30.70} & \multicolumn{1}{c}{60.80} \\   
 MVCAN{\tiny\textcolor{gray}{[CVPR’24]}}\cite{MVCANxu2024investigating}     & \multicolumn{1}{c}{98.17} & \multicolumn{1}{c}{91.27} & \multicolumn{1}{c}{94.59}  & \multicolumn{1}{c}{59.55} & \multicolumn{1}{c}{42.57} & \multicolumn{1}{c}{27.18} & \multicolumn{1}{c}{58.38} & \multicolumn{1}{c}{66.73} & \multicolumn{1}{c}{51.58} & \multicolumn{1}{c}{57.79} & \multicolumn{1}{c}{35.70} & \multicolumn{1}{c}{32.24} \\ 
DCG{\tiny\textcolor{gray}{[AAAI’25]}}\cite{zhang2025incomplete}     
& \multicolumn{1}{c}{95.83} & \multicolumn{1}{c}{84.15} & \multicolumn{1}{c}{95.83}  
& \multicolumn{1}{c}{\underline{84.64}} & \multicolumn{1}{c}{\underline{59.89}} & \multicolumn{1}{c}{\underline{84.64}} 
& \multicolumn{1}{c}{72.35} & \multicolumn{1}{c}{74.97} & \multicolumn{1}{c}{80.45} 
& \multicolumn{1}{c}{54.52} & \multicolumn{1}{c}{37.49} & \multicolumn{1}{c}{68.34} \\ 

SparseMVC{\tiny\textcolor{gray}{[NeurIPS’25]}}\cite{liu2025sparsemvc}     
& \multicolumn{1}{c}{\textbf{98.33}} & \multicolumn{1}{c}{\textbf{92.01}} & \multicolumn{1}{c}{\underline{98.33}}  
& \multicolumn{1}{c}{83.15} & \multicolumn{1}{c}{54.62} & \multicolumn{1}{c}{83.15} 
& \multicolumn{1}{c}{\textbf{95.25}} & \multicolumn{1}{c}{\textbf{89.86}} & \multicolumn{1}{c}{\textbf{95.25}} 
& \multicolumn{1}{c}{\textbf{70.10}} & \multicolumn{1}{c}{\textbf{44.90}} & \multicolumn{1}{c}{70.85} \\ 

BRIDGE{\tiny\textcolor{gray}{ICCV’25]}}\cite{jiang2025unified}    
& \multicolumn{1}{c}{96.17} & \multicolumn{1}{c}{83.74} & \multicolumn{1}{c}{96.17}  
& \multicolumn{1}{c}{56.67} & \multicolumn{1}{c}{23.21} & \multicolumn{1}{c}{69.17} 
& \multicolumn{1}{c}{82.92} & \multicolumn{1}{c}{{85.72}} & \multicolumn{1}{c}{{93.48}} 
& \multicolumn{1}{c}{58.31} & \multicolumn{1}{c}{35.40} & \multicolumn{1}{c}{65.67} \\   \midrule

    \rowcolor{gray!12}\bf GMAE (Ours)    & \multicolumn{1}{c}{\underline{98.00}} & \multicolumn{1}{c}{\underline{90.61}} & \multicolumn{1}{c}{\textbf{99.37}}  & \multicolumn{1}{c}{\textbf{92.16}} & \multicolumn{1}{c}{\textbf{75.48}} & \multicolumn{1}{c}{\textbf{98.91}} & \multicolumn{1}{c}{91.62} & \multicolumn{1}{c}{83.80} & \multicolumn{1}{c}{91.62} & \multicolumn{1}{c}{\underline{67.59}} & \multicolumn{1}{c}{\underline{41.20}} & \multicolumn{1}{c}{\textbf{84.67}} \\ \midrule \midrule
   \multirow{2}{*}{Methods}   & \multicolumn{3}{c}{Wikipedia}     & \multicolumn{3}{c}{LandUse-21}    & \multicolumn{3}{c}{RGB-D}   & \multicolumn{3}{c}{Out-Scene}   \\  \cmidrule(lr){2-4} \cmidrule(lr){5-7} \cmidrule(lr){8-10} \cmidrule(lr){11-13} 
   & \multicolumn{1}{c}{ACC}  & \multicolumn{1}{c}{NMI} & \multicolumn{1}{c}{PUR} & \multicolumn{1}{c}{ACC}  & \multicolumn{1}{c}{NMI} & \multicolumn{1}{c}{PUR} & \multicolumn{1}{c}{ACC}  & \multicolumn{1}{c}{NMI} & \multicolumn{1}{c}{PUR} & \multicolumn{1}{c}{ACC}  & \multicolumn{1}{c}{NMI} & \multicolumn{1}{c}{PUR}  \\ \midrule
     											
   COMPLETER{\tiny\textcolor{gray}{[CVPR’21]}}\cite{lin2021completer}    & \multicolumn{1}{c}{57.14} & \multicolumn{1}{c}{53.10} & \multicolumn{1}{c}{59.31}  & \multicolumn{1}{c}{20.48} & \multicolumn{1}{c}{28.47} & \multicolumn{1}{c}{23.24} & \multicolumn{1}{c}{34.78} & \multicolumn{1}{c}{18.06} & \multicolumn{1}{c}{40.65} & \multicolumn{1}{c}{69.79} & \multicolumn{1}{c}{55.39} & \multicolumn{1}{c}{69.79} \\  	  
   								
    DCP{\tiny\textcolor{gray}{[TPAMI’22]}}\cite{lin2022dual}   & \multicolumn{1}{c}{45.31} & \multicolumn{1}{c}{43.16} & \multicolumn{1}{c}{46.32}  & \multicolumn{1}{c}{20.76} & \multicolumn{1}{c}{28.16} & \multicolumn{1}{c}{24.67} & \multicolumn{1}{c}{42.24} & \multicolumn{1}{c}{31.78} & \multicolumn{1}{c}{49.69} & \multicolumn{1}{c}{56.03} & \multicolumn{1}{c}{45.59} & \multicolumn{1}{c}{56.32} \\  		
   
   MFLVC{\tiny\textcolor{gray}{[CVPR’22]}}\cite{MFLVCxu2022multi}   & \multicolumn{1}{c}{40.12} & \multicolumn{1}{c}{27.52} & \multicolumn{1}{c}{41.70}  & \multicolumn{1}{c}{24.24} & \multicolumn{1}{c}{26.85} & \multicolumn{1}{c}{26.24} & \multicolumn{1}{c}{34.23} & \multicolumn{1}{c}{15.88} & \multicolumn{1}{c}{38.65} & \multicolumn{1}{c}{58.97} & \multicolumn{1}{c}{51.31} & \multicolumn{1}{c}{58.97} \\ 
   											
    DSMVC{\tiny\textcolor{gray}{[CVPR’22]}}\cite{DSMVCtang2022deep}   & \multicolumn{1}{c}{{60.32}} & \multicolumn{1}{c}{54.74} & \multicolumn{1}{c}{{62.19}}  & \multicolumn{1}{c}{\underline{30.57}} & \multicolumn{1}{c}{\textbf{36.98}} & \multicolumn{1}{c}{\underline{34.10}} & \multicolumn{1}{c}{39.68} & \multicolumn{1}{c}{31.36} & \multicolumn{1}{c}{49.28} & \multicolumn{1}{c}{62.13} & \multicolumn{1}{c}{53.01} & \multicolumn{1}{c}{64.25} \\  

   SURE{\tiny\textcolor{gray}{[TPAMI’22]}}\cite{SUREyang2022robust}   & \multicolumn{1}{c}{50.65} & \multicolumn{1}{c}{39.97} & \multicolumn{1}{c}{54.11} & \multicolumn{1}{c}{26.57} & \multicolumn{1}{c}{30.77} & \multicolumn{1}{c}{28.14} & \multicolumn{1}{c}{28.99} & \multicolumn{1}{c}{26.75} & \multicolumn{1}{c}{48.31} & \multicolumn{1}{c}{60.97} & \multicolumn{1}{c}{48.09} & \multicolumn{1}{c}{60.97} \\
              
    DealMVC{\tiny\textcolor{gray}{[MM’23]}}\cite{Dealmvcyang2023dealmvc}    & \multicolumn{1}{c}{38.96} & \multicolumn{1}{c}{37.09} & \multicolumn{1}{c}{38.96}  & \multicolumn{1}{c}{13.48} & \multicolumn{1}{c}{12.23} & \multicolumn{1}{c}{13.40} & \multicolumn{1}{c}{30.50} & \multicolumn{1}{c}{11.95} & \multicolumn{1}{c}{30.57} & \multicolumn{1}{c}{69.57} & \multicolumn{1}{c}{59.44} & \multicolumn{1}{c}{69.57} \\
   											
   GCFAgg{\tiny\textcolor{gray}{[CVPR’23]}}\cite{Gcfaggyan2023gcfagg}    & \multicolumn{1}{c}{51.80} & \multicolumn{1}{c}{45.87} & \multicolumn{1}{c}{56.57}  & \multicolumn{1}{c}{24.76} & \multicolumn{1}{c}{28.09} & \multicolumn{1}{c}{26.10} & \multicolumn{1}{c}{25.67} & \multicolumn{1}{c}{21.00} & \multicolumn{1}{c}{44.17} & \multicolumn{1}{c}{68.23} & \multicolumn{1}{c}{57.14} & \multicolumn{1}{c}{68.23} \\  
   											
    CPSPAN{\tiny\textcolor{gray}{[CVPR’23]}}\cite{CPSPANjin2023deep}    & \multicolumn{1}{c}{22.08} & \multicolumn{1}{c}{8.35} & \multicolumn{1}{c}{24.39}  & \multicolumn{1}{c}{26.67} & \multicolumn{1}{c}{31.10} & \multicolumn{1}{c}{32.71} & \multicolumn{1}{c}{38.78} & \multicolumn{1}{c}{33.77} & \multicolumn{1}{c}{\underline{57.49}} & \multicolumn{1}{c}{59.15} & \multicolumn{1}{c}{50.46} & \multicolumn{1}{c}{61.20} \\  
   											
   SDMVC{\tiny\textcolor{gray}{[TKDE’23]}}\cite{SDMVCxu2023self}    & \multicolumn{1}{c}{55.99} & \multicolumn{1}{c}{53.98} & \multicolumn{1}{c}{62.05}  & \multicolumn{1}{c}{21.43} & \multicolumn{1}{c}{27.75} & \multicolumn{1}{c}{21.29} & \multicolumn{1}{c}{\underline{44.51}} & \multicolumn{1}{c}{{38.83}} & \multicolumn{1}{c}{51.07} & \multicolumn{1}{c}{56.03} & \multicolumn{1}{c}{46.18} & \multicolumn{1}{c}{59.93} \\  
   											
    CVCL{\tiny\textcolor{gray}{[ICCV’23]}}\cite{CVCLchen2023deep}     & \multicolumn{1}{c}{42.81} & \multicolumn{1}{c}{32.69} & \multicolumn{1}{c}{48.44}  & \multicolumn{1}{c}{29.05} & \multicolumn{1}{c}{33.43} & \multicolumn{1}{c}{30.08} & \multicolumn{1}{c}{25.14} & \multicolumn{1}{c}{16.31} & \multicolumn{1}{c}{39.28} & \multicolumn{1}{c}{{73.51}} & \multicolumn{1}{c}{59.59} & \multicolumn{1}{c}{{73.51}} \\  
   											
   SCMVC{\tiny\textcolor{gray}{[TMM’24]}}\cite{SCMVCwu2024self}     & \multicolumn{1}{c}{53.54} & \multicolumn{1}{c}{35.59} & \multicolumn{1}{c}{55.84}  & \multicolumn{1}{c}{26.57} & \multicolumn{1}{c}{30.20} & \multicolumn{1}{c}{28.57} & \multicolumn{1}{c}{35.75} & \multicolumn{1}{c}{32.83} & \multicolumn{1}{c}{52.73} & \multicolumn{1}{c}{71.54} & \multicolumn{1}{c}{{60.19}} & \multicolumn{1}{c}{71.54} \\ 

      MVCAN{\tiny\textcolor{gray}{[CVPR’24]}}\cite{MVCANxu2024investigating}     & \multicolumn{1}{c}{59.02} & \multicolumn{1}{c}{\underline{55.81}} & \multicolumn{1}{c}{45.92}  & \multicolumn{1}{c}{22.57} & \multicolumn{1}{c}{29.78} & \multicolumn{1}{c}{10.45} & \multicolumn{1}{c}{43.41} & \multicolumn{1}{c}{\textbf{41.02}} & \multicolumn{1}{c}{26.28} & \multicolumn{1}{c}{70.98} & \multicolumn{1}{c}{58.23} & \multicolumn{1}{c}{49.95} \\ DCG{\tiny\textcolor{gray}{[AAAI’25]}}\cite{zhang2025incomplete}     
& \multicolumn{1}{c}{54.40} & \multicolumn{1}{c}{50.09} & \multicolumn{1}{c}{58.44}  
& \multicolumn{1}{c}{24.62} & \multicolumn{1}{c}{26.41} & \multicolumn{1}{c}{25.81} 
& \multicolumn{1}{c}{40.03} & \multicolumn{1}{c}{29.89} & \multicolumn{1}{c}{51.48} 
& \multicolumn{1}{c}{72.32} & \multicolumn{1}{c}{57.37} & \multicolumn{1}{c}{72.32} \\	

SparseMVC{\tiny\textcolor{gray}{[NeurIPS’25]}}\cite{liu2025sparsemvc}     
& \multicolumn{1}{c}{\underline{61.04}} & \multicolumn{1}{c}{54.79} & \multicolumn{1}{c}{\underline{62.91}}  
& \multicolumn{1}{c}{25.38} & \multicolumn{1}{c}{29.16} & \multicolumn{1}{c}{26.81} 
& \multicolumn{1}{c}{35.40} & \multicolumn{1}{c}{26.35} & \multicolumn{1}{c}{49.07} 
& \multicolumn{1}{c}{\textbf{77.49}} & \multicolumn{1}{c}{\textbf{63.34}} & \multicolumn{1}{c}{\textbf{77.49}} \\ 
 
BRIDGE{\tiny\textcolor{gray}{[ICCV’25]}}\cite{jiang2025unified}     
& \multicolumn{1}{c}{45.35} & \multicolumn{1}{c}{39.23} & \multicolumn{1}{c}{52.88}  
& \multicolumn{1}{c}{22.49} & \multicolumn{1}{c}{25.37} & \multicolumn{1}{c}{25.19} 
& \multicolumn{1}{c}{29.68} & \multicolumn{1}{c}{27.60} & \multicolumn{1}{c}{31.67} 
& \multicolumn{1}{c}{59.94} & \multicolumn{1}{c}{51.19} & \multicolumn{1}{c}{60.77} \\ \midrule
    \rowcolor{gray!12}\bf GMAE (Ours)    & \multicolumn{1}{c}{\textbf{62.18}} & \multicolumn{1}{c}{\textbf{56.67}} & \multicolumn{1}{c}{\textbf{70.35}}  & \multicolumn{1}{c}{\textbf{32.39}} & \multicolumn{1}{c}{\underline{34.57}} & \multicolumn{1}{c}{\textbf{35.02}} & \multicolumn{1}{c}{\textbf{45.07}} & \multicolumn{1}{c}{\underline{40.06}} & \multicolumn{1}{c}{\textbf{58.94}} & \multicolumn{1}{c}{\underline{75.16}} & \multicolumn{1}{c}{\underline{62.14}} & \multicolumn{1}{c}{\underline{76.73}} \\  \bottomrule
\end{tabular}
\end{center}\vspace{-1em}
\end{table*}
\section{Experiments}
\label{sec:exp}
In this section, we conduct experiments on 13 public
datasets to evaluate the proposed method in terms of both complete and incomplete multi-view clustering tasks.
\subsection{Experimental Settings}
\subsubsection{Benchmark Datasets}
The experiments are carried out on 13 benchmark datasets with different scales for comparison and analysis, as shown in Table \ref{table2:stat}. The datasets are briefly described as follows:

\begin{table}[ht]\vspace{0.5em}
\begin{center}
\setlength{\tabcolsep}{3pt} 
\renewcommand{\arraystretch}{1}
\setlength{\abovecaptionskip}{0.15cm}  
\caption{\revised{Details of MVC datasets used in the experiment.}} 
\label{table2:stat}
\resizebox{\linewidth}{!}{
\begin{tabular}{@{}p{2.6cm}<{\centering}p{1.1cm}<{\centering}p{0.95cm}<{\centering}p{0.8cm}<{\centering}p{3.45cm}<{\centering}} 
\toprule

 Datasets  & Samples  & Clusters & Views & \multicolumn{1}{c}{View Dimensions} \\ \midrule
 \multicolumn{5}{c}{\textbf{\textit{Image-Text Domain}}} \\ \midrule
 Wikipedia \tablefootnote{\url{https://dumps.wikimedia.org/}} & 693 & 10  & 2 & [128, 10] \\ 
 RGB-D \cite{kong2014you} & 1,449 & 13  & 2 & [2048, 300] \\ \midrule
 \multicolumn{5}{c}{\textbf{\textit{Image Domain}}} \\ \midrule 
 STL-10 \tablefootnote{\url{https://cs.stanford.edu/~acoates/stl10/}} & 13,000 & 10  & 3 & [1024, 512, 2048] \\ 
 LandUse-21 \cite{yang2010bag} & 2,100 & 21  & 3 & [20, 59, 59] \\ 
 Out-Scene \cite{oliva2001modeling-Out-Scene} & 2,688 & 8  & 4 & [512, 432, 256, 48] \\  
 ALOI-100 \tablefootnote{\url{https://aloi.science.uva.nl/}} & 10,800 & 100  & 4 & [77, 13, 64, 125] \\ 
 NUSWIDEOBJ \cite{chua2009nus} & 30,000 & 31  & 5 & [65, 226, 145, 74, 129] \\ 
 Digits \cite{frank2010uci} & 2,000 & 10  & 6 & [240, 76, 216, 47, 64, 6] \\  
 MSRCV1 \cite{winn2005locus} & 210 & 7  & 6 & [1302, 48, 512, 100, 256, 210] \\ 
 Dermatology \tablefootnote{\url{http://archive.ics.uci.edu/dataset/33/dermatology}} & 358 & 6  & 2 & [12, 22] \\ \midrule
 
 \multicolumn{5}{c}{\textbf{\textit{Multi-Omics Domain}}} \\ \midrule
 LGG \cite{cancer2015comprehensive} & 267 & 3  & 4 & [2000, 2000, 333, 209] \\ 
 BRCA \cite{koboldt2012comprehensive} & 398 & 4  & 4 & [2000, 2000, 278, 212] \\ \midrule

 \multicolumn{5}{c}{\textbf{\textit{Synthetics Domain}}} \\ \midrule
 Synthetic3D \cite{kumar2011co} & 600 & 3  & 3 & [3, 3, 3] \\  \bottomrule

\end{tabular}
}\vspace{-2.5em}
\end{center}
\end{table}

\revised{The \textit{Image} datasets present a rich variety of visual information captured from diverse domains. \textbf{ALOI-100} includes images representing 100 small object classes, meticulously cataloged to explore the nuances of visual recognition. \textbf{Digits} dataset offers an in-depth study of handwritten numeral classification, derived from Dutch utility maps, with each digit carefully captured to reveal intricate details across 10 classes. \textbf{LandUse-21} reflects the diverse usage of land in different environments. \textbf{Dermatology} dataset provides 358 satellite images, each representing one of six distinct skin disease classes, serving as a critical resource for medical image analysis. Meanwhile, \textbf{NUSWIDEOBJ} is a robust dataset supporting object recognition across a large number of categories. \textbf{Out-Scene} dataset offers outdoor scene images, allowing for a deep exploration of environmental imagery. \textbf{STL-10} covers common vehicles and animal types, showcasing a broad range of features, making it a foundational resource for object categorization research. About \textbf{MSRCV1} dataset, which encompasses 7 distinct classes, including trees, buildings, airplanes, cars, and cows, provides a comprehensive tool for multi-view image analysis. }

\begin{table*}[htbp]
\begin{center} 
\setlength{\tabcolsep}{8.6pt}
\renewcommand{\arraystretch}{1}
\setlength{\abovecaptionskip}{0.1cm}  
\caption{\revised{Clustering comparisons on Digits with increased views.  “-\#V” represents the number of views.}} \label{table4:results4}
\begin{tabular}{p{2.64cm}p{2.18cm}<{\centering}p{2.18cm}<{\centering}p{2.18cm}<{\centering}p{2.18cm}<{\centering}p{2.18cm}<{\centering}p{2.18cm}<{\centering}p{2.18cm}<{\centering}p{2.18cm}<{\centering}p{2.18cm}<{\centering}p{2.18cm}<{\centering}p{2.18cm}<{\centering}p{2.18cm}<{\centering}}
\toprule
   \multirow{2}{*}{Methods}   & \multicolumn{3}{c}{Digits-2V }     & \multicolumn{3}{c}{Digits-3V}    & \multicolumn{3}{c}{Digits-4V}   & \multicolumn{3}{c}{Digits-6V}   \\  \cmidrule(lr){2-4} \cmidrule(lr){5-7} \cmidrule(lr){8-10} \cmidrule(lr){11-13} 
   & \multicolumn{1}{c}{ACC}  & \multicolumn{1}{c}{NMI} & \multicolumn{1}{c}{PUR} & \multicolumn{1}{c}{ACC}  & \multicolumn{1}{c}{NMI} & \multicolumn{1}{c}{PUR} & \multicolumn{1}{c}{ACC}  & \multicolumn{1}{c}{NMI} & \multicolumn{1}{c}{PUR} & \multicolumn{1}{c}{ACC}  & \multicolumn{1}{c}{NMI} & \multicolumn{1}{c}{PUR}  \\ \midrule 	
   		
   COMPLETER{\tiny\textcolor{gray}{[CVPR’21]}}\cite{lin2021completer}    & \multicolumn{1}{c}{84.10} & \multicolumn{1}{c}{\underline{85.03}} & \multicolumn{1}{c}{84.45}  & \multicolumn{1}{c}{88.85} & \multicolumn{1}{c}{81.94} & \multicolumn{1}{c}{88.85} & \multicolumn{1}{c}{73.55} & \multicolumn{1}{c}{75.17} & \multicolumn{1}{c}{77.60} & \multicolumn{1}{c}{73.60} & \multicolumn{1}{c}{73.58} & \multicolumn{1}{c}{75.30} \\  	
   											
    DCP{\tiny\textcolor{gray}{[TPAMI’22]}}\cite{lin2022dual}    & \multicolumn{1}{c}{79.45} & \multicolumn{1}{c}{79.61} & \multicolumn{1}{c}{80.65}  & \multicolumn{1}{c}{80.15} & \multicolumn{1}{c}{83.02} & \multicolumn{1}{c}{82.65} & \multicolumn{1}{c}{75.35} & \multicolumn{1}{c}{78.69} & \multicolumn{1}{c}{79.20} & \multicolumn{1}{c}{76.80} & \multicolumn{1}{c}{79.57} & \multicolumn{1}{c}{81.00} \\  	
   	 	  	 	 	  	 	 	  			
   MFLVC{\tiny\textcolor{gray}{[CVPR’22]}}\cite{MFLVCxu2022multi}    & \multicolumn{1}{c}{70.40} & \multicolumn{1}{c}{69.55} & \multicolumn{1}{c}{72.45}  & \multicolumn{1}{c}{91.85} & \multicolumn{1}{c}{85.92} & \multicolumn{1}{c}{91.85} & \multicolumn{1}{c}{82.50} & \multicolumn{1}{c}{79.88} & \multicolumn{1}{c}{82.50} & \multicolumn{1}{c}{85.00} & \multicolumn{1}{c}{84.06} & \multicolumn{1}{c}{85.00} \\ 
   	 	  	 	 	  	 	 	  			
     DSMVC{\tiny\textcolor{gray}{[CVPR’22]}}\cite{DSMVCtang2022deep}    & \multicolumn{1}{c}{79.70} & \multicolumn{1}{c}{75.23} & \multicolumn{1}{c}{79.70}  & \multicolumn{1}{c}{\underline{95.40}} & \multicolumn{1}{c}{\underline{90.54}} & \multicolumn{1}{c}{\underline{95.40}} & \multicolumn{1}{c}{88.15} & \multicolumn{1}{c}{\underline{85.10}} & \multicolumn{1}{c}{88.15} & \multicolumn{1}{c}{{95.45}} & \multicolumn{1}{c}{\underline{92.52}} & \multicolumn{1}{c}{{95.45}} \\  
   										   
    SURE{\tiny\textcolor{gray}{[TPAMI’22]}}\cite{SUREyang2022robust}    & \multicolumn{1}{c}{81.20} & \multicolumn{1}{c}{69.51} & \multicolumn{1}{c}{81.20} & \multicolumn{1}{c}{74.10} & \multicolumn{1}{c}{65.58} & \multicolumn{1}{c}{74.10} & \multicolumn{1}{c}{61.80} & \multicolumn{1}{c}{55.34} & \multicolumn{1}{c}{61.80} & \multicolumn{1}{c}{66.75} & \multicolumn{1}{c}{62.72} & \multicolumn{1}{c}{69.65} \\
   
     DealMVC{\tiny\textcolor{gray}{[MM’23]}}\cite{Dealmvcyang2023dealmvc}    & \multicolumn{1}{c}{50.30} & \multicolumn{1}{c}{50.13} & \multicolumn{1}{c}{50.30}  & \multicolumn{1}{c}{45.50} & \multicolumn{1}{c}{49.07} & \multicolumn{1}{c}{45.50} & \multicolumn{1}{c}{62.60} & \multicolumn{1}{c}{65.57} & \multicolumn{1}{c}{62.60} & \multicolumn{1}{c}{57.10} & \multicolumn{1}{c}{68.04} & \multicolumn{1}{c}{57.10} \\
   											
   GCFAgg{\tiny\textcolor{gray}{[CVPR’23]}}\cite{Gcfaggyan2023gcfagg}    & \multicolumn{1}{c}{61.55} & \multicolumn{1}{c}{66.43} & \multicolumn{1}{c}{62.90}  & \multicolumn{1}{c}{90.80} & \multicolumn{1}{c}{85.75} & \multicolumn{1}{c}{90.80} & \multicolumn{1}{c}{79.55} & \multicolumn{1}{c}{77.84} & \multicolumn{1}{c}{79.55} & \multicolumn{1}{c}{83.95} & \multicolumn{1}{c}{82.76} & \multicolumn{1}{c}{83.95} \\  
   											
     CPSPAN{\tiny\textcolor{gray}{[CVPR’23]}}\cite{CPSPANjin2023deep}    & \multicolumn{1}{c}{89.65} & \multicolumn{1}{c}{80.67} & \multicolumn{1}{c}{\underline{89.65}}  & \multicolumn{1}{c}{87.25} & \multicolumn{1}{c}{79.39} & \multicolumn{1}{c}{87.25} & \multicolumn{1}{c}{88.95} & \multicolumn{1}{c}{81.27} & \multicolumn{1}{c}{{88.95}} & \multicolumn{1}{c}{89.55} & \multicolumn{1}{c}{82.30} & \multicolumn{1}{c}{9.95} \\  
   											
   SDMVC{\tiny\textcolor{gray}{[TKDE’23]}}\cite{SDMVCxu2023self}    & \multicolumn{1}{c}{80.60} & \multicolumn{1}{c}{74.52} & \multicolumn{1}{c}{80.60}  & \multicolumn{1}{c}{77.15} & \multicolumn{1}{c}{73.83} & \multicolumn{1}{c}{77.85} & \multicolumn{1}{c}{82.40} & \multicolumn{1}{c}{77.79} & \multicolumn{1}{c}{82.40} & \multicolumn{1}{c}{84.85} & \multicolumn{1}{c}{75.05} & \multicolumn{1}{c}{84.85} \\  
   											
     CVCL{\tiny\textcolor{gray}{[ICCV’23]}}\cite{CVCLchen2023deep}     & \multicolumn{1}{c}{85.42} & \multicolumn{1}{c}{77.17} & \multicolumn{1}{c}{85.42}  & \multicolumn{1}{c}{83.96} & \multicolumn{1}{c}{76.48} & \multicolumn{1}{c}{83.96} & \multicolumn{1}{c}{79.22} & \multicolumn{1}{c}{75.45} & \multicolumn{1}{c}{79.22} & \multicolumn{1}{c}{89.15} & \multicolumn{1}{c}{80.60} & \multicolumn{1}{c}{89.15} \\  		
   
   SCMVC{\tiny\textcolor{gray}{[TMM’24]}}\cite{SCMVCwu2024self}    & \multicolumn{1}{c}{87.50} & \multicolumn{1}{c}{78.46} & \multicolumn{1}{c}{87.50}  & \multicolumn{1}{c}{89.35} & \multicolumn{1}{c}{80.87} & \multicolumn{1}{c}{89.35} & \multicolumn{1}{c}{86.15} & \multicolumn{1}{c}{78.00} & \multicolumn{1}{c}{86.15} & \multicolumn{1}{c}{94.15} & \multicolumn{1}{c}{87.91} & \multicolumn{1}{c}{94.15} \\ 	
   
    MVCAN{\tiny\textcolor{gray}{[CVPR’24]}}\cite{MVCANxu2024investigating}     & \multicolumn{1}{c}{\underline{90.00}} & \multicolumn{1}{c}{82.82} & \multicolumn{1}{c}{79.82}  & \multicolumn{1}{c}{91.90} & \multicolumn{1}{c}{85.26} & \multicolumn{1}{c}{83.20} & \multicolumn{1}{c}{91.90} & \multicolumn{1}{c}{85.03} & \multicolumn{1}{c}{83.14} & \multicolumn{1}{c}{94.90} & \multicolumn{1}{c}{89.63} & \multicolumn{1}{c}{89.09} \\ 
    
DCG{\tiny\textcolor{gray}{[AAAI’25]}}\cite{zhang2025incomplete}     
& \multicolumn{1}{c}{74.15} & \multicolumn{1}{c}{68.36} & \multicolumn{1}{c}{74.25}  
& \multicolumn{1}{c}{78.75} & \multicolumn{1}{c}{80.87} & \multicolumn{1}{c}{82.40} 
& \multicolumn{1}{c}{79.05} & \multicolumn{1}{c}{80.91} & \multicolumn{1}{c}{82.65} 
& \multicolumn{1}{c}{81.70} & \multicolumn{1}{c}{82.59} & \multicolumn{1}{c}{82.70} \\ 

SparseMVC{\tiny\textcolor{gray}{[NeurIPS’25]}}\cite{liu2025sparsemvc}     
& \multicolumn{1}{c}{84.20} & \multicolumn{1}{c}{81.45} & \multicolumn{1}{c}{84.20}  
& \multicolumn{1}{c}{93.40} & \multicolumn{1}{c}{87.23} & \multicolumn{1}{c}{93.40} 
& \multicolumn{1}{c}{84.25} & \multicolumn{1}{c}{81.35} & \multicolumn{1}{c}{84.25} 
& \multicolumn{1}{c}{91.00} & \multicolumn{1}{c}{86.28} & \multicolumn{1}{c}{91.00} \\ 
         
BRIDGE{\tiny\textcolor{gray}{ICCV’25]}}\cite{jiang2025unified}    
& \multicolumn{1}{c}{86.28} & \multicolumn{1}{c}{77.88} & \multicolumn{1}{c}{86.28}  
& \multicolumn{1}{c}{88.94} & \multicolumn{1}{c}{81.51} & \multicolumn{1}{c}{88.94} 
& \multicolumn{1}{c}{\underline{92.06}} & \multicolumn{1}{c}{84.92} & \multicolumn{1}{c}{\underline{92.06}} 
& \multicolumn{1}{c}{\underline{96.01}} & \multicolumn{1}{c}{91.52} & \multicolumn{1}{c}{\underline{96.01}} \\ \midrule
      \rowcolor{gray!12}\bf GMAE (Ours)   & \multicolumn{1}{c}{\textbf{95.55}} & \multicolumn{1}{c}{\textbf{90.88}} & \multicolumn{1}{c}{\textbf{95.90}} & \multicolumn{1}{c}{\textbf{95.90}} & \multicolumn{1}{c}{\textbf{91.51}} & \multicolumn{1}{c}{\textbf{96.00}} & \multicolumn{1}{c}{\textbf{96.50}} & \multicolumn{1}{c}{\textbf{92.07}} & \multicolumn{1}{c}{\textbf{96.65}} & \multicolumn{1}{c}{\textbf{97.45}} & \multicolumn{1}{c}{\textbf{94.16}} & \multicolumn{1}{c}{\textbf{97.45}} \\ \midrule \midrule
   \multirow{2}{*}{Methods}   & \multicolumn{3}{c}{STL-10}     & \multicolumn{3}{c}{ALOI-100}    & \multicolumn{3}{c}{NUSWIDEOBJ}   & \multicolumn{3}{c}{MSRCV1}   \\  \cmidrule(lr){2-4} \cmidrule(lr){5-7} \cmidrule(lr){8-10} \cmidrule(lr){11-13} 
   & \multicolumn{1}{c}{ACC}  & \multicolumn{1}{c}{NMI} & \multicolumn{1}{c}{PUR} & \multicolumn{1}{c}{ACC}  & \multicolumn{1}{c}{NMI} & \multicolumn{1}{c}{PUR} & \multicolumn{1}{c}{ACC}  & \multicolumn{1}{c}{NMI} & \multicolumn{1}{c}{PUR} & \multicolumn{1}{c}{ACC}  & \multicolumn{1}{c}{NMI} & \multicolumn{1}{c}{PUR}  \\ \midrule	
   COMPLETER{\tiny\textcolor{gray}{[CVPR’21]}}\cite{lin2021completer}    & \multicolumn{1}{c}{11.11} & \multicolumn{1}{c}{0.20} & \multicolumn{1}{c}{11.22}  & \multicolumn{1}{c}{30.70} & \multicolumn{1}{c}{62.12} & \multicolumn{1}{c}{33.63} & \multicolumn{1}{c}{12.73} & \multicolumn{1}{c}{8.51} & \multicolumn{1}{c}{19.53} & \multicolumn{1}{c}{90.00} & \multicolumn{1}{c}{{87.90}} & \multicolumn{1}{c}{90.00} \\  	
    DCP{\tiny\textcolor{gray}{[TPAMI’22]}}\cite{lin2022dual}    & \multicolumn{1}{c}{11.08} & \multicolumn{1}{c}{0.18} & \multicolumn{1}{c}{11.20}  & \multicolumn{1}{c}{34.01} & \multicolumn{1}{c}{60.28} & \multicolumn{1}{c}{37.32} & \multicolumn{1}{c}{10.38} & \multicolumn{1}{c}{8.91} & \multicolumn{1}{c}{20.06} & \multicolumn{1}{c}{25.71} & \multicolumn{1}{c}{23.25} & \multicolumn{1}{c}{27.14} \\  
   MFLVC{\tiny\textcolor{gray}{[CVPR’22]}}\cite{MFLVCxu2022multi}    & \multicolumn{1}{c}{20.18} & \multicolumn{1}{c}{11.46} & \multicolumn{1}{c}{21.72}  & \multicolumn{1}{c}{33.17} & \multicolumn{1}{c}{73.28} & \multicolumn{1}{c}{33.17} & \multicolumn{1}{c}{16.59} & \multicolumn{1}{c}{10.06} & \multicolumn{1}{c}{23.15} & \multicolumn{1}{c}{63.33} & \multicolumn{1}{c}{66.11} & \multicolumn{1}{c}{64.29} \\ 	
    DSMVC{\tiny\textcolor{gray}{[CVPR’22]}}\cite{DSMVCtang2022deep}    & \multicolumn{1}{c}{\textbf{96.85}} & \multicolumn{1}{c}{\textbf{92.40}} & \multicolumn{1}{c}{\textbf{96.85}}  & \multicolumn{1}{c}{71.52} & \multicolumn{1}{c}{\underline{90.87}} & \multicolumn{1}{c}{72.72} & \multicolumn{1}{c}{12.79} & \multicolumn{1}{c}{10.01} & \multicolumn{1}{c}{22.66} & \multicolumn{1}{c}{64.29} & \multicolumn{1}{c}{54.29} & \multicolumn{1}{c}{64.29} \\ 
   SURE{\tiny\textcolor{gray}{[TPAMI’22]}}\cite{SUREyang2022robust}    & \multicolumn{1}{c}{11.13} & \multicolumn{1}{c}{0.16} & \multicolumn{1}{c}{11.23} & \multicolumn{1}{c}{10.13} & \multicolumn{1}{c}{34.19} & \multicolumn{1}{c}{11.90} & \multicolumn{1}{c}{11.26} & \multicolumn{1}{c}{10.11} & \multicolumn{1}{c}{20.51} & \multicolumn{1}{c}{{91.43}} & \multicolumn{1}{c}{85.84} & \multicolumn{1}{c}{{91.43}} \\						
    DealMVC{\tiny\textcolor{gray}{[MM’23]}}\cite{Dealmvcyang2023dealmvc}    & \multicolumn{1}{c}{60.40} & \multicolumn{1}{c}{48.96} & \multicolumn{1}{c}{60.40}  & \multicolumn{1}{c}{13.11} & \multicolumn{1}{c}{48.54} & \multicolumn{1}{c}{13.10} & \multicolumn{1}{c}{13.64} & \multicolumn{1}{c}{2.31} & \multicolumn{1}{c}{13.64} & \multicolumn{1}{c}{82.00} & \multicolumn{1}{c}{75.54} & \multicolumn{1}{c}{82.00} \\	
   GCFAgg{\tiny\textcolor{gray}{[CVPR’23]}}\cite{Gcfaggyan2023gcfagg}    & \multicolumn{1}{c}{17.58} & \multicolumn{1}{c}{4.46} & \multicolumn{1}{c}{18.24}  & \multicolumn{1}{c}{74.11} & \multicolumn{1}{c}{88.30} & \multicolumn{1}{c}{76.63} & \multicolumn{1}{c}{\underline{18.30}} & \multicolumn{1}{c}{\underline{19.41}} & \multicolumn{1}{c}{{31.12}} & \multicolumn{1}{c}{39.52} & \multicolumn{1}{c}{31.91} & \multicolumn{1}{c}{42.86} \\  
    CPSPAN{\tiny\textcolor{gray}{[CVPR’23]}}\cite{CPSPANjin2023deep}    & \multicolumn{1}{c}{20.76} & \multicolumn{1}{c}{9.54} & \multicolumn{1}{c}{21.72}  & \multicolumn{1}{c}{56.96} & \multicolumn{1}{c}{78.78} & \multicolumn{1}{c}{67.99} & \multicolumn{1}{c}{10.32} & \multicolumn{1}{c}{9.21} & \multicolumn{1}{c}{22.65} & \multicolumn{1}{c}{67.62} & \multicolumn{1}{c}{69.83} & \multicolumn{1}{c}{89.52} \\  	
   SDMVC{\tiny\textcolor{gray}{[TKDE’23]}}\cite{SDMVCxu2023self}    & \multicolumn{1}{c}{63.32} & \multicolumn{1}{c}{55.36} & \multicolumn{1}{c}{63.32}  & \multicolumn{1}{c}{52.02} & \multicolumn{1}{c}{74.70} & \multicolumn{1}{c}{56.56} & \multicolumn{1}{c}{13.20} & \multicolumn{1}{c}{14.99} & \multicolumn{1}{c}{24.97} & \multicolumn{1}{c}{59.52} & \multicolumn{1}{c}{52.51} & \multicolumn{1}{c}{45.24} \\
    CVCL{\tiny\textcolor{gray}{[ICCV’23]}}\cite{CVCLchen2023deep}     & \multicolumn{1}{c}{15.56} & \multicolumn{1}{c}{1.13} & \multicolumn{1}{c}{15.63}  & \multicolumn{1}{c}{21.86} & \multicolumn{1}{c}{43.13} & \multicolumn{1}{c}{23.29} & \multicolumn{1}{c}{14.17} & \multicolumn{1}{c}{14.83} & \multicolumn{1}{c}{27.44} & \multicolumn{1}{c}{90.62} & \multicolumn{1}{c}{84.57} & \multicolumn{1}{c}{90.62} \\  
   SCMVC{\tiny\textcolor{gray}{[TMM’24]}}\cite{SCMVCwu2024self}    & \multicolumn{1}{c}{79.68} & \multicolumn{1}{c}{73.49} & \multicolumn{1}{c}{79.68}  & \multicolumn{1}{c}{\underline{77.72}} & \multicolumn{1}{c}{{89.42}} & \multicolumn{1}{c}{\underline{81.05}} & \multicolumn{1}{c}{{17.74}} & \multicolumn{1}{c}{{19.01}} & \multicolumn{1}{c}{30.35} & \multicolumn{1}{c}{90.95} & \multicolumn{1}{c}{83.92} & \multicolumn{1}{c}{90.95} \\         MVCAN{\tiny\textcolor{gray}{[CVPR’24]}}\cite{MVCANxu2024investigating}     & \multicolumn{1}{c}{66.62} & \multicolumn{1}{c}{72.25} & \multicolumn{1}{c}{58.49}  & \multicolumn{1}{c}{67.48} & \multicolumn{1}{c}{83.78} & \multicolumn{1}{c}{56.71} & \multicolumn{1}{c}{15.49} & \multicolumn{1}{c}{13.09} & \multicolumn{1}{c}{4.06} & \multicolumn{1}{c}{71.54} & \multicolumn{1}{c}{{60.19}} & \multicolumn{1}{c}{71.54} \\
   DCG{\tiny\textcolor{gray}{[AAAI’25]}}\cite{zhang2025incomplete}     
& \multicolumn{1}{c}{77.44} & \multicolumn{1}{c}{84.87} & \multicolumn{1}{c}{77.45}  
& \multicolumn{1}{c}{46.06} & \multicolumn{1}{c}{75.74} & \multicolumn{1}{c}{47.30} 
& \multicolumn{1}{c}{15.19} & \multicolumn{1}{c}{12.25} & \multicolumn{1}{c}{24.52} 
& \multicolumn{1}{c}{52.38} & \multicolumn{1}{c}{41.66} & \multicolumn{1}{c}{53.81} \\ 

SparseMVC{\tiny\textcolor{gray}{[NeurIPS’25]}}\cite{liu2025sparsemvc}     
& \multicolumn{1}{c}{43.68} & \multicolumn{1}{c}{44.69} & \multicolumn{1}{c}{44.64}  
& \multicolumn{1}{c}{\textbf{82.21}} & \multicolumn{1}{c}{\textbf{92.65}} & \multicolumn{1}{c}{84.19} 
& \multicolumn{1}{c}{\textbf{19.05}} & \multicolumn{1}{c}{\textbf{20.29}} & \multicolumn{1}{c}{\underline{31.52}} 
& \multicolumn{1}{c}{\textbf{97.14}} & \multicolumn{1}{c}{\textbf{94.22}} & \multicolumn{1}{c}{\textbf{97.14}} \\ 

BRIDGE{\tiny\textcolor{gray}{ICCV’25]}}\cite{jiang2025unified}    
& \multicolumn{1}{c}{28.29} & \multicolumn{1}{c}{20.61} & \multicolumn{1}{c}{30.23}  
& \multicolumn{1}{c}{73.84} & \multicolumn{1}{c}{{88.87}} & \multicolumn{1}{c}{\textbf{86.06}} 
& \multicolumn{1}{c}{15.53} & \multicolumn{1}{c}{16.38} & \multicolumn{1}{c}{18.67} 
& \multicolumn{1}{c}{55.03} & \multicolumn{1}{c}{51.51} & \multicolumn{1}{c}{57.14} \\ \midrule
  \rowcolor{gray!12}\bf GMAE (Ours) & \multicolumn{1}{c}{\underline{96.25}} & \multicolumn{1}{c}{\underline{90.26}} & \multicolumn{1}{c}{\underline{96.37}} & \multicolumn{1}{c}{{75.49}} & \multicolumn{1}{c}{85.26} & \multicolumn{1}{c}{{78.42}} & \multicolumn{1}{c}{17.28} & \multicolumn{1}{c}{17.05} & \multicolumn{1}{c}{\textbf{31.59}} & \multicolumn{1}{c}{\underline{94.29}} & \multicolumn{1}{c}{\underline{88.16}} & \multicolumn{1}{c}{\underline{94.29}} \\ \bottomrule
\end{tabular}
\end{center}\vspace{-1em}
\end{table*}

In the domain of \textit{Image-Text} datasets, the integration of visual and textual modalities opens new frontiers in data analysis. \textbf{RGB-D} provides indoor RGB and depth images, each paired with detailed textual descriptions, enabling sophisticated cross-modal retrieval and fusion tasks. \textbf{Wikipedia} dataset, derived from selected articles, includes documents paired with word frequency histograms for text and SIFT histograms for images, bridging the gap between language and vision. This dataset is a valuable resource for tasks such as image captioning and document classification.

The \textit{Omics} datasets explore the molecular intricacies of biological data, providing critical insights into genomic variation. In particular, \textbf{BRCA} is meticulously classified into four clinically significant subtypes: LuminalA, LuminalB, HER2-enriched and Basal-like, each representing distinct genomic profiles that advance understanding of breast cancer heterogeneity. Similarly, \textbf{LGG} dataset encompasses RNA expression, DNA copy number, and DNA methylation characteristics, stratified into three prognostic subtypes of lower-grade glioma. These provide essential resources for studying tumorigenesis and therapeutic response.

The \textit{Synthetics} category, represented by \textbf{Synthetic3D}, offers a controlled and well-defined environment for testing algorithms in synthetic scenarios. Each sample in this dataset is represented by three distinct synthetic feature sets, providing an idealized setting for algorithm evaluation.

\subsubsection{Evaluation Metrics} We use three common quantitative metrics, 
including unsupervised clustering accuracy (ACC), normalized mutual information (NMI), and purity (PUR).
The reported results are the average values (\%) of 10 final runs. Larger values of ACC/NMI/PUR indicate better Figures.
\subsubsection{Compared Methods}\label{sec:sota}
We compare our proposed method, GMAE, with 12 recent state-of-the-art deep learning-based multi-view clustering approaches, each with distinct methodological focuses.

\textbf{COMPLETER} \cite{lin2021completer} and \textbf{DCP} \cite{lin2022dual} employ contrastive learning to maximize mutual information between views and minimize conditional entropy, addressing both cross-view consistency and missing data recovery. \textbf{MFLVC} \cite{MFLVCxu2022multi} and \textbf{DealMVC} \cite{Dealmvcyang2023dealmvc} focus on reducing conflicts between common semantics and view-specific information, with MFLVC learning features at multiple levels and DealMVC using dual contrastive calibration for local and global consistency. \textbf{DSMVC} \cite{DSMVCtang2022deep}, \textbf{GCFAgg} \cite{Gcfaggyan2023gcfagg}, and \textbf{SCMVC} \cite{SCMVCwu2024self} use optimization frameworks that enhance feature selection and fusion, with DSMVC improving clustering stability across varying view counts, GCFAgg aggregating features across views to improve consistency, and SCMVC applying self-weighted contrastive fusion for robust feature learning. \textbf{CPSPAN} \cite{CPSPANjin2023deep} and \textbf{SDMVC} \cite{SDMVCxu2023self} both use self-training strategies to handle incomplete data, with CPSPAN aligning partial samples and prototypes across views, and SDMVC leveraging pseudo-labels for discriminative feature learning. \textbf{SURE} \cite{SUREyang2022robust} and \textbf{CVCL} \cite{CVCLchen2023deep} integrate autoencoders with contrastive learning, with SURE focusing on noise-robust learning under incomplete information, and CVCL learning view-invariant representations by contrasting cluster assignments across views. \textbf{MVCAN} \cite{MVCANxu2024investigating} tackles the challenge of noisy views in multi-view clustering by utilizing unshared parameters and inconsistent clustering predictions along with a two-level iterative optimization framework.

To ensure fairness, we did not fine-tune the hyperparameters of our method for any specific dataset. Similarly, all comparative methods were also restricted from performing dataset-specific hyperparameter fine-tuning, relying instead on a single pre-defined rules of hyperparameters. In cases where no such pre-set existed, the default hyperparameters from the first dataset were used. By comparing GMAE with these diverse methodologies, we aim to highlight its strengths in generating cohesive clustering results across multiple views, demonstrating its superior performance in various challenging multi-view clustering scenarios.

\begin{table*}[htbp]
\begin{center} 
\setlength{\tabcolsep}{8.6pt}
\renewcommand{\arraystretch}{1}
\setlength{\abovecaptionskip}{0.1cm}  
\caption{Comparison of clustering results of different methods on incomplete multi-view datasets.} \label{table5:results5}
\begin{tabular}{p{2.64cm}p{2.18cm}<{\centering}p{2.18cm}<{\centering}p{2.18cm}<{\centering}p{2.18cm}<{\centering}p{2.18cm}<{\centering}p{2.18cm}<{\centering}p{2.18cm}<{\centering}p{2.18cm}<{\centering}p{2.18cm}<{\centering}p{2.18cm}<{\centering}p{2.18cm}<{\centering}p{2.18cm}<{\centering}}
\toprule
   \multirow{2}{1in}{{Methods \\(missing rate 0.5)}}   & \multicolumn{3}{c}{BRCA}     & \multicolumn{3}{c}{LGG}    & \multicolumn{3}{c}{Dermatology}   & \multicolumn{3}{c}{MSRCV1}   \\  \cmidrule(lr){2-4} \cmidrule(lr){5-7} \cmidrule(lr){8-10} \cmidrule(lr){11-13} 
   & \multicolumn{1}{c}{ACC}  & \multicolumn{1}{c}{NMI} & \multicolumn{1}{c}{PUR} & \multicolumn{1}{c}{ACC}  & \multicolumn{1}{c}{NMI} & \multicolumn{1}{c}{PUR} & \multicolumn{1}{c}{ACC}  & \multicolumn{1}{c}{NMI} & \multicolumn{1}{c}{PUR} & \multicolumn{1}{c}{ACC}  & \multicolumn{1}{c}{NMI} & \multicolumn{1}{c}{PUR}  \\ \midrule	 		
     COMPLETER{\tiny\textcolor{gray}{[CVPR’21]}}\cite{lin2021completer}     & \multicolumn{1}{c}{37.69} & \multicolumn{1}{c}{7.52} & \multicolumn{1}{c}{48.49}  & \multicolumn{1}{c}{51.31} & \multicolumn{1}{c}{18.46} & \multicolumn{1}{c}{56.18} & \multicolumn{1}{c}{46.37} & \multicolumn{1}{c}{39.53} & \multicolumn{1}{c}{{55.31}} & \multicolumn{1}{c}{41.43} & \multicolumn{1}{c}{38.57} & \multicolumn{1}{c}{44.76} \\
   DCP{\tiny\textcolor{gray}{[TPAMI’22]}}\cite{lin2022dual}    & \multicolumn{1}{c}{40.20} & \multicolumn{1}{c}{11.85} & \multicolumn{1}{c}{47.74}  & \multicolumn{1}{c}{56.93} & \multicolumn{1}{c}{17.68} & \multicolumn{1}{c}{56.93} & \multicolumn{1}{c}{43.58} & \multicolumn{1}{c}{37.23} & \multicolumn{1}{c}{51.40} & \multicolumn{1}{c}{36.67} & \multicolumn{1}{c}{29.55} & \multicolumn{1}{c}{38.57} \\  			
   DSMVC{\tiny\textcolor{gray}{[CVPR’22]}}\cite{DSMVCtang2022deep}    & \multicolumn{1}{c}{40.70} & \multicolumn{1}{c}{4.95} & \multicolumn{1}{c}{42.96}  & \multicolumn{1}{c}{52.81} & \multicolumn{1}{c}{18.69} & \multicolumn{1}{c}{57.68} & \multicolumn{1}{c}{47.49} & \multicolumn{1}{c}{33.67} & \multicolumn{1}{c}{53.91} & \multicolumn{1}{c}{23.33} & \multicolumn{1}{c}{6.43} & \multicolumn{1}{c}{23.33} \\  
   
     SURE{\tiny\textcolor{gray}{[TPAMI’22]}}\cite{SUREyang2022robust}     & \multicolumn{1}{c}{38.44} & \multicolumn{1}{c}{10.09} & \multicolumn{1}{c}{49.50}  & \multicolumn{1}{c}{53.56} & \multicolumn{1}{c}{16.31} & \multicolumn{1}{c}{54.31} & \multicolumn{1}{c}{38.27} & \multicolumn{1}{c}{18.68} & \multicolumn{1}{c}{42.18} & \multicolumn{1}{c}{42.38} & \multicolumn{1}{c}{26.61} & \multicolumn{1}{c}{45.24} \\  	 																	
    DealMVC{\tiny\textcolor{gray}{[MM’23]}}\cite{Dealmvcyang2023dealmvc}     & \multicolumn{1}{c}{40.70} & \multicolumn{1}{c}{4.95} & \multicolumn{1}{c}{42.96}  & \multicolumn{1}{c}{\underline{67.42}} & \multicolumn{1}{c}{34.34} & \multicolumn{1}{c}{\underline{67.42}} & \multicolumn{1}{c}{\textbf{71.51}} & \multicolumn{1}{c}{55.46} & \multicolumn{1}{c}{\underline{72.35}} & \multicolumn{1}{c}{\underline{48.10}} & \multicolumn{1}{c}{\underline{41.04}} & \multicolumn{1}{c}{\underline{48.10}} \\
   								
    GCFAgg{\tiny\textcolor{gray}{[CVPR’23]}}\cite{Gcfaggyan2023gcfagg}     & \multicolumn{1}{c}{46.23} & \multicolumn{1}{c}{21.19} & \multicolumn{1}{c}{59.05}  & \multicolumn{1}{c}{12.20} & \multicolumn{1}{c}{\textbf{49.44}} & \multicolumn{1}{c}{56.55} & \multicolumn{1}{c}{48.50} & \multicolumn{1}{c}{\underline{59.22}} & \multicolumn{1}{c}{62.85} & \multicolumn{1}{c}{15.62} & \multicolumn{1}{c}{28.10} & \multicolumn{1}{c}{29.05} \\ 
    
   SCMVC{\tiny\textcolor{gray}{[TMM’24]}}\cite{SCMVCwu2024self}     & \multicolumn{1}{c}{\underline{48.49}} & \multicolumn{1}{c}{\underline{24.71}} & \multicolumn{1}{c}{\underline{59.80}}  & \multicolumn{1}{c}{54.31} & \multicolumn{1}{c}{23.98} & \multicolumn{1}{c}{63.30} & \multicolumn{1}{c}{48.60} & \multicolumn{1}{c}{35.85} & \multicolumn{1}{c}{54.75} & \multicolumn{1}{c}{39.52} & \multicolumn{1}{c}{36.72} & \multicolumn{1}{c}{44.29} \\ 	
   \midrule	
    \rowcolor{gray!12}\bf GMAE (Ours)   & \multicolumn{1}{c}{\textbf{51.01}} & \multicolumn{1}{c}{\textbf{32.42}} & \multicolumn{1}{c}{\textbf{64.82}}  & \multicolumn{1}{c}{\textbf{74.16}} & \multicolumn{1}{c}{\underline{34.91}} & \multicolumn{1}{c}{\textbf{74.16}} & \multicolumn{1}{c}{\underline{55.59}} & \multicolumn{1}{c}{\textbf{62.15}} & \multicolumn{1}{c}{\textbf{73.46}} & \multicolumn{1}{c}{\textbf{54.62}} & \multicolumn{1}{c}{\textbf{49.00}} & \multicolumn{1}{c}{\textbf{57.62}} \\  \bottomrule

\end{tabular}
\end{center}\vspace{-1em}
\end{table*}
\begin{figure*}[htbp]
  \centering
        \subfigure[MSRCV1 (Complete)] { 
        \includegraphics[width=0.48\columnwidth]{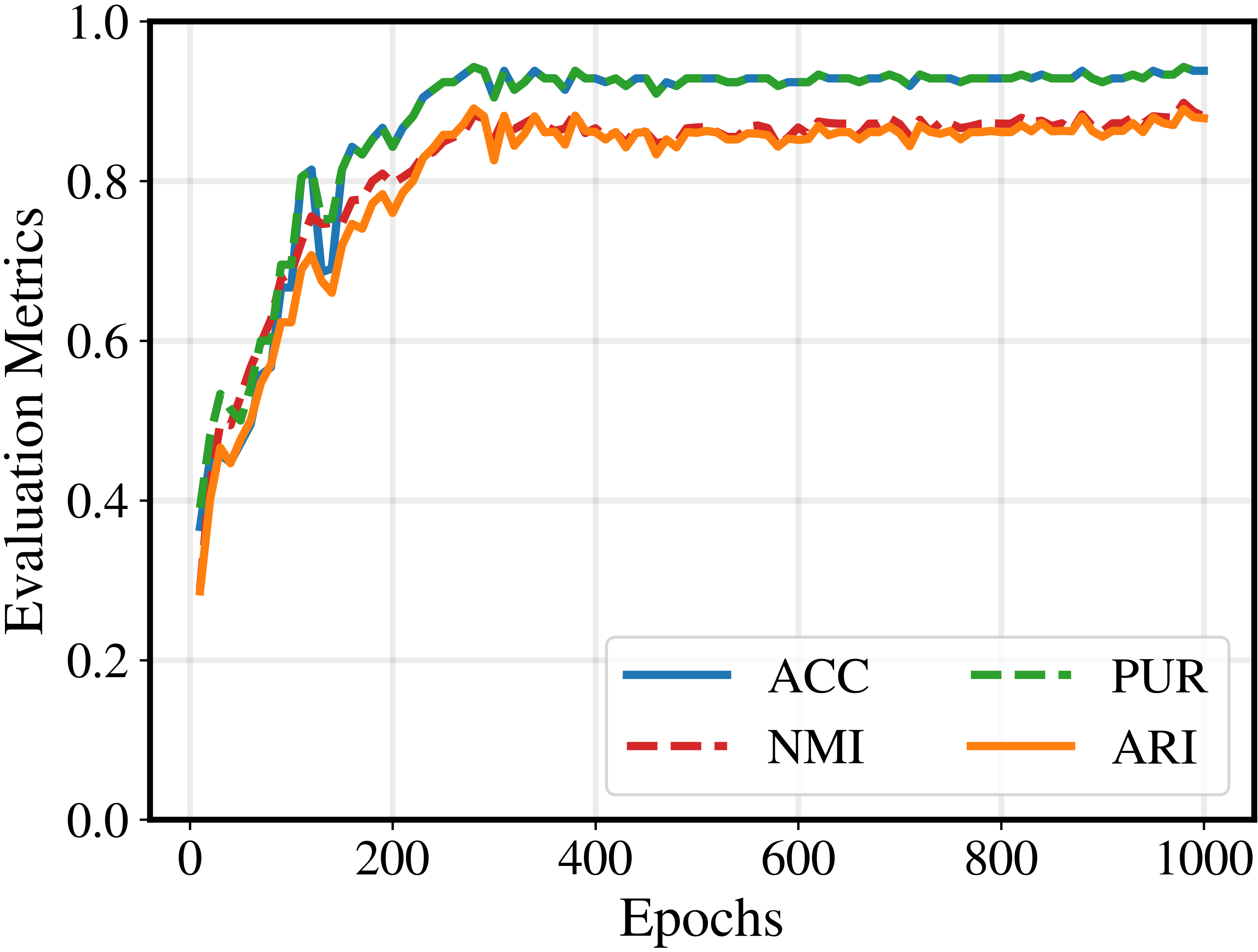}
        }\hspace{-0.1cm}
        \subfigure[MSRCV1 (Incomplete)] { 
        \includegraphics[width=0.48\columnwidth]{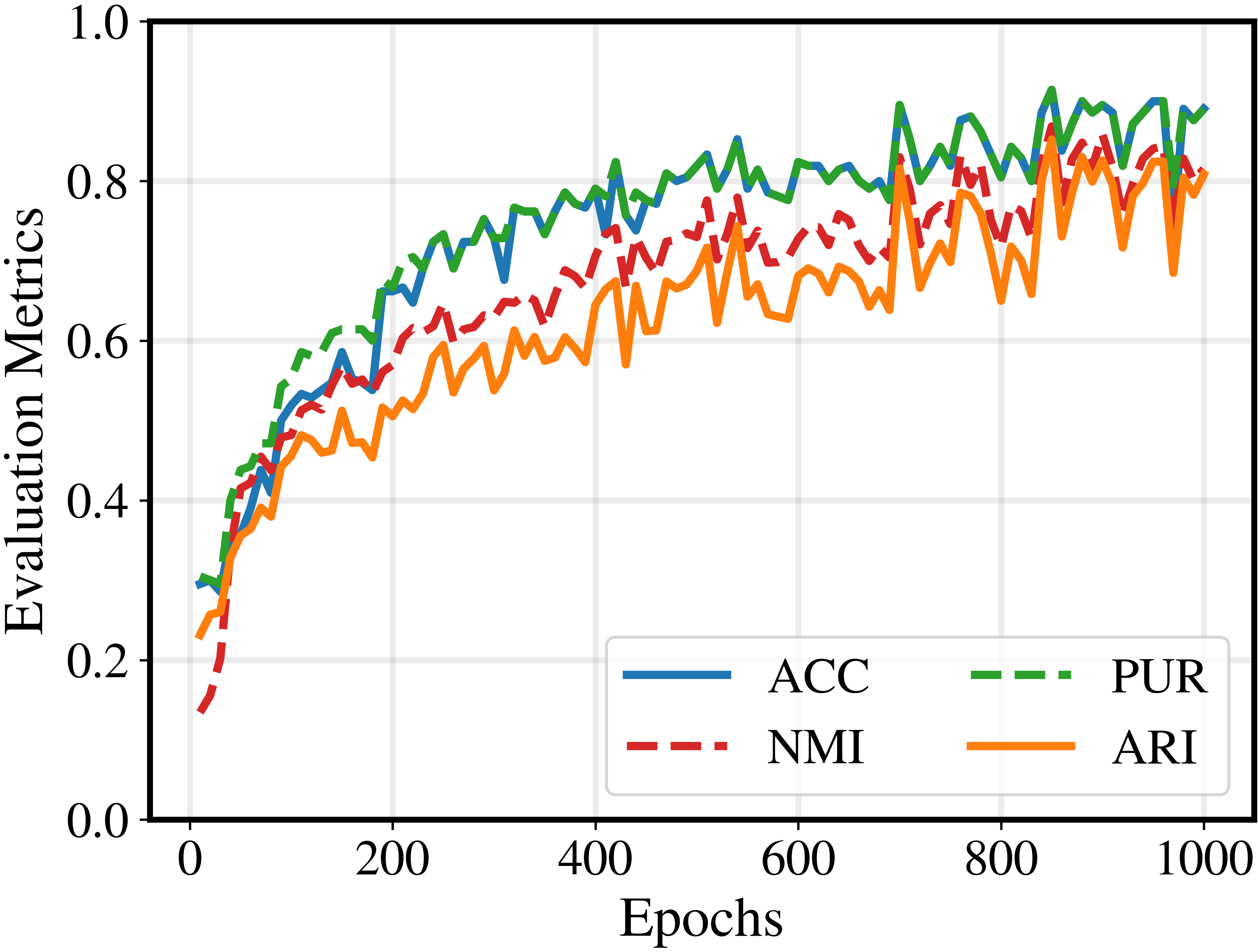}
        }\hspace{-0.1cm}
        \subfigure[Digits-2V (Complete)] { 
        \includegraphics[width=0.48\columnwidth]{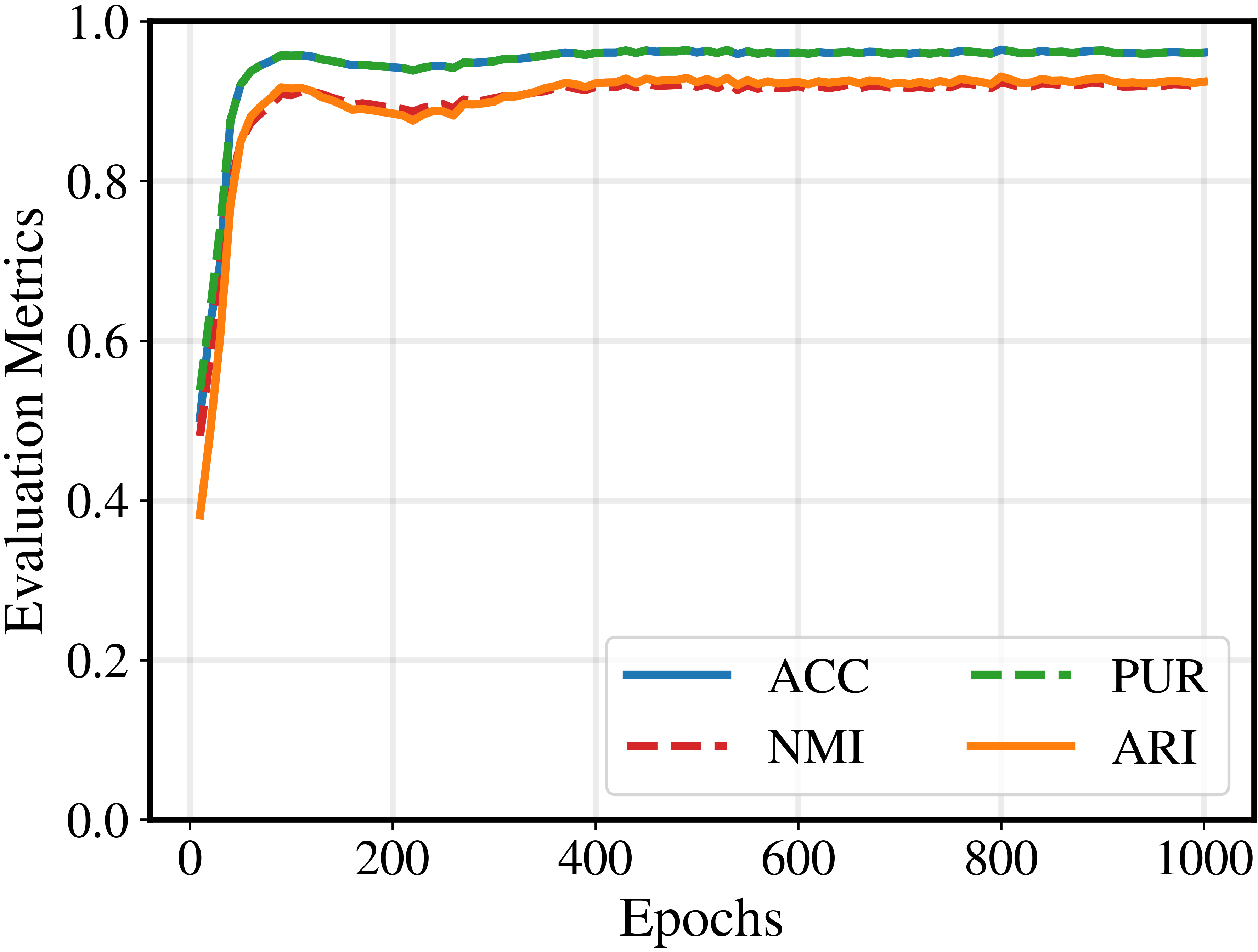}
        }\hspace{-0.1cm}
        \subfigure[Digits-2V (Classification)] { 
        \includegraphics[width=0.48\columnwidth]{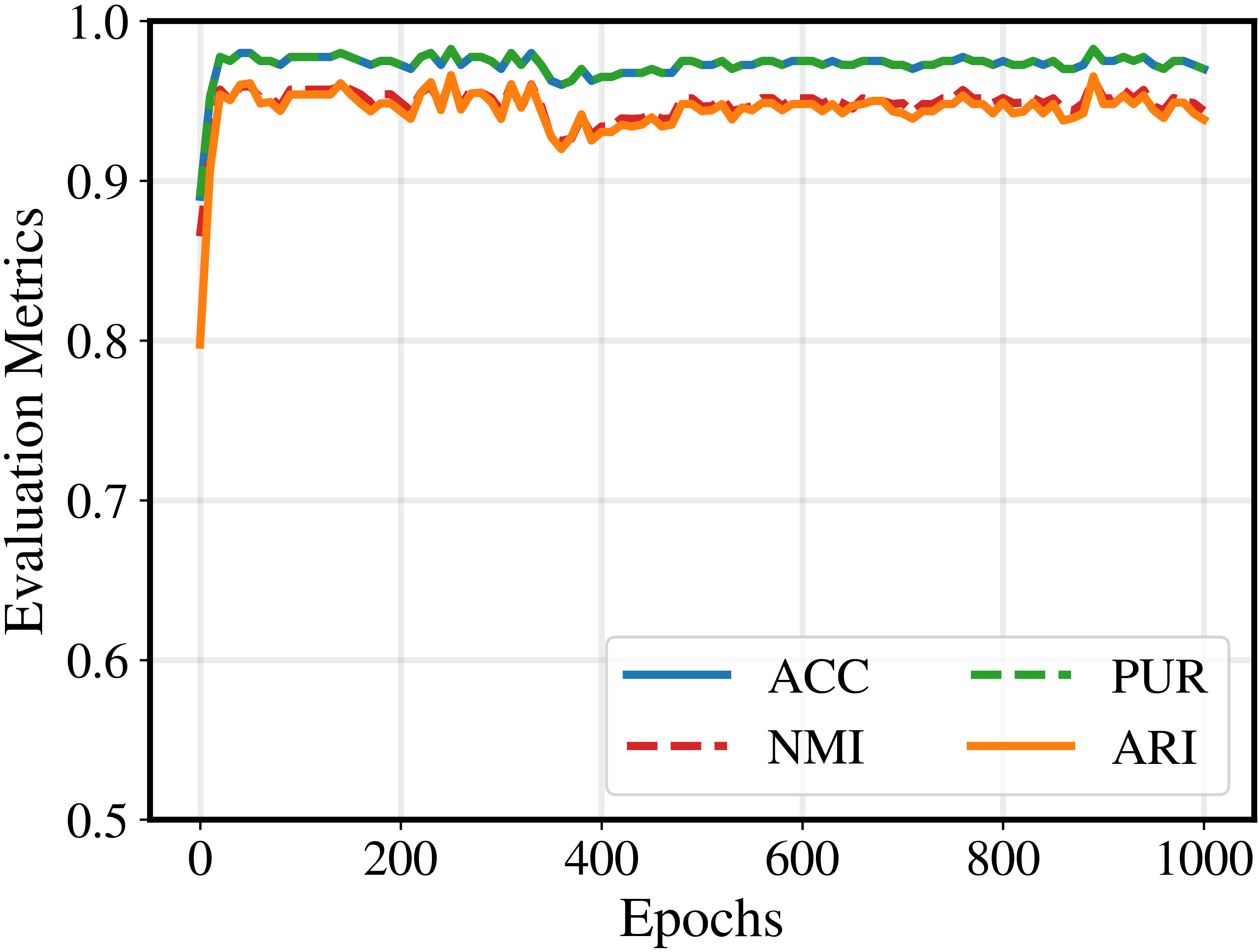}
        }\vspace{-1em}
\caption{\revised{(a-d) An analysis of the training process, including complete and incomplete multi-view datasets, respectively.}}
\label{fig:resline} \vspace{-1em}
\end{figure*}
\subsubsection{Implementation Details}
\revised{Our method is implemented using PyTorch 1.13.1+cu116 with Python 3.8.15 and tested on an AMD Ryzen 9 5900HX CPU with 32GB of RAM and an Nvidia RTX 3080 GPU (16GB memory). The models are trained using the Adam optimizer with a learning rate of 0.001. To ensure the reproducibility of the models, we set the random seed to 42 in PyTorch. In this work, we reshape all datasets into vectors and implement autoencoders using a fully connected network. We train the model for 500 epochs on all datasets, using a batch size equal to the total number of samples in each dataset. As the sampling process is random, the final experimental results may vary slightly. \textsc{Open Source} \footnote{Code is available at \url{https://github.com/obananas/GMAE}}.}

\revised{We also evaluate the performance of various Deep MVC methods on incomplete MVC datasets, which are generated from the original datasets by introducing varying levels of missing data. The incomplete samples are created based on the COMPLETER method in BRCA, LGG, Dermatology, and MSRCV1, by randomly removing views while ensuring that at least one view remains for each sample. }

\revised{The specific process is as follows: given a missing ratio ($m_r$) ranging from 0 to 1, we randomly select a proportion of samples equal to $m_r$ of the total dataset. For each selected sample, we randomly choose between one and $n-1$ views (where $n$ is the total number of views) and set all data in the selected views to zero to simulate missing data.}

\subsection{Experimental Comparative Results}
To ensure fairness, all comparative experiment results are based on the final round rather than the best outcomes. Tables \ref{table1:results1}, and \ref{table4:results4} present the performance of all comparative methods across a variety of dataset types. Table \ref{table4:results4} further explores the impact of varying the number of views within the same dataset on Figures, while Table \ref{table5:results5} and Figure \ref{fig:miss} validate the effectiveness of the proposed method on datasets with random missing views. We can draw the following noteworthy observations:

\ding{182} Our method achieved 15 first-place, 3 second-place, and 2 third-place finishes across 20 tests, highlighting its strong generalizability and potential in handling various data distributions and various downstream tasks. By aligning the distribution between view-specific and view-common features, GMAE adapts flexibly to the unique characteristics of different views, effectively capturing both distinct information from each view and the correlations between views. This adaptability enables GMAE to excel even on highly heterogeneous datasets.

\ding{183} GMAE presents exceptional stability during training and test. In all 20 tests, GMAE consistently produced stable and reliable results, unlike methods such as CPSPAN on Wikipedia or DCP on STL-10, which experienced significant performance fluctuations or dramatic declines with increased training epochs. This underscores GMAE's practical utility, ensuring near-optimal performance as long as sufficient training is provided, without needing to pause at an exact moment to capture a fleeting optimal result.

\ding{184} \revised{GMAE maintains high performance even in the presence of missing data, showcasing the robustness of the model. This robustness primarily stems from its internal feature alignment and disentanglement mechanisms, which mitigate the negative impact of noise or anomalous data, allowing the model to make accurate decisions in imperfect data environments, which is important in real applications.}

\ding{185} On the Digits dataset, GMAE's performance steadily improved as the number of views increased, in contrast to methods like DSMVC and SURE, where performance deteriorated with additional views. This suggests that the multi-view disentangled representation learning architecture offers strong multi-view safeness, enabling GMAE to accommodate data scale growth and the introduction of new views without suffering from noise or irrelevant information that could otherwise degrade overall performance. This feature is crucial in dynamic environments, especially where data is continually expanding, ensuring that GMAE can positively leverage the added view information, avoiding view conflicts and preventing negative transfer.

\begin{figure*}[htbp]
  \centering
\subfigure[Dermatology (ACC)] { 
\includegraphics[width=0.233\textwidth]{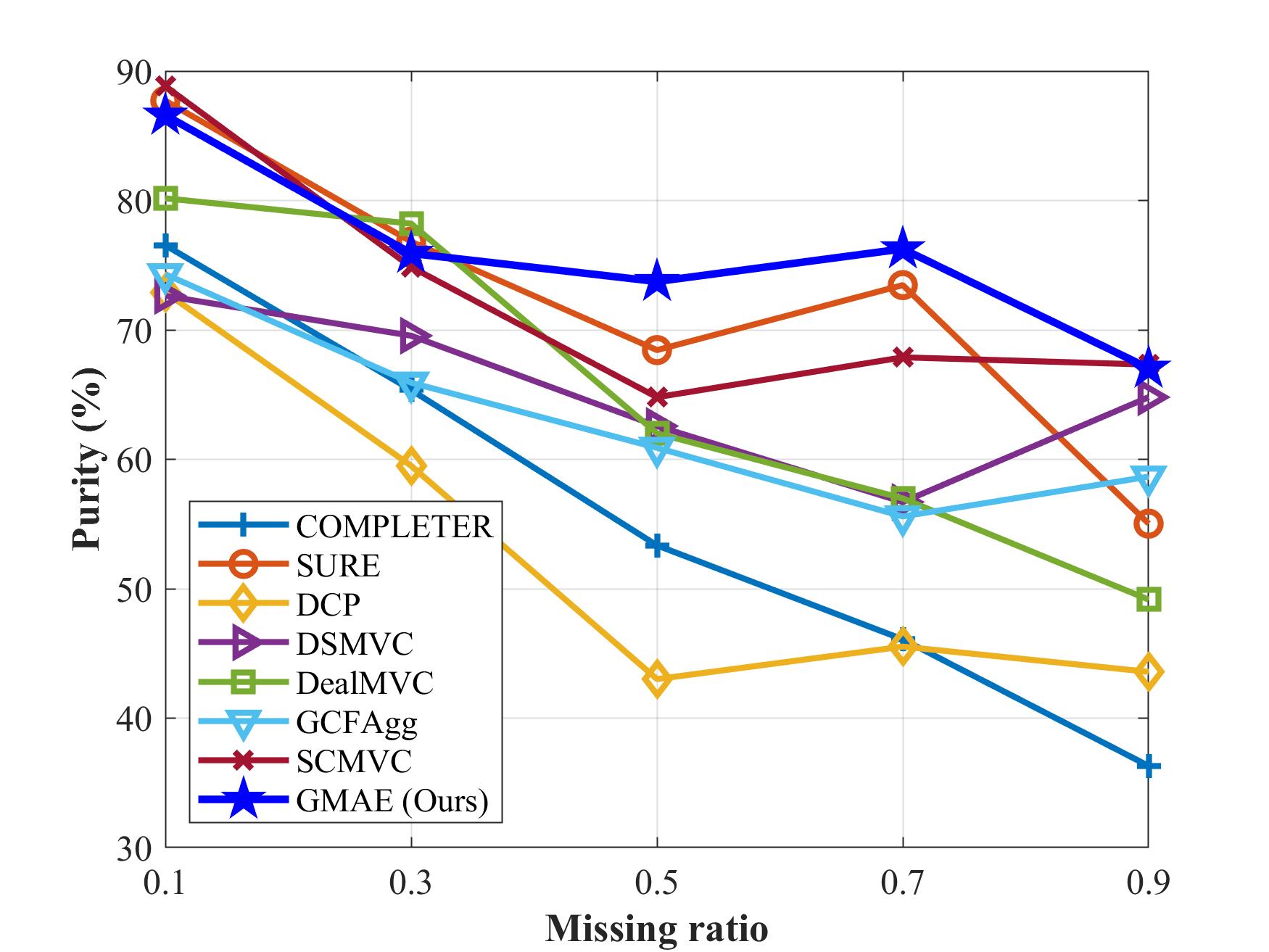}
}\hspace{0.05cm}
\subfigure[Dermatology (NMI)] { 
\includegraphics[width=0.233\textwidth]{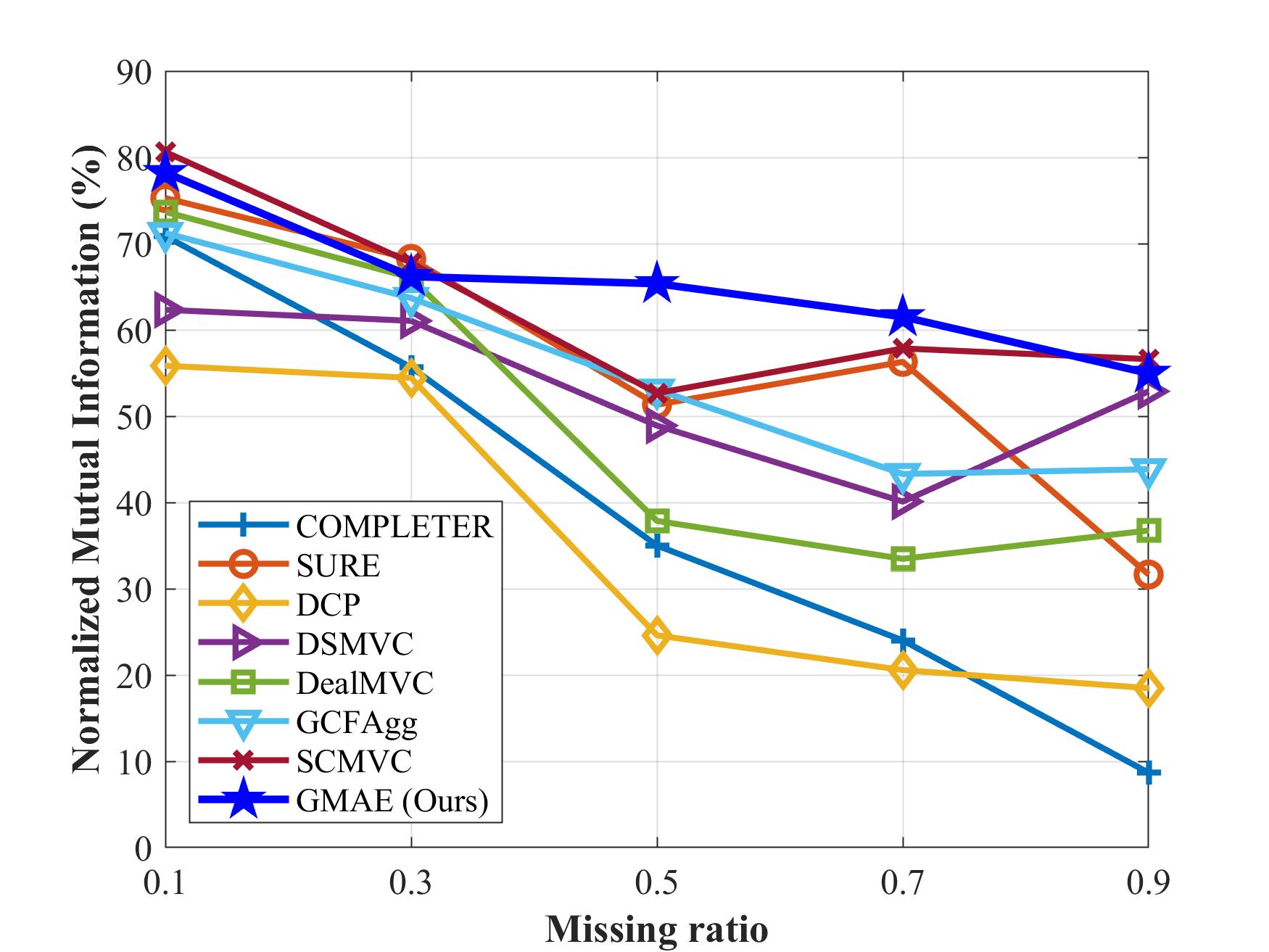}
}\hspace{0.05cm}
\subfigure[MSRCV1 (ACC)] { 
\includegraphics[width=0.233\textwidth]{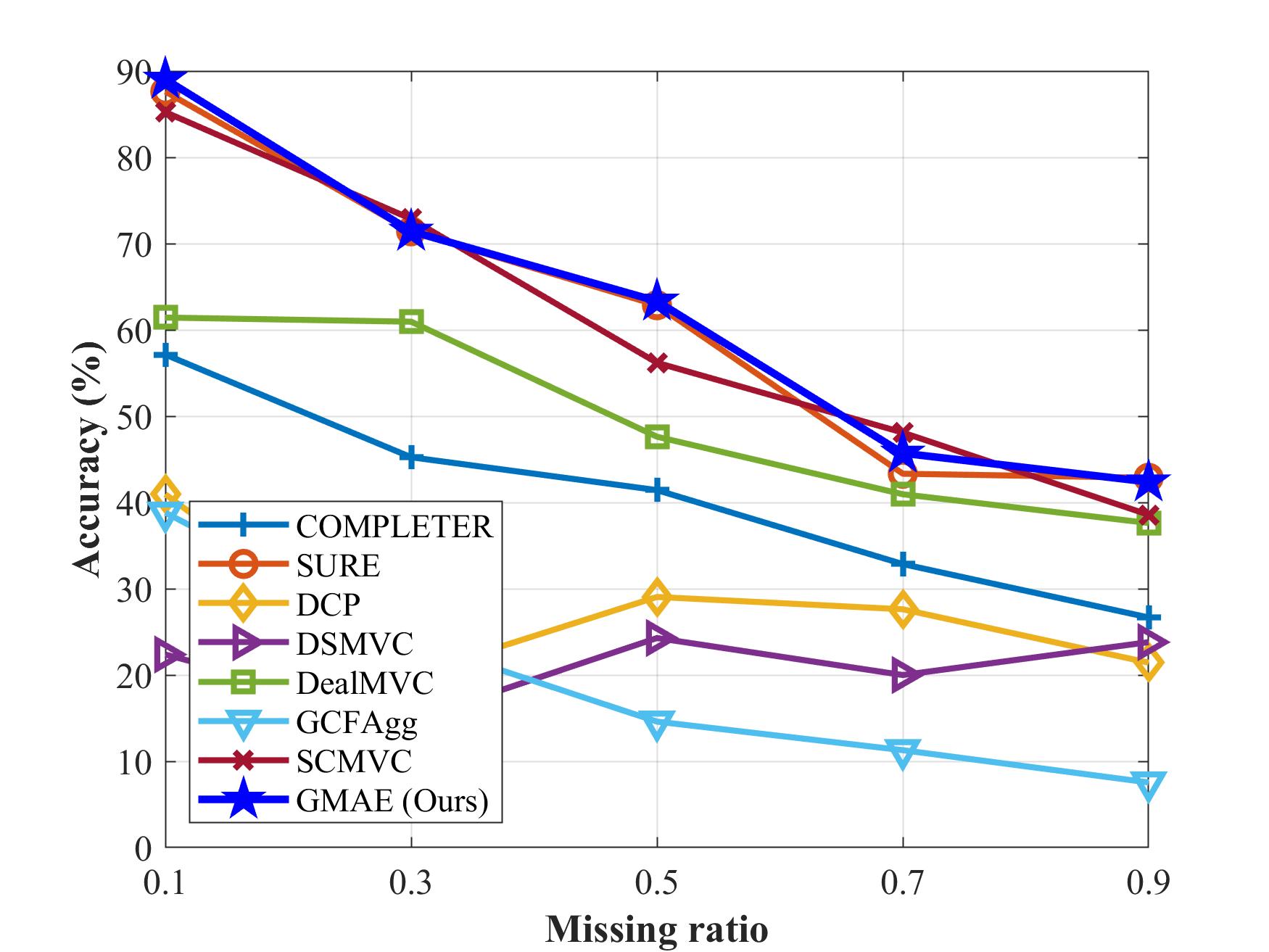}
}\hspace{0.05cm}
\subfigure[MSRCV1 (NMI)] { 
\includegraphics[width=0.233\textwidth]{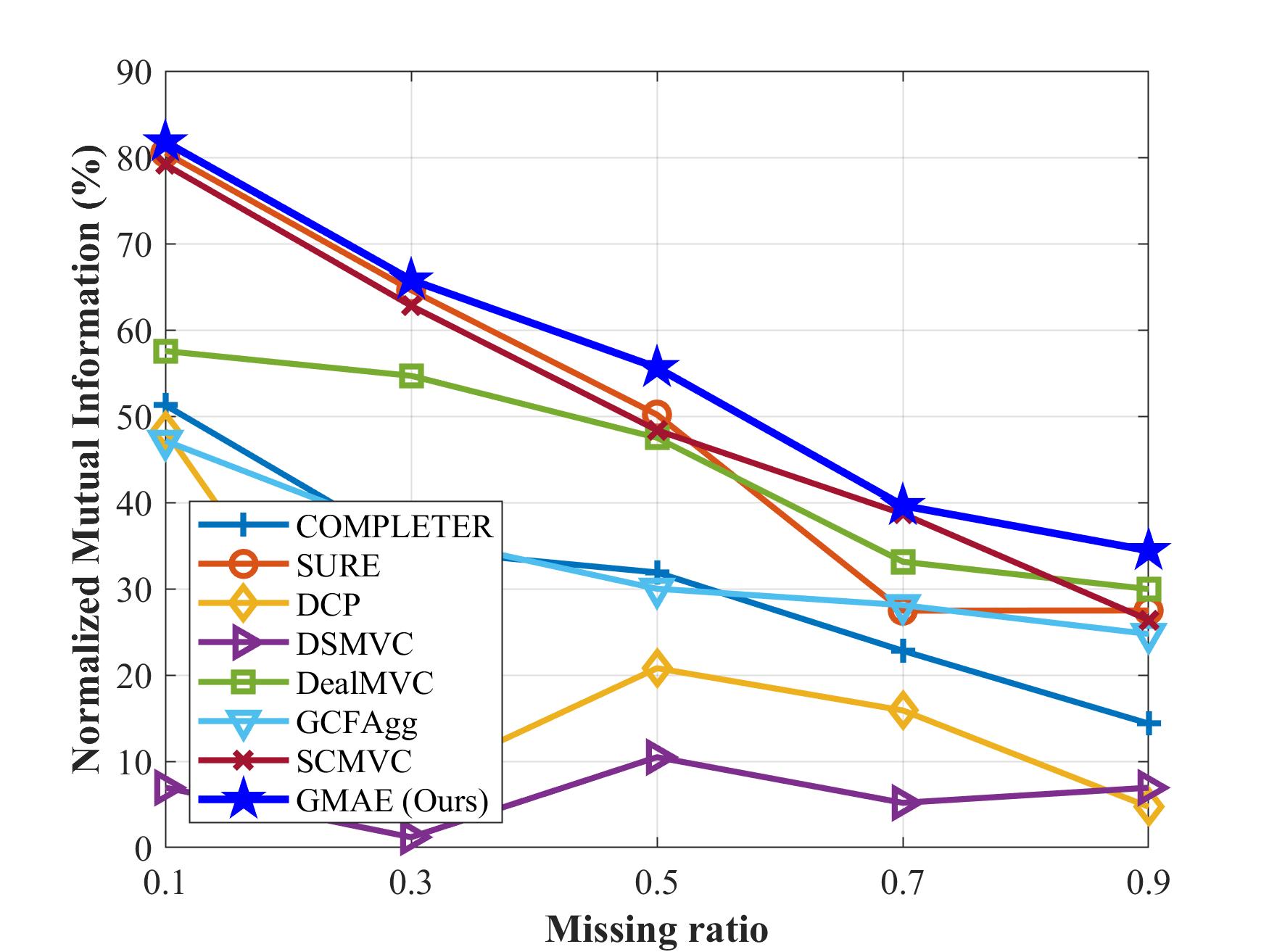}
}
\caption{Effect of different missing radio on (ACC/NMI) evaluation metrics on Dermatology and MSRCV1 datasets.}
\label{fig:miss} 
\end{figure*}

\ding{186} As shown in Figure \ref{fig:miss}, to systematically investigate the relationship between missing rates and the performance of multi-view clustering methods, the proportion of incomplete samples relative to the overall dataset was varied from 0.1 to 0.9, with an interval of 0.2. The results, presented as a line chart, indicate that while the performance of all methods generally declines as the missing rate increases, our method exhibits a significantly slower rate of decline. In contrast, competing methods experience a more pronounced degradation as the missing rate increases. These findings emphasize the superior resilience and effectiveness of our approach, particularly in scenarios characterized by high levels of data incompleteness.
\begin{figure}[htbp]
\centering
\subfigure[$\overline{\mathbf{Z}}$ features] { 
\includegraphics[width=0.32\columnwidth]{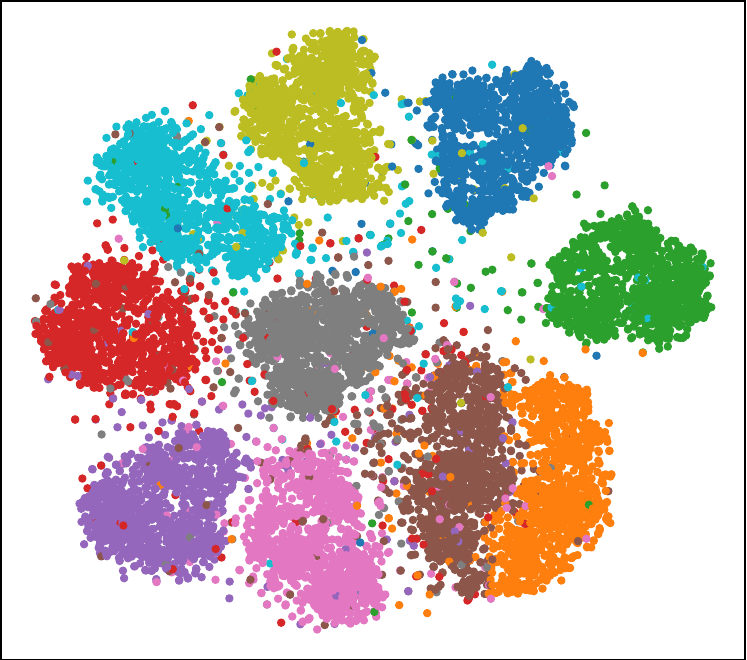}
}\hspace{-0.2cm}
\subfigure[$\overline{\mathbf{H}}$ features] { 
\includegraphics[width=0.32\columnwidth]{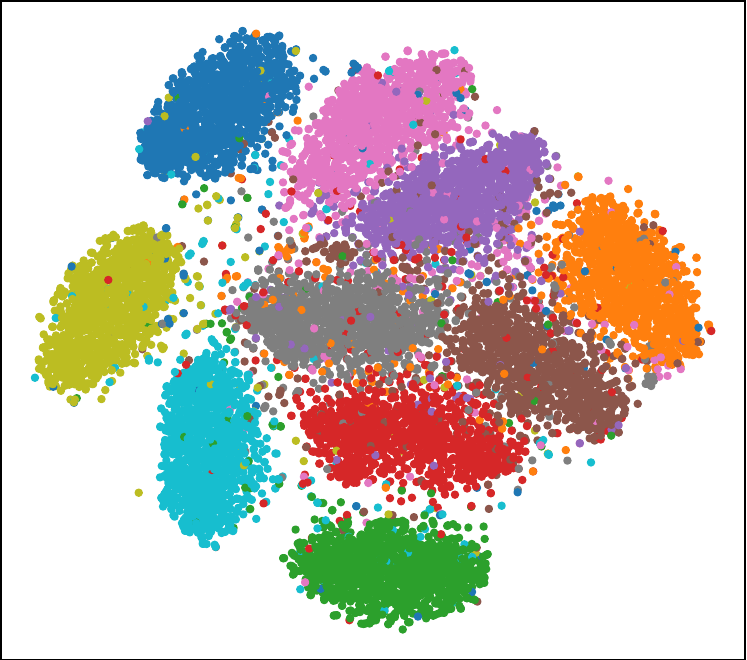}
}\hspace{-0.2cm}
\subfigure[${\mathbf{Q}}$ features] { 
\includegraphics[width=0.32\columnwidth]{Figures/0stl10.pdf}
}\vspace{-1em}
\caption{The visualization results of feature representations on different linear layers of the GMAE network on the STL-10 dataset after convergence.}
\label{fig:tsne2} \vspace{-1em}
\end{figure}

\begin{figure}[htbp]
\centering
\subfigure[Digits (97.45)] { 
\includegraphics[width=0.32\columnwidth]{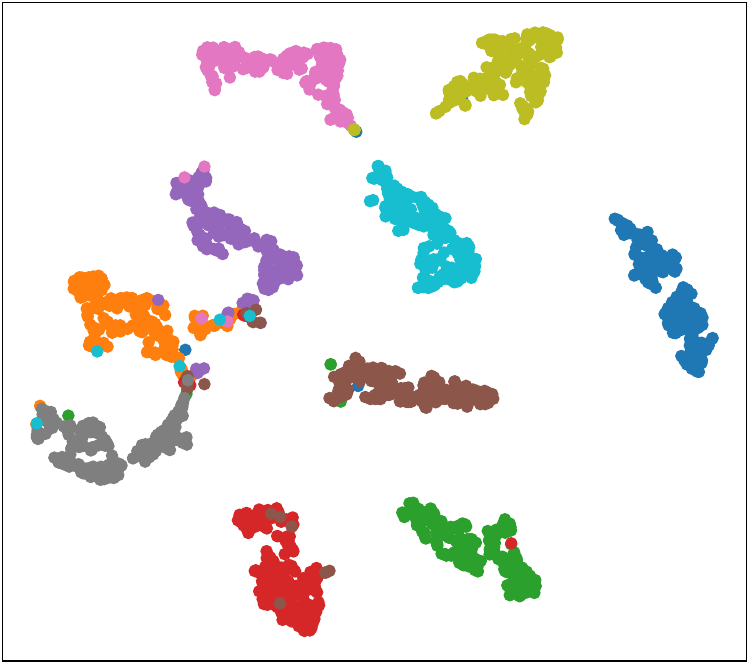}
}\hspace{-0.2cm}
\subfigure[RGB-D  (45.07)] { 
\includegraphics[width=0.32\columnwidth]{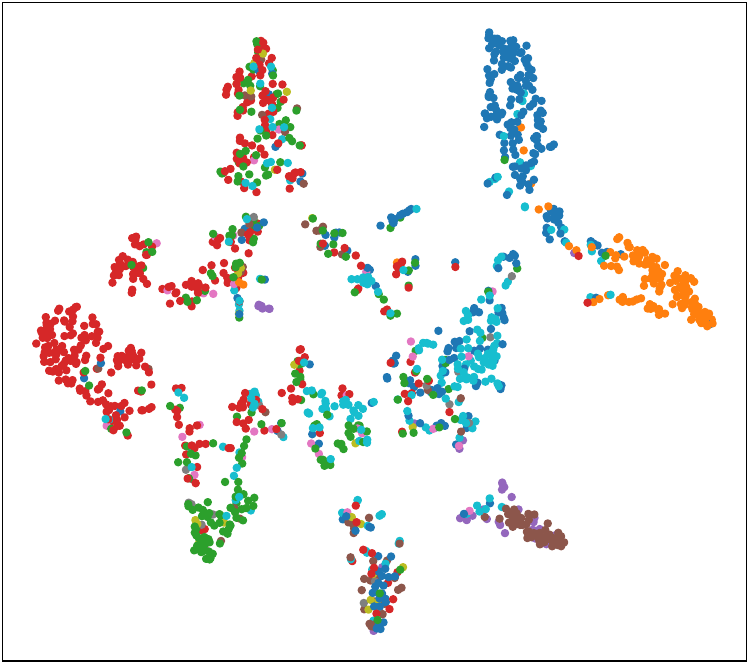}
}\hspace{-0.2cm}
\subfigure[ALOI-100 (75.49)] { 
\includegraphics[width=0.32\columnwidth]{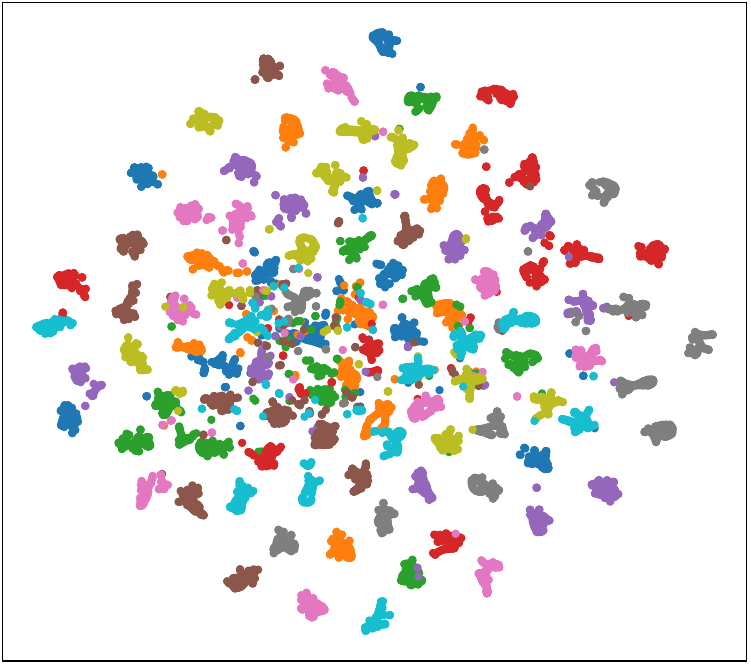}
}\vspace{-1em}
\caption{Visualization of embedded features ${\mathbf{Q}}$.}
\label{fig:tsne1} \vspace{-1em}
\end{figure}

\begin{figure*}[ht]
\centering
\subfigure[Digits-3V (ACC)] { 
\includegraphics[width=0.47\columnwidth]{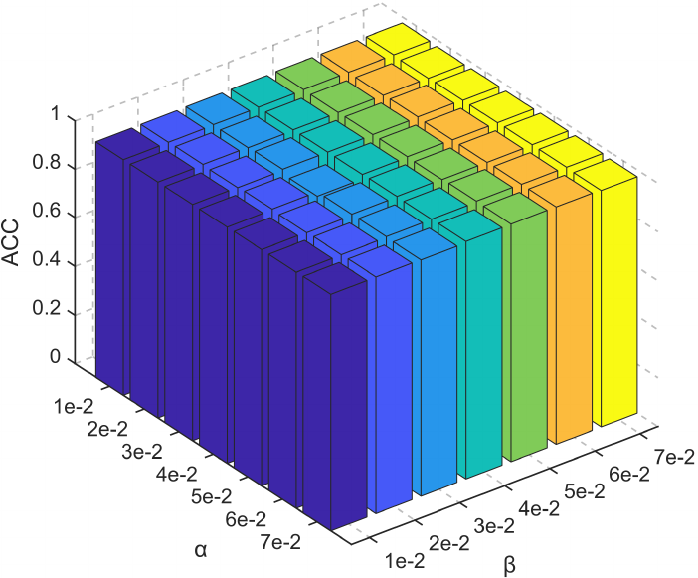}
}\hspace{-0.1cm}
\subfigure[Digits-3V (NMI)] { 
\includegraphics[width=0.47\columnwidth]{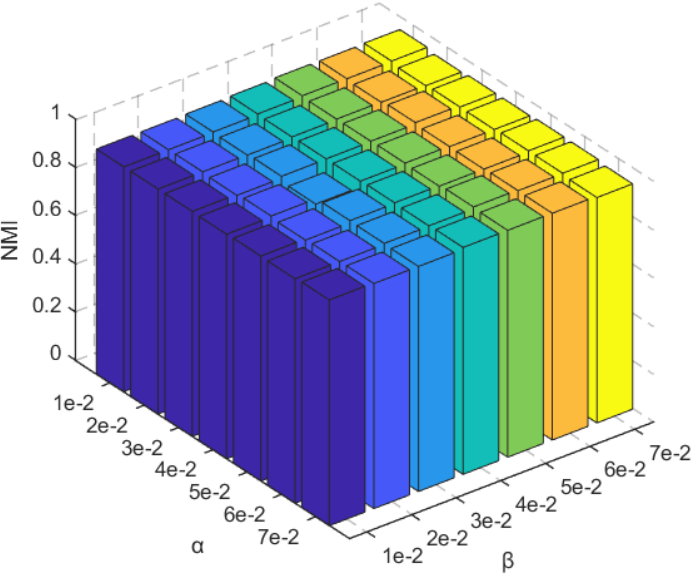}
}\hspace{-0.1cm}
\subfigure[LandUse-21 (ACC)] { 
\includegraphics[width=0.47\columnwidth]{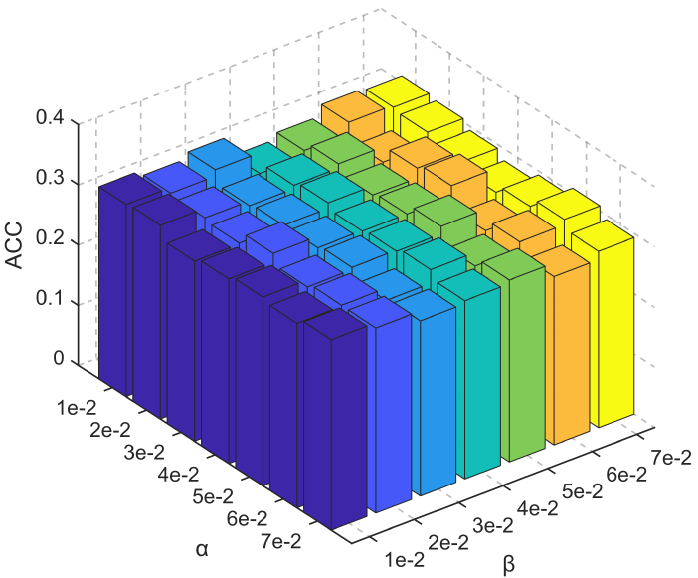}
}\hspace{-0.1cm}
\subfigure[LandUse-21 (NMI)] { 
\includegraphics[width=0.47\columnwidth]{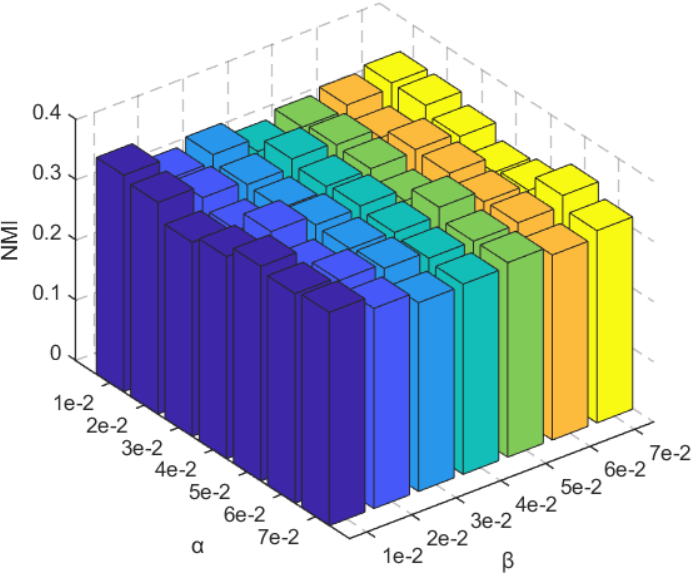}
}\vspace{-1em}
\caption{\revised{Hyperparameter sensitivity analysis. The clustering results (ACC/NMI) vary with different values of $\alpha$ and $\beta$.}}
\label{fig:bar} 
\end{figure*}
\begin{figure*}[ht]
\centering
\vspace{-1em}
\subfigure[BRCA (ACC)] { 
\includegraphics[width=0.47\columnwidth]{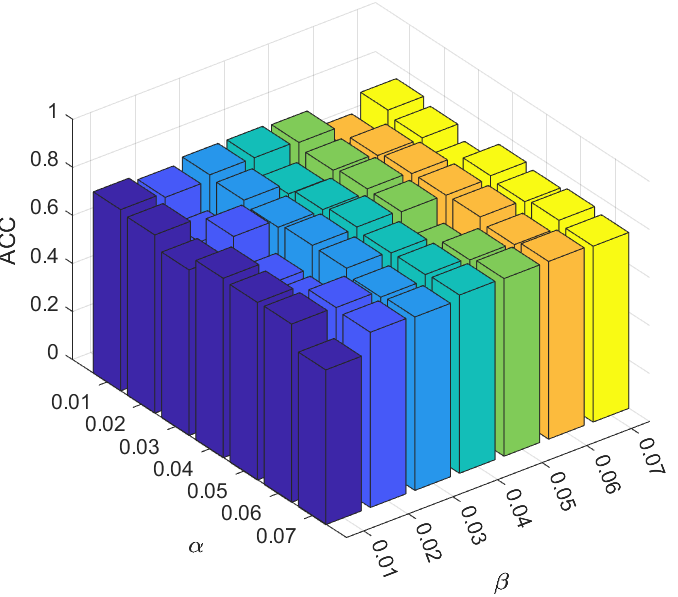}
}\hspace{-0.1cm}
\subfigure[BRCA (NMI)] { 
\includegraphics[width=0.47\columnwidth]{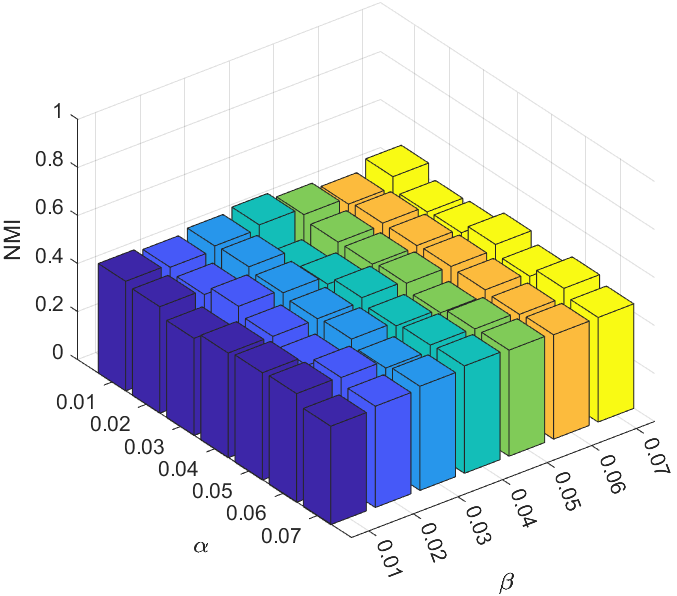}
}\hspace{-0.1cm}
\subfigure[LGG (PUR)] { 
\includegraphics[width=0.47\columnwidth]{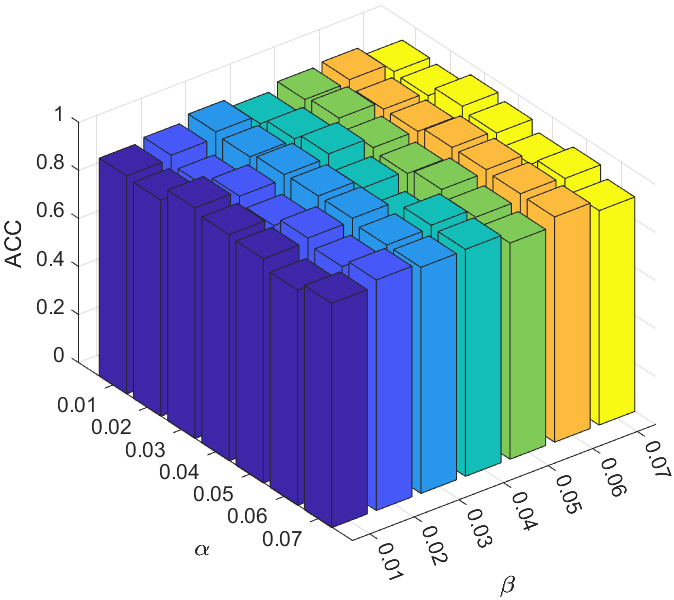}
}\hspace{-0.1cm}
\subfigure[LGG (NMI)] { 
\includegraphics[width=0.47\columnwidth]{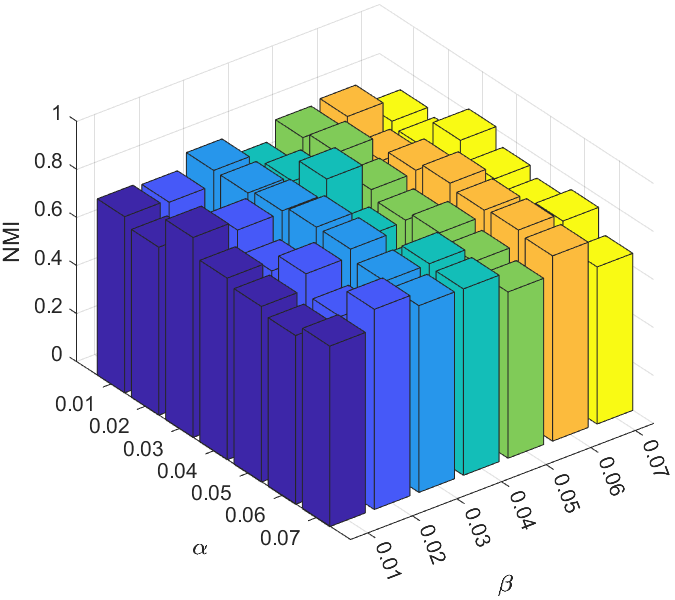}
}\vspace{-1em}
\caption{\revised{The hyperparameter sensitivity analysis on the omics-based MVC datasets varies across different values of $\alpha$ and $\beta$.}}
\label{fig:bar_2}
\end{figure*}

\subsection{Ablation Study}
\revised{The \(Z\), \(H\), and \(C\) embeddings play distinct roles at different stages in GMAE, reflecting the progressive refinement of feature learning. The t-SNE visualizations in Figure \ref{fig:tsne2} offer deeper insights into the evolution and function of these features. The view-specific representations \(Z\) are independently generated by the encoders for each view, capturing the unique characteristics of individual view inputs. While these features effectively capture view-specific content, the lack of cross-view fusion introduces redundancy and overlap between various categories. As shown in Figure \ref{fig:tsne2}(a), the clusters exhibit some structure but remain entangled, indicating that the learned representations at this stage have not yet achieved cross-view distribution alignment.  }

Meanwhile, the intermediate representations \(H\) integrate both view-specific information and initial shared patterns. At this stage, the model begins to capture cross-view commonalities, but incomplete optimization leaves some category boundaries blurred. As illustrated in Figure \ref{fig:tsne2}(b), the clusters are more distinct than those formed by \(Z\) features, with clearer boundaries and reduced overlap, reflecting improved discriminative ability. However, further refinement is needed to achieve full separation.  
\begin{table}[ht]
\centering
\renewcommand{\arraystretch}{1.} 
\setlength{\tabcolsep}{4pt} 
\caption{\revised{Ablative comparison of loss combinations.}}\vspace{-0.5em} 
\begin{tabular}{@{}c|ccc|ccc@{}}
\toprule
\multirow{2}{*}{{Datasets}} & \multicolumn{3}{c|}{{Components}} & \multicolumn{3}{c}{{Evaluation Metrics}} \\
\cmidrule(lr){2-4} \cmidrule(l){5-7}
 & \textbf{$\mathcal{L}_{\text{Rec}}$}& \textbf{$\mathcal{L}_{\text{Cor}}$}+\textbf{$\mathcal{L}_{\text{Dis}}$}& \textbf{$\mathcal{L}_{\text{Ent}}$}& {\quad ACC \quad} & {\quad NMI\quad} & {\quad PUR\;\quad} \\ 
\midrule
\multirow{4}{*}{{Digits-6V}}& \ding{51} & \textcolor{gray}{--} & \textcolor{gray}{--} & 64.10& 61.88& 65.40\\
& \ding{51} & \ding{51} & \textcolor{gray}{--} & 82.35& 85.46& 85.50\\
& \ding{51} & \textcolor{gray}{--} & \ding{51} & 83.05& 84.57& 83.05\\
& \ding{51} & \ding{51} & \ding{51} & \textbf{97.45}& \textbf{94.16}& \textbf{97.45}\\
\midrule
\multirow{4}{*}{{RGB-D}}& \ding{51} & \textcolor{gray}{--} & \textcolor{gray}{--} & 17.46& 5.20& 29.74\\
& \ding{51} & \ding{51} & \textcolor{gray}{--} & 16.77& 5.78& 28.30\\
& \ding{51} & \textcolor{gray}{--} & \ding{51} & 37.06& 34.84& 52.17\\
& \ding{51} & \ding{51} & \ding{51} & \textbf{45.07}& \textbf{40.06}& \textbf{58.94}\\
\midrule
\multirow{4}{*}{{Out-Scene}}& \ding{51} & \textcolor{gray}{--} & \textcolor{gray}{--} & 48.36& 44.02& 52.64\\
& \ding{51} & \ding{51} & \textcolor{gray}{--} & 20.54& 19.73& 24.22\\
& \ding{51} & \textcolor{gray}{--} & \ding{51} & 55.62& 59.09& 59.23\\
& \ding{51} & \ding{51} & \ding{51} & \textbf{75.16}& \textbf{62.14}& \textbf{76.73}\\
\bottomrule
\end{tabular}
\label{tab:loss_comparison}
\end{table}

Finally, the common representations \(C\) serve as the final optimized features, encapsulating aligned and complementary information from multiple views. By optimizing with correlation loss \(L_{\text{Cor}}\) and adversarial discrimination loss \(L_{\text{Dis}}\), GMAE effectively integrates multi-view information while ensuring that the features remain disentangled and independent. In Figure \ref{fig:tsne2}(c), the clustering structure of \(C\) features is compact and well-separated, with minimal overlap between categories, demonstrating that GMAE has successfully extracted high-quality fused features and achieved effective category separation.

As shown in Table \ref{tab:loss_comparison}, we conducted a series of experiments to systematically evaluate the impact of each loss function and its corresponding network module. Specifically, we assessed the model's performance under different combinations of loss functions. The results reveal that the evaluation metric reaches its lowest value when neither \textbf{$\mathcal{L}_{\text{Cor}}$}+\textbf{$\mathcal{L}_{\text{Dis}}$} nor \textbf{$\mathcal{L}_{\text{Ent}}$} is included. Adding either loss function individually leads to a noticeable improvement in performance, while incorporating both simultaneously achieves the most significant enhancement, outperforming the results obtained by adding any single loss function alone. This clearly demonstrates the complementary relationship between these loss functions and highlights their combined effectiveness in improving the model’s performance.

\subsection{Disentangled Representation Analysis}

Figure \ref{fig:tsne1} illustrates the feature distributions of various datasets after model convergence through t-SNE dimensionality reduction. It can be observed that, across datasets of different scales with an increasing number of clusters, the features learned by our model consistently form compact and well-separated clusters, indicating the model’s capability to effectively capture discriminative features among different categories. As shown in Figure \ref{fig:tsne1} (b) for the RGB-D dataset, slight overlaps appear between certain categories, which can be inferred to result from the similarity of cluster features across multiple views within the original dataset. 

In summary, the integrated features extracted by our proposed GMAE successfully produce dense clusters with well-defined boundaries while maintaining effective control over the emergence of outliers.

\subsection{Parameter Sensitivity Analysis}

\subsubsection{Feature embedding vector dimensions}
As the feature embedding vector dimensions increase, our proposed GMAE model exhibits a consistent trend in Figures across different datasets. Initially, performance improves with increasing dimensions, stabilizes at a certain point, and may even decline in some cases.

As shown in Fig. \ref{fig:fdims}, for low dimensions (\textit{e.g.}, $2^3$, $2^4$), the model struggles to capture data complexity, leading to poor clustering. As dimensions increase (\textit{e.g.}, $2^4$ to $2^6$), performance improves significantly due to better feature representation. However, at higher dimensions (\textit{e.g.}, $2^8$ and above), performance may plateau or decline, likely due to the curse of dimensionality, where increased noise and model complexity lead to overfitting.

The GMAE model's architecture, which uses an autoencoder to compress data and a multi-view design to integrate features, performs well at moderate dimensions. Yet, at very high dimensions, the model risks learning redundant features, reducing generalization.

The information bottleneck principle emphasizes retaining relevant information while discarding noise. Choosing the right dimension is crucial: too few dimensions cause information loss, while too many introduce noise. High dimensionality also increases model complexity, potentially leading to overfitting if the data is insufficient.

In conclusion, the feature embedding dimensions critically impact the GMAE model's clustering effectiveness, with 64 being the most generally effective dimension.

\begin{figure*}[htbp]
  \centering
\subfigure[Digits-6V] { 
\includegraphics[width=0.233\textwidth]{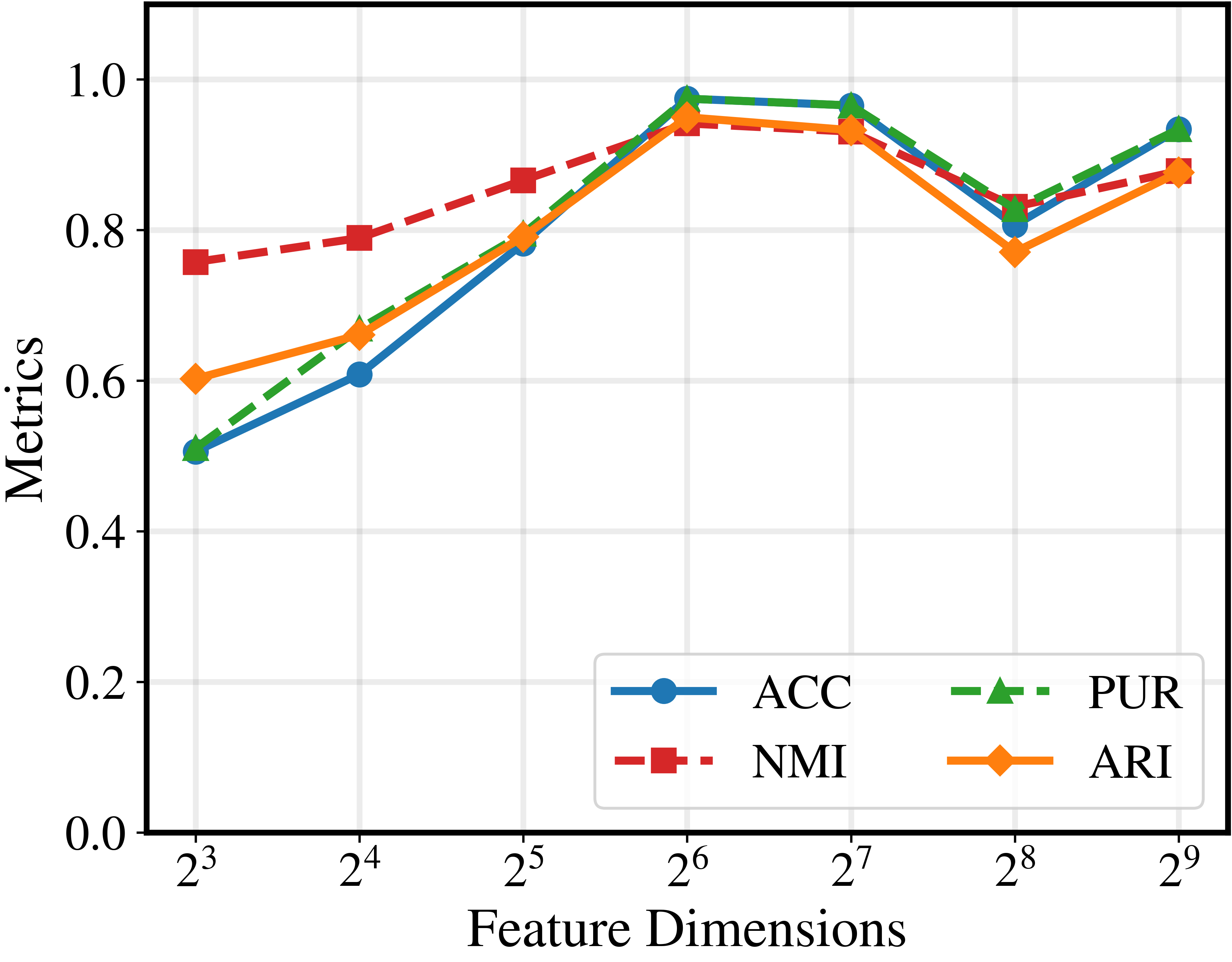}
}\hspace{0.05cm}
\subfigure[MSRCV1] { 
\includegraphics[width=0.233\textwidth]{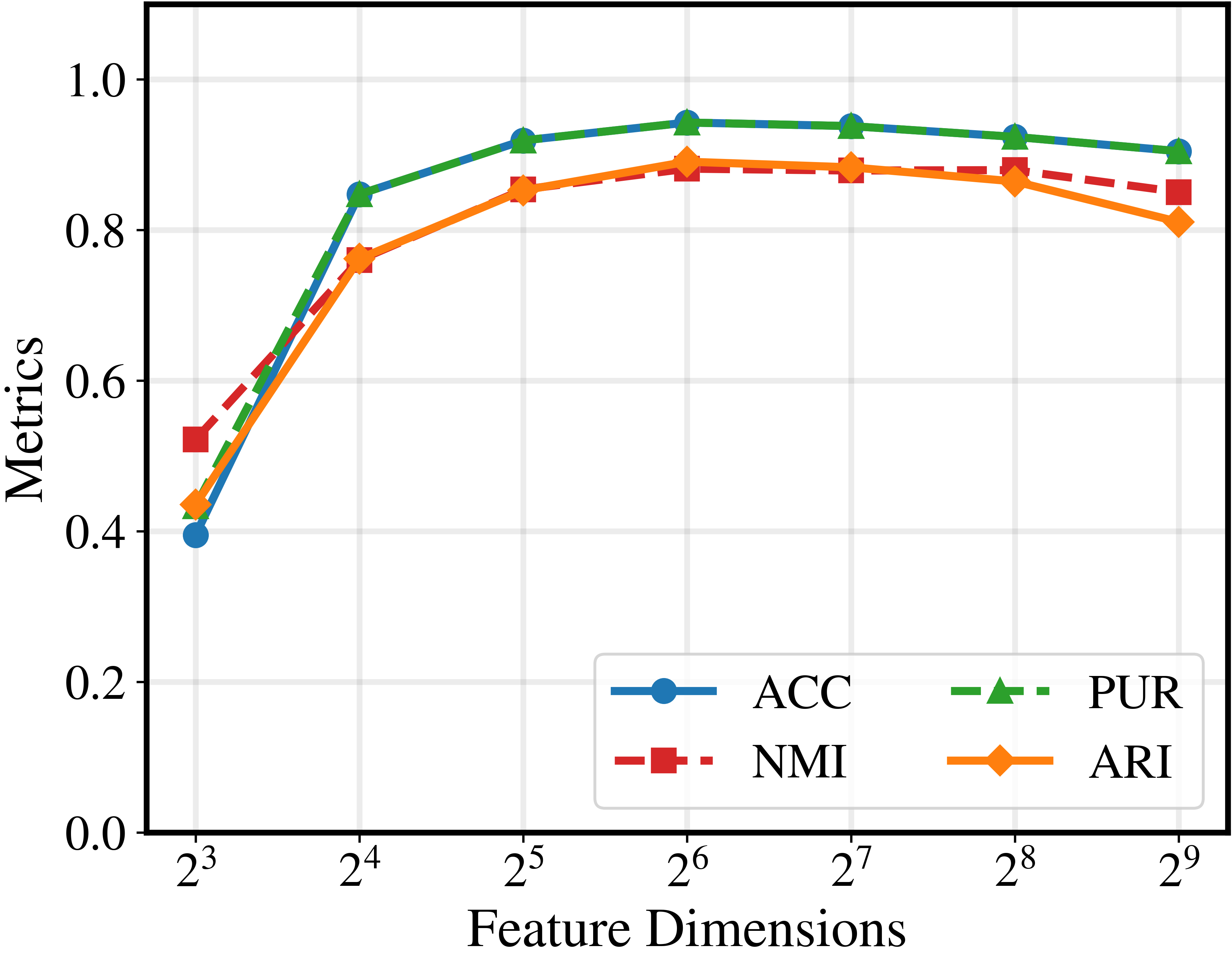}
}\hspace{0.05cm}
\subfigure[RGB-D] { 
\includegraphics[width=0.233\textwidth]{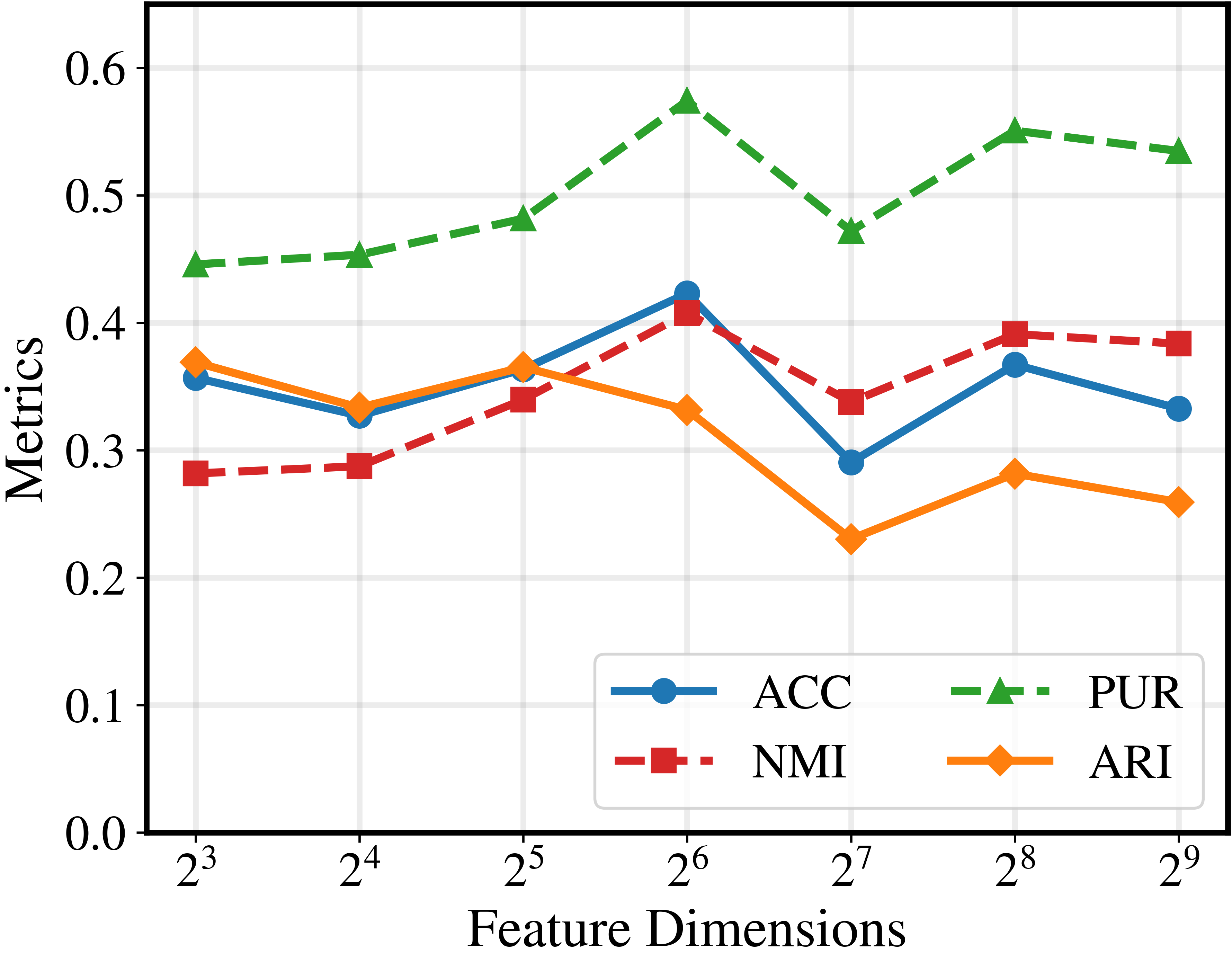}
}\hspace{0.05cm}
\subfigure[LGG] { 
\includegraphics[width=0.233\textwidth]{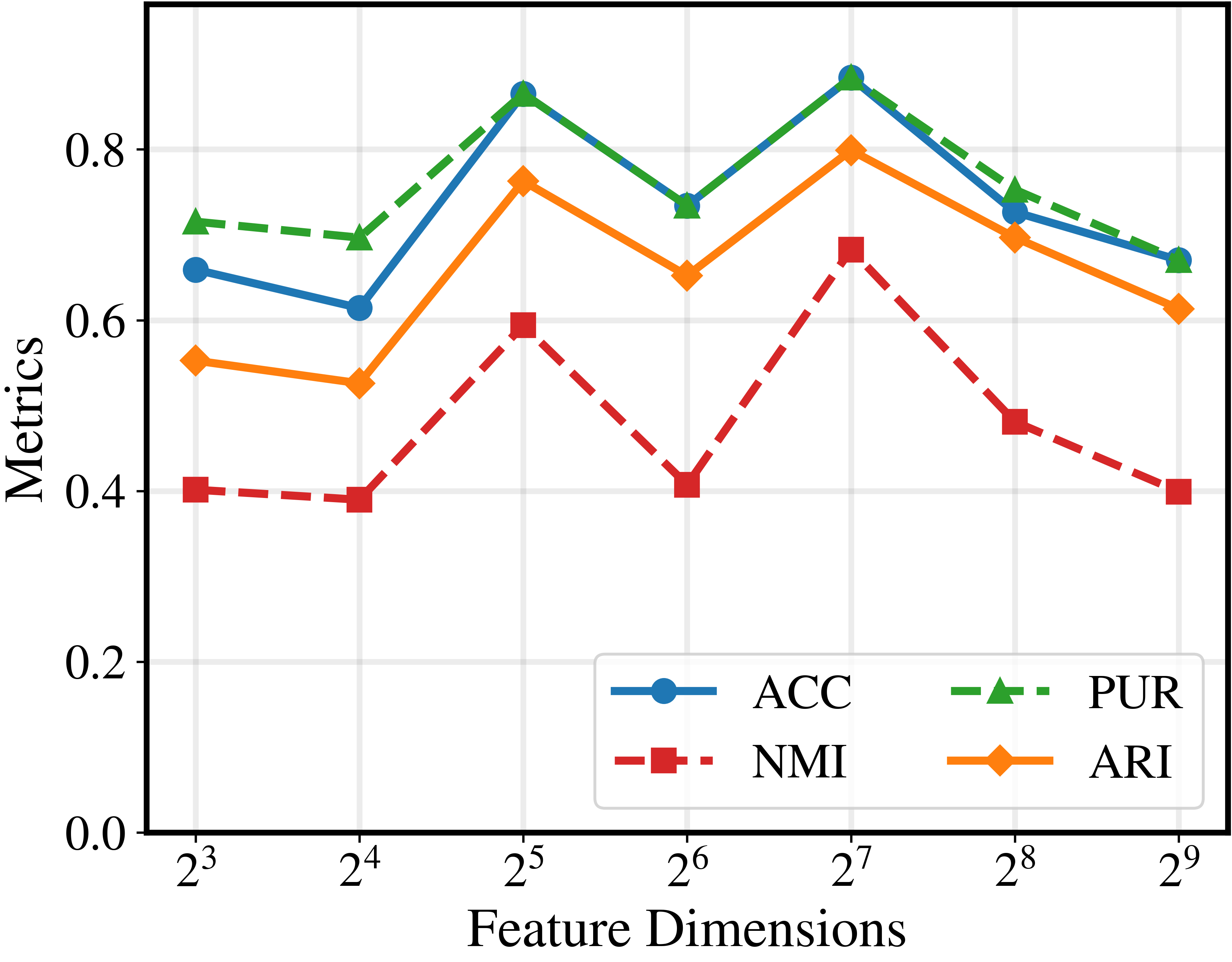}
}\vspace{-1em}
\caption{\revised{(a-d) Effect of embedding dimension $d$ on a range of evaluation metrics across different MVC datasets.}}
\label{fig:fdims} 
\end{figure*}
\subsubsection{Effectiveness of Loss Functions}
The overall loss function Eq. \eqref{eq:16overloss} includes two trade-off parameters, $\alpha$ and $\beta$. Fig. \ref{fig:bar} and Fig. \ref{fig:bar_2} show the Figures of GMAE when traversing $\alpha$ and $\beta$ from 0.01 to 0.07 with a step size of 0.01. We can observe that within the range of $\alpha$ from 0.01 to 0.02 and $\beta$ from 0.01 to 0.03, the clustering accuracy of GMAE is relatively insensitive to parameter variations and can be maintained at a relatively high level compared to other parameter settings. When the value of $\alpha$ is too large, $\mathcal{L}_{Cor}$ and $\mathcal{L}_{Dis}$ will dominate the optimization process. When the value of $\beta$ is too large, $\mathcal{L}_{Ent}$ will dominate the optimization process. Similarly, when both $\alpha$ and $\beta$ are too small, $\mathcal{L}_{Rec}$ will dominate the optimization process. This indicates the importance of the reconstruction loss in our framework, as well as the necessity of balancing all the losses. The suggested values for both $\alpha$ and $\beta$ are 0.01.

\subsection{Convergence analysis}
To thoroughly validate the convergence behavior of GMAE, we selected two datasets with notably different sample scales: the relatively small MSRCV1 (a-b) and the comparatively larger Digits (c-d). Both the complete setting (a, c) and the incomplete setting (b), along with the multi-view classification setting (d) used for reference, were considered in the evaluation. The experimental results in Fig. \ref{fig:resline} demonstrate that our method exhibits strong convergence across various experimental conditions. Even when dealing with the multiple views dataset Digits-6V (icomplete), which entails the dual challenge of view incompleteness and increased multi-view fusion complexity, our model continues to exhibit stable convergence as training is extended from 500 to 1000 epochs. The evaluation metrics ultimately stabilized after showing noticeable fluctuations in the early stages.

\section{Conclusion}
\label{sec:conclusion}
\revised{In this work, we propose GMAE to address inter-cluster entanglement in multi-view clustering. Through disentangled autoencoders and mutual information modulation, GMAE aligns cross-view distributions while preserving unique and common features. Our framework demonstrates state-of-the-art performance on 13 diverse datasets, including incomplete scenarios. With its linear complexity and robust representation capabilities, GMAE offers a scalable solution for various downstream tasks, particularly in clustering.}


\ifCLASSOPTIONcaptionsoff
  \newpage
\fi
\bibliographystyle{IEEEtran}
\bibliography{References} 

\begin{IEEEbiography}[{\includegraphics[width=1in,height=1.25in,clip,keepaspectratio]{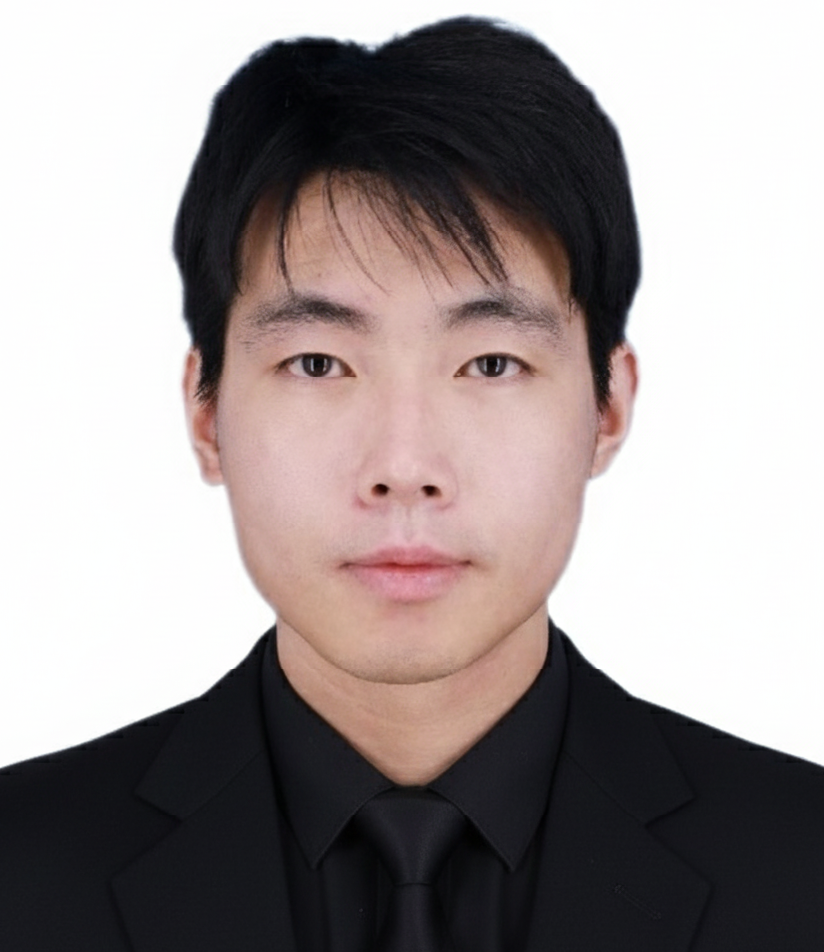}}]{Xin Zou} is currently pursuing the Ph.D. degree at
The Hong Kong University of Science and Technology, Guangzhou campus, China. He received the M.S. degree from China University of Geosciences in 2025. His research interests
include multimodal learning, trustworthy and efficient AI. He has published several peer-reviewed
papers such as ICML, NeurIPS, ICLR, etc. 
\end{IEEEbiography}
\begin{IEEEbiography}[{\includegraphics[width=1in,height=1.25in,clip,keepaspectratio]{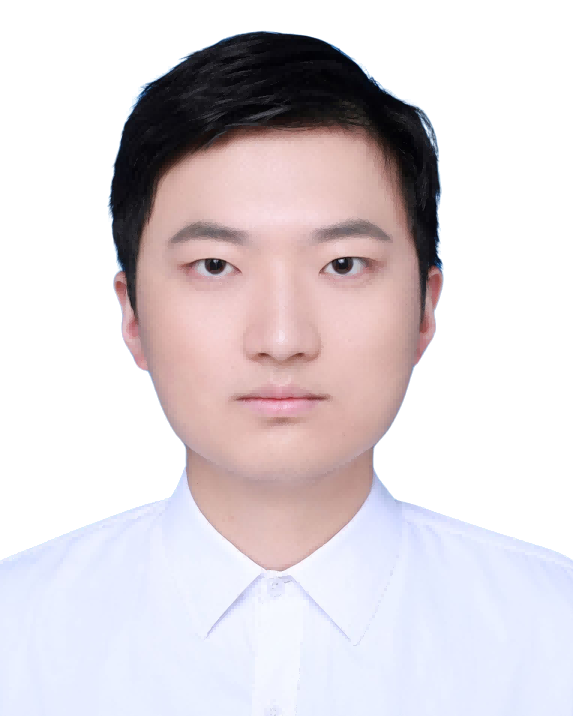}}]{Ruimeng Liu} is currently pursuing the Ph.D. degree in the School of Computer Science and Technology, Huazhong University of Science and Technology, Wuhan, China. He received the B.Eng. degree in Mechanical Design, Manufacturing, and Automation from the China University of Geosciences (Wuhan) in 2022, where he is also completing the M.Eng. degree in the school of Computer Science. His current research interests include multimodal representation learning and self-supervised learning, with a particular focus on investigating their interpretability and applications in computer vision and bioinformatics. More information can be found at \url{https://github.com/cleste-pome}.
\end{IEEEbiography}
\begin{IEEEbiography}[{\includegraphics[width=1in,height=1.25in,clip,keepaspectratio]{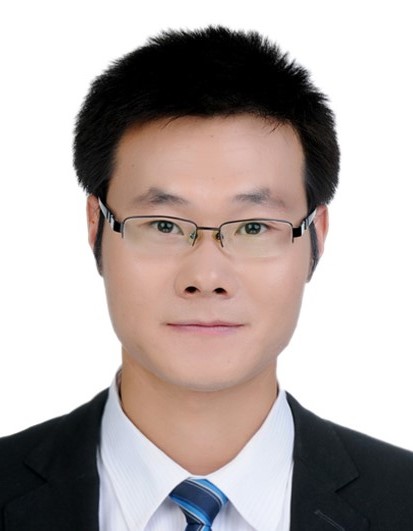}}]{Chang Tang (Senior Member, IEEE)} 
received his Ph.D. degree from Tianjin University, Tianjin, China in 2016. He joined the AMRL Lab of the University of Wollongong between Sep. 2014 and Sep. 2015. He is now a full Professor at the School of Software Engineering, Huazhong University of Science and Technology, Wuhan, China. Dr. Tang has published 100+ peer-reviewed papers, including those in highly regarded journals and conferences such as IEEE T-PAMI, IEEE T-MM, IEEE T-KDE, IEEE T-HMS, ICCV, CVPR, IJCAI, AAAI and ACM MM, etc. He serves as an associate editor of Neural Networks, IEEE Journal of Biomedical and Health Informatics, CAAI Transactions on Intelligence Technology, and BMC Bioinformatics. He regularly serves on the Technical Program Committees or as Area Chair of some top conferences such as NIPS, ICML, CVPR, ICCV, ECCV, IJCAI, ICME and AAAI. His current research interests include multi-modal learning and pattern recognition.
\end{IEEEbiography}
\begin{IEEEbiography}[{\includegraphics[width=1in,height=1.25in,clip,keepaspectratio]{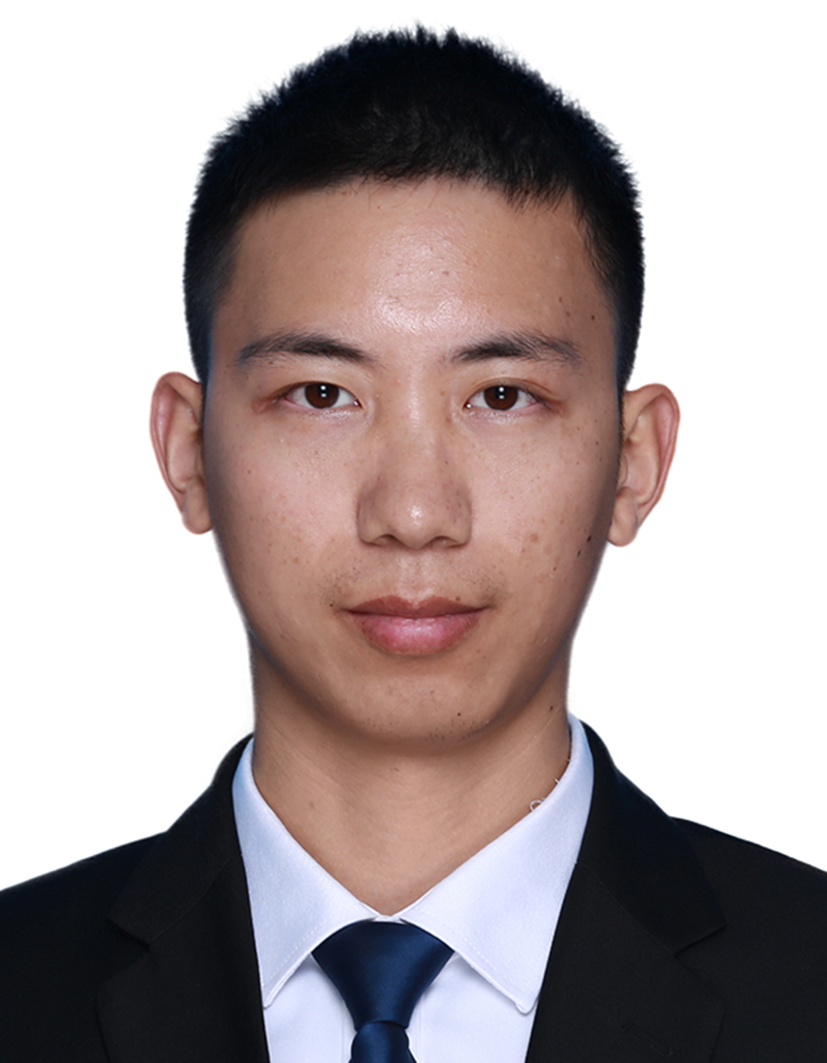}}]{Zhenglai Li (Member, IEEE)} obtained his Ph.D., M.S., and B.E. degrees from the China University of Geosciences, Wuhan, China, in 2024, 2021, and 2018, respectively. He is currently a postdoctoral research fellow at the Shenzhen Institute of Advanced Technology, Chinese Academy of Sciences, Shenzhen, China. Dr. Li has published 20 peer-reviewed papers in highly regarded journals and conferences, including IEEE TKDE, IEEE TMM, IEEE TIP, IEEE TCSVT, IEEE TGRS, and ACM MM. He regularly serves on the Technical Program Committees for leading conferences such as NeurIPS, ICLR, KDD, and ACM MM. His current research interests include machine learning and computer vision, with a focus on unsupervised learning, clustering, multi-view/multi-modal representation learning algorithms, and their applications in healthcare.
\end{IEEEbiography}
\begin{IEEEbiography}[{\includegraphics[width=1in,height=1.25in,clip,keepaspectratio]{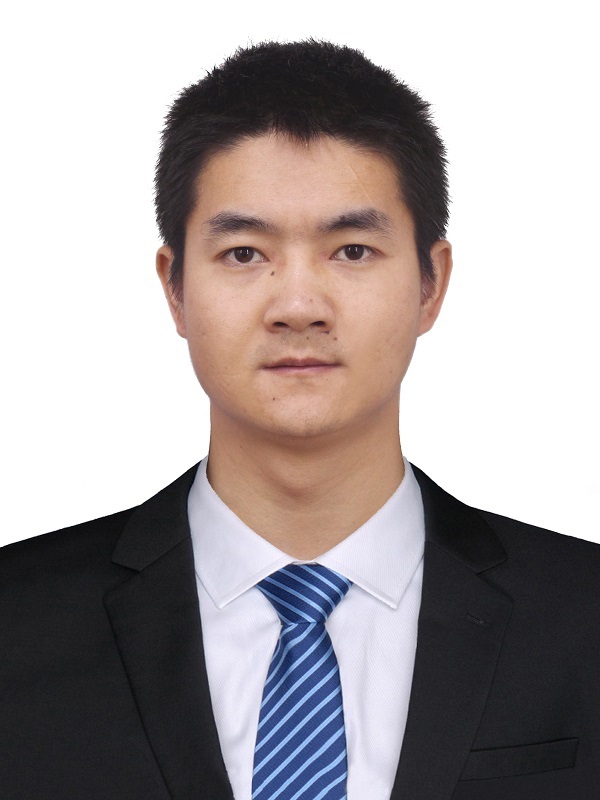}}]{Xinwang Liu (Senior Member, IEEE)} 
received his PhD degree from National University of Defense Technology (NUDT), China. He is now a full professor at the School of Computer, NUDT. His current research interests include kernel learning and unsupervised feature learning. Dr. Liu has published 60+ peer-reviewed papers, including those in highly regarded journals and conferences such as IEEE T-PAMI, IEEE T-KDE, IEEE T-IP, IEEE T-NNLS, IEEE T-MM, IEEE T-IFS, NeurIPS, ICCV, CVPR, AAAI, IJCAI, etc. He is a senior member of IEEE. More information can be found at \url{xinwangliu.github.io}.
\end{IEEEbiography}
\begin{IEEEbiography}[{\includegraphics[width=1in,height=1.25in,clip,keepaspectratio]{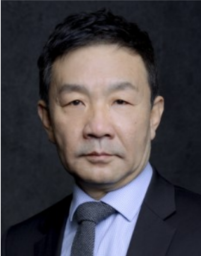}}]{Kunlun He} 
received his M.D. degree from the 3rd Military Medical University, Chongqing, China in 1988, and Ph.D. degree in Cardiology from Chinese PLA Medical School, Beijing, China in 1999. He worked as a postdoctoral research fellow at Division of circulatory physiology
of Columbia University from 1999 to 2003. He is the director and professor of Medical Big Data Research Center, Medical Engineering Laboratory of Chinese PLA General Hospital, Beijing, China. His research interests include big data and artificial intelligence of cardiovascular disease.
\end{IEEEbiography}
\begin{IEEEbiography}[{\includegraphics[width=1in,height=1.25in,clip,keepaspectratio]{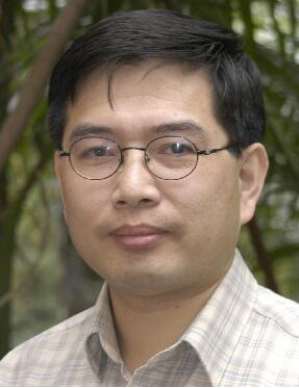}}]{Wanqing Li (Senior Member, IEEE)} 
received his PhD in electronic engineering from The University of Western Australia. He was an Associate Professor (91-92) at Zhejiang University, a Senior Researcher and later a Principal Researcher (98-03) at Motorola Research Lab, and a visiting researcher (08,10 and 13) at Microsoft Research US. He is currently a Professor and Director of Advanced Multimedia Research Lab (AMRL) of University of Wollongong, Australia. His research areas include machine learning, 3D computer vision, 3D multimedia signal processing and medical image analysis. Dr. Li is now a Senior Member of IEEE. He serves as an Associate Editor for IEEE Transactions on Circuits and Systems for Video Technology, IEEE Transactions on Multimedia and Journal of Visual Communication and Image Representation.
\end{IEEEbiography}
\end{document}